\documentclass[
 reprint,
showpacs,
 amsmath,amssymb,amsfonts,amsthm
 aps,
pre,
floatfix,
]{revtex4-2} 

\usepackage{amsthm}
\usepackage{graphicx}
\usepackage{etoolbox}
\usepackage{booktabs}
\usepackage{xcolor}

\usepackage{hyperref}

\usepackage[capitalize]{cleveref} 
\crefname{appendix}{Appendix}{Appendices}

\newtoggle{publish}
\toggletrue{publish}

\newtoggle{PRE}
\toggletrue{PRE}

\iftoggle{publish}{}
{
\usepackage{refcheck} 
\makeatletter
\newcommand{\refcheckize}[1]{%
  \expandafter\let\csname @@\string#1\endcsname#1%
  \expandafter\DeclareRobustCommand\csname relax\string#1\endcsname[1]{%
    \csname @@\string#1\endcsname{##1}\wrtusdrf{##1}}%
  \expandafter\let\expandafter#1\csname relax\string#1\endcsname
}
\makeatother

\refcheckize{\cref}
\refcheckize{\Cref}

}

\newtheorem{property}{Property}[section]
\newtheorem{theorem}{Theorem}[section]
\newtheorem{lemma}[theorem]{Lemma}

\newtheorem{corollary}[theorem]{Corollary}

\newcommand{\dominanceindex}{\bar{S}}

\newcommand{\averageswapdistance}{\left< d \right>}
\newcommand{\localaverageswapdistance}[1][i] {\left< d | #1 \right>}
\newcommand{\nmax}{7}

\newcommand{\lonelypermutohedron}{0.8}

\begin{document}

\title{How to measure the optimality of word or gesture order with respect to the principle of swap distance minimization}

\author{Ramon Ferrer-i-Cancho$^1$}
 \email{rferrericancho@cs.upc.edu}
 \homepage{https://cqllab.upc.edu/people/rferrericancho/} 

 \affiliation{$^1$Quantitative, Mathematical and Computational Linguistics Research Group (LQMC) \\
  Departament de Ci\`encies de la Computaci\'o \\
  Universitat Polit\`ecnica de Catalunya \\
  Campus Nord, Edifici Omega\\
  Jordi Girona Salgado 1-3 \\
  08034 Barcelona, Catalonia, Spain
 }
 

\date{\today}
             
\begin{abstract}
The structure of all the permutations of a sequence can be represented as a permutohedron, a graph where vertices are permutations and two vertices are linked if a swap of adjacent elements in the permutation of one of the vertices produces the permutation of the other vertex. It has been hypothesized that word orders in languages minimize the swap distance in the permutohedron: given a source order, word orders that are closer in the permutohedron should be less costly and thus more likely. Here we explain how to measure the degree of optimality of word order variation with respect to swap distance minimization. We illustrate the power of our novel mathematical framework by examining spontaneous gestures produced by speakers from four languages (two SVO languages and two VSO languages, which represent two linguistic families). We show that crosslinguistic gestures are at least $77\%$ optimal. It is unlikely that the multiple times where crosslinguistic gestures hit optimality are due to chance. We establish the theoretical foundations for research on the optimality of word or gesture order with respect to swap distance minimization in communication systems. Finally, we introduce the quadratic assignment problem (QAP) into language research as an umbrella for multiple optimization problems and, accordingly, postulate a general principle of optimal assignment that unifies various linguistic principles including swap distance minimization.    
\end{abstract}

\maketitle

\iftoggle{publish}{}{

\tableofcontents

\input{pending}

}

\section{Introduction}
\label{sec:introduction}

There are six possible orders of subject (S), object (O) and verb (V). Languages such as English and Russian tend to follow the SVO order as in the following English example:
\begin{equation*}
\underbrace{\text{The chef}}_S~\underbrace{\text{cooked}}_V~\underbrace{\text{a delicious meal}}_O.
\end{equation*}
In contrast, languages such as Tagalog and Irish, tend to follow the VSO order. For that reason, English and Russian are referred to as SVO languages while Tagalog and Irish are referred to as VSO languages. 

The mathematical structure of the six possible orders of S, O and V can be represented as a graph, called permutohedron, where two orders are joined by an edge if one order produces the other by swapping a couple of adjacent constituents and {\em vice versa} (\cref{fig:permutohedron}). For instance, SOV produces SVO by swapping OV and SVO produces SOV by swapping VO. 
The distance between orders in the graph is the swap distance \citep{Franco2024a}. In \cref{fig:permutohedron}, SOV is at distance 0 from itself, at distance 1 from SVO, at distance 2 from VSO and at distance 3 from VOS, that is the inverse of SOV. 
In a sequence of three constituents such as S, O and V, the maximum distance is 3.  
In a sequence of length $n$, the number of possible orders is $N = n!$ and
the maximum distance between any two vertices is 
\begin{equation}
d_{max} = {n \choose 2}
\label{eq:diameter_of_permutohedron}
\end{equation}
which gives $d_{max} = 3$ for $n = 3$.

\begin{figure}
\centering
\includegraphics[width = \lonelypermutohedron \linewidth]{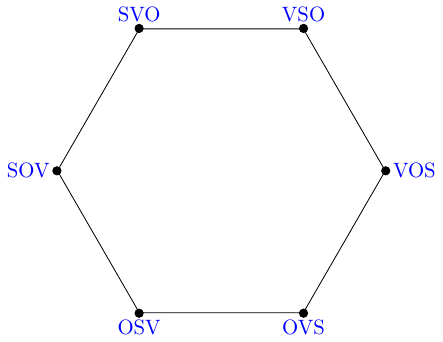}
\caption{\label{fig:permutohedron} The permutohedron of order 3. Vertices are labeled with all the possible permutations of SOV.} 
\end{figure}

It has been hypothesized that the cost of a word order variant is a monotonically decreasing function of
the swap distance between the variant and the source order \cite{Ferrer2023a,Ferrer2016c,Ferrer2013e}. 

Recently, a new index, called average swap distance, has been proposed to estimate such a cost \cite{Franco2024a}. The index is defined as 
\begin{equation}
\averageswapdistance = \sum_{i=1}^{N} \sum_{j=1}^{N} d_{ij} p_i p_j,
\label{eq:average_swap_distance}
\end{equation}
where $p_i$ is the probability of the $i$-th order of a structure of $n$ constituents and $d_{ij}$ is the swap distance between permutations $i$ and $j$. $d_{ij}$ is the distance (in edges) between vertex $i$ and $j$ in the permutohedron graph.  
The definition $\averageswapdistance$ makes a series of assumptions. First, it assumes that the swap distance cost is directly proportional to $d_{ij}$. However, the actual cognitive cost may be non-linear. Second, the definition can be seen as a particular case of a more general function \cite{Franco2024a}, i.e.  
\begin{equation*}
\averageswapdistance = \sum_{i=1}^{N} \sum_{j=1}^{N} d_{ij} p_{ij},
\end{equation*}
where $p_{ij}$ is the joint probability of orders $i$ and $j$. Then the original definition (\cref{eq:average_swap_distance}) is retrieved when assuming independence, namely $p_{ij} = p_i p_j$. 
Third, the cost of certain adjacent swaps may be higher than others. For instance, the transition from SOV to SVO may have a lower cost than a transition from SVO to VOS, although the swap distance in both transitions is one. The reason is that SVO minimizes dependency distances and the transition from SOV to SVO would imply shifting to an optimal configuration while the transition from SVO to VOS would imply abandoning such optimal configuration \cite{Ferrer2013e}. These differences could be incorporated by refining the definition of $d_{ij}$ or that of $p_{ij}$. Thus, another implicit assumption is that swap distance minimization is acting in isolation or it is the primary word order force. 

The minimum value of $\averageswapdistance$, i.e. $\averageswapdistance = 0$ is achieved when only one order has non-zero probability. The maximum value of $\averageswapdistance$ is \citep{Franco2024a} found 
\begin{equation*}
\averageswapdistance_{max} = \frac{d_{max}}{2} = \frac{n(n-1)}{4}
\end{equation*}
at least for $n \geq \nmax$. The maximum value is achieved when (a) all orders are equally likely and also when (b) the only two orders that are equally likely are located at maximum distance in the permutohedron graph.
$\averageswapdistance$ is a diversity index akin to entropy or Simpson's index \citep{Sommerfield2008a} but more powerful: configurations with the same word order entropy may differ in terms of $\averageswapdistance$ due to how probabilities are distributed on the permutohedron \cite{Franco2024a}.

According to the principle of swap distance minimization, $\averageswapdistance$ has to be minimized \citep{Franco2024a}. 
The hypothesis has been motivated in different ways. First, by the higher cost of swapping non-adjacent elements (versus swapping adjacent elements) due to the memory limitations underlying the dependency distance minimization principle \cite{Ferrer2016c}. Second, by means of an analog of rotation and canonical representation in visual recognition experiments \cite{Ferrer2023a}. These are preliminary approaches requiring further investigation. At present, the hypothesis is mainly supported by its predictive power.
Swap distance minimization has been applied successfully to distinct levels of explanation. First, diachronic word order variation, namely the evolution of the dominant order of a language over time \cite{Ferrer2013e,Ferrer2016c}. Second, synchronic word order variation, namely the variation that languages exhibit at a certain time \cite{Ferrer2016c,Franco2024a,Rios-El-Yazidi2026a,Rios-El-Yazidi2026b}. Third, word order acceptability \cite{Ferrer2023a}.
As for synchronic variation, massive evidence indicates that $\averageswapdistance$ is smaller than expected by chance in spoken languages of the world for SOV structures even in languages that do not follow the typical dominant orders (SOV and SVO), even in languages lacking a dominant order and even in other syntactic structures \cite{Rios-El-Yazidi2026a,Rios-El-Yazidi2026b,Franco2024a}.

Here we address the question of how to normalize $\averageswapdistance$ so as to measure the degree of optimization of word order with respect to the principle of swap distance minimization.
As a first approximation, we could normalize $\averageswapdistance$ by means of its theoretical range of variation. A simple normalization is thus $\averageswapdistance/\averageswapdistance_{max}$, which satisfies 
\begin{equation*}
0 \leq \averageswapdistance/\averageswapdistance_{max} \leq 1.
\end{equation*}
We could also normalize $\averageswapdistance$ in a more sophisticated way, following a state-of-the-art normalization method in language optimization research \citep{Ferrer2020b,Petrini2022a}.
For a certain score $X$, such normalization follows the template
\begin{equation} 
\Omega = \frac{X_r - X}{X_r - X_{min}}, 
\label{eq:optimality_score}
\end{equation}
where $X_r$ is the expected value of the score under some null hypothesis and $X_{min}$ is the minimum value that the score can achieve. 
By definition, such a score is normalized ($\Omega \leq 1$), takes a constant value when the arrangement is optimal ($\Omega = 1$) and is stable under the null hypothesis, namely it takes a value of $0$, on average, under the null hypothesis. $\Omega$ is a measure of closeness to optimality while $\averageswapdistance/\averageswapdistance_{max}$ is a measure of distance to optimality.
The mathematical structure of \cref{eq:optimality_score} is the same as the general form of adjusted similarity score for comparing partitions introduced by Hubert and Arabie \cite[p. 212]{Hubert1985a} and revisited 25 years later \cite{Vinh2010a}. The notion of similarity scores that are ``corrected or adjusted for chance'' is equivalent to our stability under the null hypothesis.

\cite{Ferrer2020b} applied the template to $D$, the sum of edge distances of a graph of $N$ vertices, that is defined as
\begin{equation*}
D = \frac{1}{2} \sum_{i=1}^N \sum_{j=1}^N |i - j| a_{ij}, 
\end{equation*}
where $A=\{a_{ij}\}$ is the adjacency matrix of the graph.
Then the template yields the optimality score 
\begin{equation} 
\Omega = \frac{D_r - D}{D_r - D_{min}}
\label{eq:optimality_dependency_distances}
\end{equation}
where $D_r$ and $D_{min}$ are the average and the minimum value of $D$ over all linear arrangements of vertices. 
\cite{Ferrer2020b} applied $\Omega$ (\cref{eq:optimality_dependency_distances}) to measure the degree of optimality of syntactic dependency distances in sentences according to the principle of syntactic dependency distance minimization \citep{Ferrer2004b}, which has been detected in ensembles of languages of increasing size and increasing linguistic diversity \cite{Liu2008a,Futrell2015a,Ferrer2020b}.
The score (\cref{eq:optimality_dependency_distances}) has been used to compare humans versus state-of-the-art generative models: the degree of optimality of humans is higher than that of Large Language Models \citep{Munoz-Ortiz2024a}. 
\cite{Petrini2022a} applied the template to $L$, the average word token length, which yields the optimality score
\begin{equation*} 
\Omega = \frac{L_r - L}{L_r - L_{min}},
\end{equation*}
where $L_r$ and $L_{min}$ are the average and the minimum value of $L$ over all one-to-one mappings of word probability into word length.
Replacing $X$ by our swap distance score, $\averageswapdistance$, we obtain
\begin{equation} 
\Omega = \frac{\averageswapdistance_r - \averageswapdistance}{\averageswapdistance_r - \averageswapdistance_{min}},
\label{eq:optimality_template}
\end{equation}
where $\averageswapdistance_r$ and $\averageswapdistance_{min}$ will be derived later on. 

The goal of the present article is two-fold. First, to set the theoretical foundations to estimate the degree of optimality of word or gesture order with respect to the principle of swap distance minimization and to investigate the epiphenomena of swap distance minimization, namely consequences of swap distance minimization on how probabilities are arranged on the permutohedron. Second, to demonstrate the application of the theory in a specific case: the order of subject, object and verb in unconventional gestures \citep{Futrell2015b}. We choose gestures to show that our theoretical framework is valid for any modality (vocal or gestural) and to shed light on a specific context of synchronic variation, namely the real-time production of words or gestures. In addition, unconventional gesture production is interesting because the preferred order of the gesturer can differ from the preferred order in his/her spoken language \cite{Goldin-Meadow2008a,Langus2010a,Futrell2015a}. For instance, the two most frequent orders in gestures produced by speakers of Irish and Tagalog (which are VSO languages), are SOV and SVO \cite{Futrell2015a}. However, the goal of the present article is not to explain why a particular gesture order is selected but to shed light on what remains {\em invariant} independently of the preferred order.

\begin{figure*}
\centering
\includegraphics[width = 0.7\linewidth]{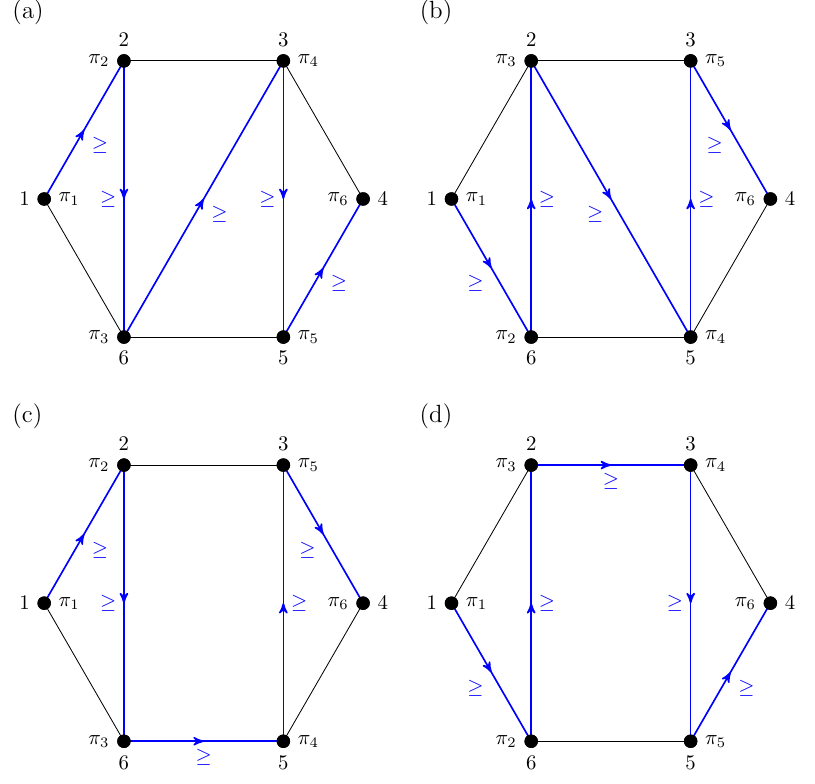}
\caption{\label{fig:optimal_probability_arrangement} Hasse diagrams (blue arrows) of four total orders of vertices of the permutohedron. 
Every Hasse diagram describes an arrangement of probabilities on the permutohedron. Vertices are labeled with numbers from 1 to 6 in clockwise sense starting from the left-most vertex. Near each vertex, we show the probability of the vertex. $\pi_i$ stands for the $i$-th largest probability. 
A blue arrow from vertex $i$ to vertex $j$ indicates that $p_i \geq p_j$. Black lines indicate edges of the permutohedron (only edges that are not hidden by a blue arrow are visible). 
(a) A Hasse diagram that minimizes $\averageswapdistance$ and $\localaverageswapdistance[1]$. 
(b) A symmetric Hasse diagram that also minimizes $\averageswapdistance$ and $\localaverageswapdistance[1]$.
(c) A Hasse diagram that minimizes $\localaverageswapdistance[1]$ but not $\averageswapdistance$. 
(d) A symmetric Hasse diagram that also minimizes $\localaverageswapdistance[1]$ but not $\averageswapdistance$. 
} 
\end{figure*}

The remainder of the article is organized as follows. Section \ref{sec:theory} completes the derivation of $\Omega$ for $\averageswapdistance$ combining mathematical and linguistic arguments to guide the choice of $\averageswapdistance_r$ and $\averageswapdistance_{min}$. Section \ref{sec:theory} also presents a series of theoretical results on $\Omega$, $\averageswapdistance_{min}$ and the epiphenomena of $\averageswapdistance$ minimization. Section \ref{sec:methods} presents the data and the methods used. Section \ref{sec:results} analyzes the optimality of unconventional gestures. It shows that unconventional gestures are at least $77\%$ optimal according to $\Omega$ (\cref{eq:optimality_template}) in addition to exhibiting various epiphenomena of $\averageswapdistance$ minimization. Finally, Section \ref{sec:discussion} discusses the findings and suggests future work. The supporting mathematical theory is presented in the Appendix. 

\section{Theoretical foundations}
\label{sec:theory}

\subsection{How to measure the degree of optimality}
\label{subsec:how}

To complete the derivation of $\Omega$ for $\averageswapdistance$ from \cref{eq:optimality_template}, we need to specify $\averageswapdistance_r$ and $\averageswapdistance_{min}$.
If we consider that the $p_i$'s are relative frequencies produced by rolling a fair die to pick one of the orders at random, one gets \citep{Franco2024a}
\begin{equation*}
\averageswapdistance_r = \frac{F - 1}{F} \averageswapdistance_{max},
\end{equation*}
where $F$ is the number of die rolls. Once the null hypothesis for $\averageswapdistance_r$ has been chosen, we have to choose $\averageswapdistance_{min}$ accordingly. Then $\averageswapdistance_{min}$ has to be the minimum value of $\averageswapdistance$ that die rolling can produce, i.e. $\averageswapdistance = 0$, that is achieved when all rolls produce the same side. Hence \cref{eq:optimality_template} with $\averageswapdistance_{min} = 0$ yields
\begin{align*}
\Omega & = 1 - \frac{F}{F -1} \frac{\averageswapdistance}{\averageswapdistance_{max}} \\
       & \approx 1 - \frac{\averageswapdistance}{\averageswapdistance_{max}}  
\end{align*}
for sufficiently large $F$.
Therefore, applying the die rolling null hypothesis to the general definition $\Omega$, we obtain again a score that is simply the complementary of our initial normalization attempt ($\averageswapdistance/\averageswapdistance_{max}$). Furthermore, the normalizations of $\averageswapdistance$ above take for granted that the minimum $\averageswapdistance$ can be achieved but the probability of hitting the minimum $\averageswapdistance$ (i.e. getting always the same side of the die) under the null hypothesis tends to zero as $F$ increases. Therefore, we need a proper null hypothesis. We will propose a null hypothesis after revising the typology of word order variation \citep{wals-81}.

Languages can be classified into rigid word order languages and flexible word order languages \cite{wals-81}. Rigid order languages are characterized by one of the six orders because all orders other than one are either ungrammatical or used relatively infrequently and only in special pragmatic contexts. For these languages, one would expect $\averageswapdistance \approx 0$.
However, flexible order languages, that is languages where all six orders are grammatical, are challenging. In some languages with flexible word order, there is one order which is most common and which can be described as the dominant order \citep{wals-81}. Some languages exhibit a pair of primary alternating dominant orders, e.g., SOV and SVO in German. These are languages in which it is not possible to identify a single dominant order but two orders, neither of which is dominant relative to the other, but which can be said to be dominant relative to other orders. Finally, there are languages lacking a dominant order such as Hungarian. 
Interestingly, swap distance minimization manifests even in languages lacking a dominant order \cite{Rios-El-Yazidi2026a}. Admittedly, the classification of languages into rigid/flexible order and further divisions of languages with flexible order that we borrow from \cite{wals-81} is rather idealized. See \cite{Rios-El-Yazidi2026a} for the challenge of classifying a language as lacking a dominant order using statistical criteria. We simply use this classification to motivate our choice for the normalization of $\averageswapdistance$. The point is that we
aim to find a way to normalize $\averageswapdistance$ that adapts to the inherent degree of word order flexibility of certain languages. For instance, we wish to find a way to normalize $\averageswapdistance$ conditioning on the fact that the language exhibits a pair of dominant orders and a single order will never dominate completely although the latter would lead to the absolute minimum value of $\averageswapdistance$, that is $0$ \cite{Franco2024a}. The solution is to consider a random permutation null hypothesis, namely a shuffling of the original word order probabilities \citep{Franco2024a}. There are $N = n!$ orders and $N!$ possible shufflings. When $n = 3$, the total number of probability shuffling is $(3!)! = 6! = 720$.

Consider a permutation $\sigma$, a one-to-one mapping of integers in $[1, N]$ to integers in $[1, N]$.
We use $\sigma$ to assign probabilities to vertices of the permutohedron, that is $\sigma(i)$ is the index of the probability assigned to vertex $i$. Hence vertex $i$ is assigned probability $p_{\sigma(i)}$.
Given the vector of probabilities
\begin{equation*}
\mathbf{p} = (p_1, ..., p_i,...,p_N),
\end{equation*}
the permutation $\sigma$ produces a value of $\averageswapdistance$ that is
\begin{equation}
\averageswapdistance(\sigma) = \sum_{i=1}^N \sum_{j=1}^N p_{\sigma(i)}p_{\sigma(j)}d_{ij}. \label{eq:average_swap_distance_permutation}
\end{equation}
By definition, a shuffling preserves the empirical distribution of the ${p_i}'s$, e.g., a preference for a pair of orders (no matter which pair), but destroys any constraint that the structure of the permutohedron may impose on how the probabilities are assigned to word orders. We cannot exclude, however, that the permutohedron has had some influence in the values of the word order probabilities.  
We use the term arrangement to define the probability vector obtained from $\mathbf{p}$ by applying some permutation, i.e. 
\begin{equation*}
(p_\sigma(1), ..., p_\sigma(i),...,p_\sigma(N)).
\end{equation*}

The expected value of $\averageswapdistance$ under the null hypothesis of a random permutation is \cite{Franco2024a}
\begin{equation}
\averageswapdistance_{r} = \dominanceindex \frac{N}{N-1}\frac{d_{max}}{2} = \dominanceindex \frac{N}{N-1}\frac{n(n-1)}{4},
\label{eq:expected_average_swap_distance_random_permutation}
\end{equation}
where $\dominanceindex = 1 - S$ is the so-called dominance index and $S$ is the Simpson index, that is defined as \citep{Sommerfield2008a}
\begin{equation}
S= \sum_{i=1}^{N} p_i^2.
\label{eq:Simpson_index}
\end{equation}
The $\averageswapdistance_{min}$ must be consistent with $\averageswapdistance_r$. Since $\averageswapdistance_r$ is the average value of $\averageswapdistance$ over all probability shufflings ($N!$ probability shufflings), $\averageswapdistance_{min}$ must be chosen as the minimum value of $\averageswapdistance$ over all these shufflings. Choosing $\averageswapdistance_{min} = 0$ as before would produce an inconsistent optimality score. Thus, our final version of $\Omega$ is obtained by defining $\averageswapdistance_r$ and $\averageswapdistance_{min}$ as the average and the minimum $\averageswapdistance$ over all permutations. 
Formally, 
\begin{align}
\averageswapdistance_{r}   & = \frac{1}{N!} \sum_{\sigma \in {\cal S}} \averageswapdistance(\sigma) \nonumber \\
\averageswapdistance_{min} & = \min_{\sigma \in {\cal S}} \averageswapdistance(\sigma) \label{eq:average_swap_distance_min_QAP},
\end{align}
where ${\cal S}$ is the set of all permutation functions over $[1, N]$.

The decisions above parallel the decisions in the application of $\Omega$ to $D$, the sum of dependency distances, (\cref{eq:optimality_dependency_distances}) in \cite{Ferrer2020b}, where
\begin{align}
D(\sigma) & = \frac{1}{2} \sum_{i=1}^N \sum_{j=1}^N |\sigma(i) - \sigma(j)| a_{ij} \label{eq:sum_of_dependency_distance_for_a_permutation} \\ 
D_r       & = \frac{1}{N!} \sum_{\sigma \in {\cal S}} D(\sigma) \label{eq:average_linear_arrangement_problem} \\
D_{min}   & = \min_{\sigma \in {\cal S}} D(\sigma). \label{eq:minimum_linear_arrangement_problem}
\end{align}
$D_{min}$ is known as the solution of the minimum linear arrangement problem \cite{Diaz2002,Petit2011a}. $D_{min}$ is the minimum value of $D$ over all possible permutations of the vertices of a graph of $N$ vertices. In dependency linguistics, $D_{min}$ is the minimum value of $D$ of a sentence of $N$ words according to the principle of dependency distance minimization \cite{Ferrer2004b,Liu2008a,Futrell2015a,Ferrer2020b} when the dependency structure is fixed and all word permutations are possible, namely no formal constraint \cite{Gomez2026a_Encyclopedia} or no additional word order principle limits the possible permutations of words.

Consider $m$, the number of non-zero probability orders, that satisfies $1 \leq m \leq N$ (the case $m = 0$ is not possible since the $p_i's$ must sum to $1$). The final version of $\Omega$ is undefined when $m = 1$ or when all orders are equally likely ($\pi_1 = 1/N$). When $m = 1$, all permutations have the same $\averageswapdistance$, namely $\averageswapdistance = 0$ and then $\averageswapdistance = \averageswapdistance_{min} = \averageswapdistance_r$. When all orders are equally likely, all permutations also have maximum $\averageswapdistance$, namely $\averageswapdistance = \averageswapdistance_{max}$ and then $\averageswapdistance = \averageswapdistance_{min} = \averageswapdistance_r$ \footnote{Notice that $S = \frac{1}{N}$ (\cref{eq:Simpson_index}) and then (\cref{eq:expected_average_swap_distance_random_permutation}) 
\begin{equation*}
\averageswapdistance_r = (1 - S)\frac{N}{N-1}\averageswapdistance_{max} = \averageswapdistance_{max}. 
\end{equation*}
}
. 
Hereafter, we assume $m > 1$ and that orders are not equally likely for calculating $\Omega$.

\subsection{The problem of $\averageswapdistance_{min}$}

To calculate $\Omega$, we need to compute $\averageswapdistance$, $\averageswapdistance_{r}$ and $\averageswapdistance_{min}$. We have a formula to compute $\averageswapdistance$ (\cref{eq:average_swap_distance}) and another to compute $\averageswapdistance_{r}$ (\cref{eq:expected_average_swap_distance_random_permutation}). $\averageswapdistance_{min}$ satisfies (Appendix \ref{app:bounds}, \cref{eq:Franco_Sanchez_et_al_bounds_refined})
\begin{equation*}
\dominanceindex \leq \averageswapdistance_{min}. 
\end{equation*}
The calculation of $\averageswapdistance_{min}$ (\cref{eq:average_swap_distance_min}) is a particular case of the Quadratic Assignment Problem (QAP) (Appendix \ref{app:quadratic_assignment_problem}). QAP is the general problem of finding a minimum cost allocation of facilities into locations, taking the costs as the sum of all possible distance-flow products \cite{Loiola2007a}. Interestingly, the minimum linear arrangement problem (dependency distance minimization) and the problem of compression with prescribed probabilities and magnitudes are special cases of QAP, too (Appendix \ref{app:quadratic_assignment_problem}). 
We can compute $\averageswapdistance_{min}$ with a brute force procedure, i.e. calculating the value of $\averageswapdistance$ for each of the $N!$ shufflings and selecting the minimum or a sophisticated algorithm to solve the QAP problem \cite{Loiola2007a,Wang2021a}. 
Fortunately, for $n = 3$, $\averageswapdistance_{min}$ can be computed efficiently with a straightforward procedure (Appendix \ref{app:optimal_arrangements}): 
\begin{enumerate}
\item
Sort the probabilities decreasingly so that $\pi_i$ is the $i$-th largest probability.
\item
Assign the sorted probabilities to vertices on the permutohedron following one of the schemes in \cref{fig:optimal_probability_arrangement} (a) and (b). 
\item
Calculate $\averageswapdistance_{min}$ by means of \cref{eq:average_swap_distance}.
\end{enumerate}
Indeed, $\averageswapdistance_{min}$ can be calculated as (Appendix \ref{app:optimal_arrangements})
\begin{align}
\averageswapdistance_{min} & = 3\dominanceindex - 2\left[ \right. \nonumber \\
                           &   \pi_1(2\pi_2 + \pi_4) + \pi_2(2\pi_4 + \pi_6) + \nonumber \\ 
                           &   \pi_3(2\pi_1 + \pi_2) + \pi_4(2\pi_6 + \pi_5) + \nonumber \\ 
                           &   \left. \pi_5(2\pi_3 + \pi_1) + \pi_6(2\pi_5 + \pi_3) \right]. \label{eq:average_swap_distance_min_introduction}                           
\end{align}

\subsection{An example of the calculation of $\Omega$}
\label{subsec:example}

We borrow the frequencies of the six possible orders produced by speakers of Tagalog in unconventional gesturing experiments under the non-reversible condition \cite{Futrell2015a}, which are reproduced in \cref{fig:permutohedron_combo_Tagalog_non_reversible} (a), to illustrate the calculation of $\Omega$, $\averageswapdistance$, $\averageswapdistance_r$ and $\averageswapdistance_{min}$ as follows
\begin{itemize}
\item
$\averageswapdistance_r$ is the easiest to calculate. 
If we plug the frequencies in \cref{fig:permutohedron_combo_Tagalog_non_reversible} (a) into \cref{eq:Simpson_index}, we obtain that the Simpson index is $S=0.5$. Plugging $S=0.5$ into \cref{eq:expected_average_swap_distance_random_permutation} with $n = 3$ and $N = 3! = 6$, we obtain $\averageswapdistance_r = 0.9$.
\item
$\averageswapdistance$ can be calculated by means of \cref{eq:average_swap_distance} using \cref{fig:permutohedron} as a guide to calculate the $d_{ij}$'s.
If we plug the frequencies in \cref{fig:permutohedron_combo_Tagalog_non_reversible} (a) into \cref{eq:average_swap_distance}, we obtain $\averageswapdistance = 0.73$.
\item
$\averageswapdistance_{min}$ can be calculated in two ways. First, mechanically, by sorting the frequencies in \cref{fig:permutohedron_combo_Tagalog_non_reversible} (a) decreasingly and plugging them into \cref{eq:average_swap_distance_min_introduction}, which gives $\averageswapdistance_{min} = 0.7$. Second, with a more insightful procedure that consists of rearranging the frequencies in \cref{fig:permutohedron_combo_Tagalog_non_reversible} (a) to produce the optimal arrangement in \cref{fig:optimal_probability_arrangement} (a), that is \cref{fig:permutohedron_combo_Tagalog_non_reversible} (b), and then plugging the frequencies in \cref{fig:permutohedron_combo_Tagalog_non_reversible} (b) into \cref{eq:average_swap_distance} to obtain $\averageswapdistance = 0.7 = \averageswapdistance_{min}$. 
\item
Finally, plugging $\averageswapdistance_r = 0.9$, $\averageswapdistance = 0.73$ and $\averageswapdistance_{min} = 0.7$ into \cref{eq:optimality_template}, we obtain $\Omega=0.86$. 
\end{itemize}
      
The high degree of optimality of Tagalog in this condition ($\Omega=0.86$) is not very surprising: the assignment in \cref{fig:permutohedron_combo_Tagalog_non_reversible} (a) is almost optimal, i.e. it suffices to swap the frequencies of OSV and VSO in \cref{fig:permutohedron_combo_Tagalog_non_reversible} (a) to produce the optimal assignment in \cref{fig:permutohedron_combo_Tagalog_non_reversible} (b).
It is easy to see that, for the assignment in \cref{fig:permutohedron_combo_Tagalog_non_reversible} (b), $\averageswapdistance_r = 0.9$ again and $\averageswapdistance = \averageswapdistance_{min} = 0.7$, which gives $\Omega = 1$.

\begin{figure}
\centering
\includegraphics[width = \linewidth]{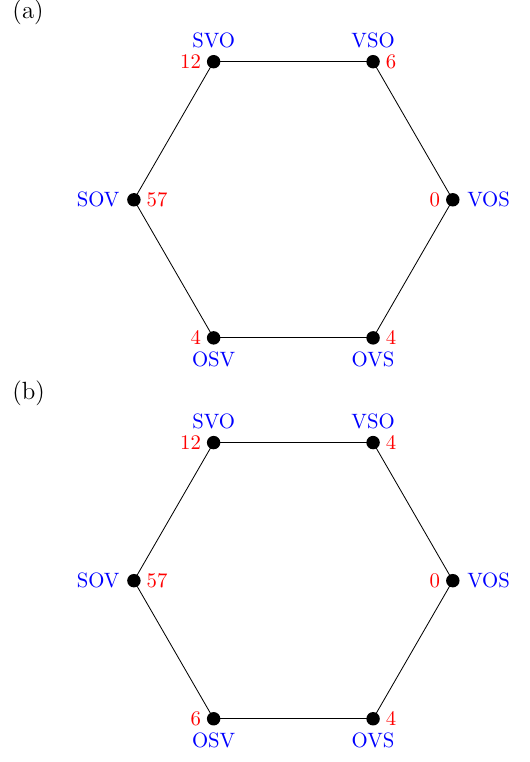}
\caption{\label{fig:permutohedron_combo_Tagalog_non_reversible} (a) The frequency of the six possible orders of S, O and V produced by speakers of Tagalog in unconventional gesturing experiments under the non-reversible condition \cite{Futrell2015a}. (b) An optimal assignment of the frequencies in (a).}
\end{figure}

\subsection{Overview of theoretical foundations}
\label{subsec:overview}

By its design, $\Omega \leq 1$ and the expected value of $\Omega$ in a random permutation is zero. $\Omega$ can take negative values, that indicate that $\averageswapdistance$ is larger than expected by chance.
$\Omega$ can take negative values for any $n \geq 3$ (Appendix \ref{app:bounds}).
When $n = 3$ and $m = 2$ or $m = 3$, $\Omega \geq -3/2$.

So far we have presented essential theory about $\Omega$ that suffices to apply $\Omega$ to our case study. Appendix \ref{app:bounds} presents general lower and upper bounds of $\averageswapdistance$ that are simple functions of $\dominanceindex$ in general bounds tailored for the specific case of $n = 3$ that improve with respect to previous work \citep{Franco2024a}. 
When $n = 3$, 
\begin{equation*}
\averageswapdistance_{low} \leq \averageswapdistance \leq \averageswapdistance_{up},
\end{equation*}
where
\begin{align*}
\averageswapdistance_{low}(\dominanceindex) & = \max\left(\dominanceindex, 2\dominanceindex - 2/3 \right) \\  
\averageswapdistance_{up}(\dominanceindex)  & = \min\left(\dominanceindex + 1, 2\dominanceindex + 1/2, \frac{3}{2}\right).
\end{align*}

Appendix \ref{app:optimal_arrangements} focuses on $n = 3$, presents the theory that is required to derive \cref{eq:average_swap_distance_min_introduction} and examines the structure of optimal arrangements. The major findings are summarized next.

Consider $V$, the set of vertices of the permutohedron. We use $t$, $u$ and $v$ to refer to vertices in $V$. 
Assigning probabilities to each vertex of the permutohedron induces a total relation between vertices.
If the probabilities do not have repeated values, the relation is a total order and then, sorting the vertices decreasingly (or increasingly) by probability, one obtains a single vertex sequence.
For instance, the probability assignment in \cref{fig:probability_assignment} (a) produces the total order of vertices (the vertex sequence)  
\begin{equation*}
(2,5,6,4,3,1)
\end{equation*}
when sorting vertices by decreasing probability.
If the probabilities have repeated values, the relation is a total preorder (weak order) and then multiple sequences are obtained when sorting the vertices. Crucially, in all these sequences, the vertices with unique probabilities are always in the same position in the sequences. 
For instance, the probability assignment in \cref{fig:probability_assignment} (b) produces the vertex sequences
\begin{align*}
& (2, 5, 6, 3, 4, 1) \\
& (2, 6, 5, 3, 4, 1) \\
& (2, 5, 6, 4, 3, 1) \\
& (2, 6, 5, 4, 3, 1)
\end{align*}
when sorting vertices by decreasing probability. Notice that the vertices with unique probabilities (1 and 2) are always in the same position.

\begin{figure*}
\centering
\includegraphics[width = 0.7\linewidth]{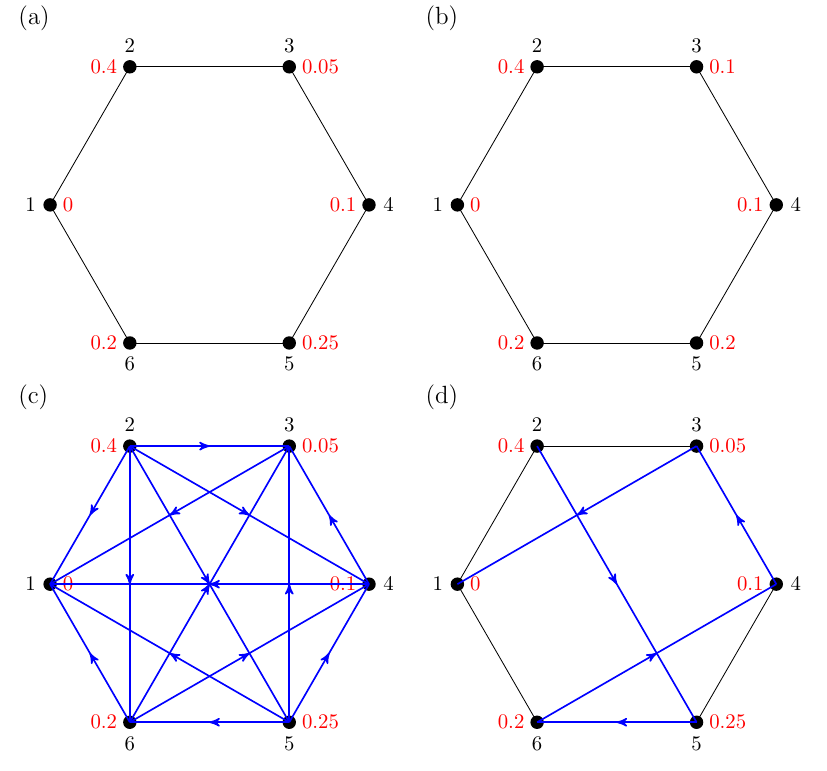} 
\caption{\label{fig:probability_assignment} Probability assigments to vertices of the permutohedron. Black numbers indicate vertex labels while red numbers indicate probabilities. (a) A probability assignment to vertices of the permutohedron without probability ties. (b) A probability assignment to vertices of the permutohedron with probability ties. (c) The directed graph (blue arcs) of the vertex relation induced by the probability assignment in (a). A blue arrow from vertex $i$ to vertex $j$ indicates that $p_i \geq p_j$. Black lines indicate edges of the permutohedron (only edges that are not hidden by a blue arrow are visible). (d) The Hasse diagram of the previous directed graph following the same conventions as in (c).
}
\end{figure*}

Consider a generic binary relation $R$ between vertices. The total orders and total preorders induced by the assignment of probabilities to vertices above are particular cases of $R$ where $u R v$ is equivalent to $p_u \leq p_v$. As we will see later on, the minimization $\averageswapdistance$ and $\localaverageswapdistance$ induce order relations of another kind, namely order relations 
to assign the probabilities $(\pi_1, \pi_2,...\pi_6)$ to vertices in the order prescribed by the order relation.

In order theory, a total order and a total preorder are reflexive, transitive and total relations. In addition, a total order is antisymmetric. For orders induced by the assignment of probabilities to vertices, the meaning of these properties is as follows. 
Total means that, for each pair of distinct vertices $u$ and $v$, either $p_u \leq p_v$ or $p_u \geq p_v$. Reflexivity means that $p_v \leq p_v$ for each vertex. Transitivity means that if $p_t \leq p_u$ and $p_u \leq p_v$ then $p_t \leq p_v$. A total order satisfies antisymmetry while a total preorder does not. Antisymmetry means that if $p_u \leq p_v$ and $p_u \geq p_v$ then $u = v$. When there are no probability ties, antisymmetry holds. When there are tied probabilities, the fact that $p_u \leq p_v$ and $p_u \geq p_v$ implies that
$p_u = p_v$ but it does not imply that $u = v$. For orders induced by the minimization of $\averageswapdistance$ or $\localaverageswapdistance$, it suffices to replace $p_u \leq p_v$ by $u R v$ in the explanations above.

A relation on vertices of the permutohedron can be visualized as a directed graph $G$, where vertices are the same vertices as in the permutohedron and an edge from $u$ to $v$ indicates $p_u \geq p_v$. \Cref{fig:probability_assignment} (c) shows the graph for the total order relation induced by the assignment in \cref{fig:probability_assignment} (a). 
To ease visualization and simplify modeling, a compact representation of $G$ is needed. A vertex $v$ can be reached from a vertex $u$ if there is a path from $u$ to $v$. Here we use a Hasse diagram, a directed graph $H$ that preserves the reachability with repect to $G$ for all pairs of vertices but using as few edges as possible. $H$ has no edge that can be inferred by reflexivity and transitivity \cite{Clough2014a}. Given $G$, $H$ is obtained by removing self-edges and computing the transitive reduction 
\cite{Warshall1962a,Aho1972a}. \Cref{fig:probability_assignment} (d) shows the Hasse diagram of the graph in \cref{fig:probability_assignment} (c). The Hasse diagram of a total order is a path graph.

The minimization of $\averageswapdistance$ induces total orders of the vertices of the permutohedron. Each of these total orders is a sequence of vertices $(t, u, v,...)$. The probabilities $(\pi_1,\pi_2, \pi_3,...)$ are assigned to vertices following the sequence, i.e. $p_t = \pi_1$, $p_u = \pi_2$, $p_v = \pi_3$, so as to minimize $\averageswapdistance$ (Appendix \ref{app:optimal_arrangements}). In these total orders, $t$ can be any vertex of the graph. 
The Hasse diagram of the two total orders that minimize $\averageswapdistance$ assuming $t = 1$ are shown in \cref{fig:optimal_probability_arrangement} (a) and (b), that correspond to the vertex sequences
\begin{align*}
& (1, 2, 6, 3, 5, 4) \\
& (1, 6, 2, 5, 3, 4).
\end{align*}
These total orders may produce an assignment of probabilities to vertices that in turn induces a total preorder of probabilities because the probabilities have ties.

\begin{figure}
\centering
\includegraphics[width = \linewidth]{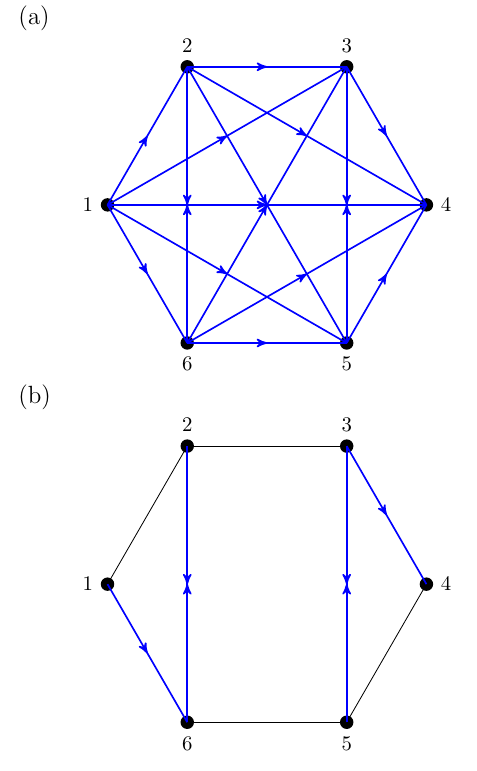} 
\caption{\label{fig:locally_optimal_probability_arrangement} The total preorder of vertices of the permutohedron (blue arrows) induced by the minimization of $\localaverageswapdistance[1]$. Vertices are labeled with numbers from $1$ to $6$ in a clockwise sense. A blue arrow from vertex $i$ to vertex $j$ indicates that $p_i \geq p_j$. 
Black lines indicate edges of the permutohedron (only edges that are not hidden by a blue arrow are visible). 
(a) The total preorder as a directed graph. Self-edges resulting from reflexivity are omitted for simplicity. (b) A {\em Hasse diagram} of the previous graph.
As the directed graph that defines the total preorder has cycles, the transitive reduction may not be unique \cite{Aho1972a}, and then we show only one of the Hasse diagrams (Appendix \ref{app:optimal_arrangements}).
}
\end{figure}


$\averageswapdistance$ (\cref{eq:average_swap_distance}) can be expressed equivalently as
\begin{equation}
\averageswapdistance = \sum_{i=1}^N p_i \localaverageswapdistance,
\label{eq:average_swap_distance_using_local_average_swap_distance}
\end{equation}
where
\begin{equation}
\localaverageswapdistance = \sum_{j=1}^N p_j d_{ij}.
\label{eq:local_average_swap_distance}
\end{equation}
The minimization of $\localaverageswapdistance$ for some vertex $i$ induces a ranking of the vertices according their distance to $i$; the minimum probability arrangements are obtained by assigning the probabilities $(\pi_1,\pi_2, \pi_3,...)$ to each vertex so that the vertex of distance zero receives $\pi_1$, the two vertices at distance 1 receive $\pi_2$ and $\pi_3$, and so on (Appendix \ref{app:optimal_arrangements}). 
Such a ranking with ties results in a total preorder of the vertices as when inducing an order from probabilities with ties above.
The total preorder involves the following vertex sequences when $i = 1$:
\begin{align*}
& (1, 2, 6, 3, 5, 4) \\
& (1, 6, 2, 5, 3, 4) \\
& (1, 2, 6, 5, 3, 4) \\
& (1, 6, 2, 3, 5, 4).
\end{align*}
Notice that the vertices with a unique rank (1 and 4) are always in the same position in the sequence. As every vertex sequence corresponds to a total order of vertices, Figure \ref{fig:optimal_probability_arrangement} shows the four total orders that minimize $\localaverageswapdistance[1]$. 

Optimality in a broad sense (either the minimization of 
$\averageswapdistance$ or the minimization of $\localaverageswapdistance$ for a given $i$ (\cref{fig:optimal_probability_arrangement}), 
has the following consequences (when probability is maximized by a single vertex)
\begin{enumerate}
\item
Radiation from the most likely order. The structure of optimal arrangements implies a tendency for probability to decrease as one moves away from the most likely order in the permutohedron (\cref{fig:optimal_probability_arrangement} (a-b) and \cref{fig:locally_optimal_probability_arrangement} (b)). 
\item 
Adjacency of the two most likely orders. The two most likely orders are connected in the permutohedron (\cref{fig:optimal_probability_arrangement}).
\item
Contiguity. The orders that have non-zero probability form a path on the permutohedron (\cref{fig:contiguity}).
\end{enumerate} 
A prominent consequence of optimality, contiguity, is examined in detail in Appendix \ref{app:contiguous}, where contiguity is confronted to optimality and it is shown to be expected even in suboptimal arrangements, as any non-contiguous arrangement can be transformed into a contiguous arrangement with smaller $\averageswapdistance$. 

\begin{figure}
\centering
\includegraphics[width = \linewidth]{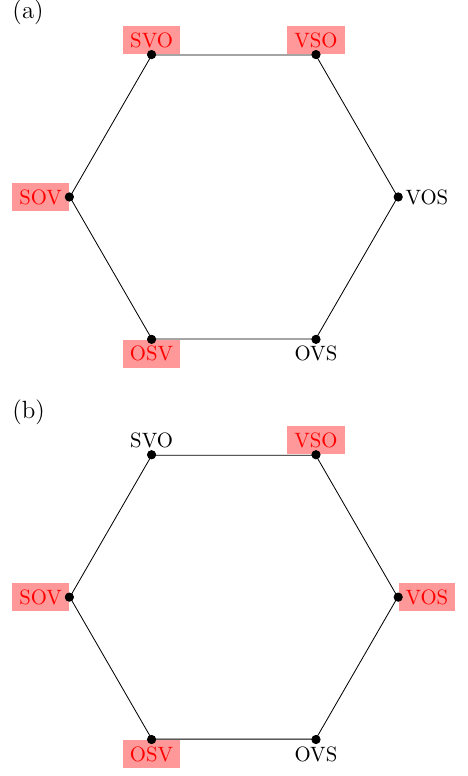}
\caption{\label{fig:contiguity} Arrangements with four non-zero probability orders (vertices with non-zero probability are marked in red). (a) The non-zero probability orders are contiguous, namely they form a path. (b) The non-zero probability orders are not contiguous. } 
\end{figure}

\section{Methods}
\label{sec:methods}

\subsection{Data}

Here we borrow the frequency of the 6 orders of the SOV structure produced by participants in unconventional gesturing experiments by \cite{Futrell2015b}. 
The participants were native or highly proficient bilingual speakers of two SVO languages (English and Russian) and two VSO languages (Irish and Tagalog). Thus the speakers represent two linguistic families: Austronesian (Tagalog) and Indo-European (English, Irish and Russian). Gesturing experiments took place in two conditions depending on the reversibility of the event (verb). Reversible events are those where the agent (subject) and the patient (object) could be plausibly reversed, e.g., the agent and the patient of the action are both humans and thus are both plausible as agents. Non-reversible are events where the agent and the patient cannot be reversed. For instance, "The boy kicked the girl." is reversible (both the boy and the girl can perform the action) while "The boy kicked the ball." is non-reversible (a ball cannot kick a boy) \cite{Futrell2015b}. 
The frequencies of the orders produced by reversibility condition (reversible or non-reversible) and language are borrowed from Tables 1 and 2 of \cite{Futrell2015b}.

\subsection{Calculation of $\averageswapdistance_{min}$}

We calculate $\averageswapdistance_{min}$ using two methods: the brute force procedure and the straightforward procedure described in Section \ref{sec:introduction}. The brute force procedure is used to verify every calculation of $\averageswapdistance_{min}$ with the straightforward procedure. $\averageswapdistance$ can be calculated by a direct implementation of \cref{eq:average_swap_distance}, that implies 36 sums and $36 \cdot 2$ products or a faster formula that requires fewer operations \citep[Appendix C.2]{Franco2024a}. The latter speeds up the brute force procedure. 

\subsection{The chance of the difference between $\averageswapdistance$ and $\averageswapdistance_r$}

Following \cite{Franco2024a}, we apply a Wilcoxon signed-rank test to determine if the difference between $\averageswapdistance$ and $\averageswapdistance_r$, when matched by language and reversibility condition, is statistically significant. The test is applied to all $\averageswapdistance$ and $\averageswapdistance_r$ matched pairs and also to each reversibility condition. The outcome of a test is $V$, the statistic and $\mathbb{P}_W$, the $p$-value.

\subsection{The chance that $\averageswapdistance = \averageswapdistance_{min}$}

We have the frequency of each order from speakers from four languages and two reversibility conditions.  
We wish to investigate the chance that $\averageswapdistance = \averageswapdistance_{min}$ in three ensembles: the reversible condition (the ensemble is formed by the four languages in that condition), the non-reversible condition (the ensemble is formed by the four languages in that condition) and any condition (the ensemble is formed by the eight language-reversibility pairs). 
We define a trial as the outcome of an experiment in one of the members of the ensemble. We wish to evaluate how likely it is that $\averageswapdistance$ is optimal across trials, that is, $\averageswapdistance = \averageswapdistance_{min}$, in an ensemble. 
We define $b_i$ as Bernoulli random variable that indicates if $\averageswapdistance$ is optimal for the $i-$th trial ($b_i = 1$ if optimal; $b_i = 0$ otherwise).  We define $B$ as the number of trials such that $\averageswapdistance$ is optimal, that is 
\begin{equation*}
B = \sum_{i=1}^T b_i,
\end{equation*} 
where $T$ is the number of trials. 

We define $\pi_o$ as the probability that a random permutation of probabilities is optimal for a certain trial. That value depends on $m$, the number of non-zero probabilities and the specific values of the non-zero probabilities.
For simplicity, we calculate $\pi_o$ numerically as the proportion of shufflings of probabilities such that $\averageswapdistance = \averageswapdistance_{min}$. We define $p_o(m)$ as the probability that a random permutation of probabilities is optimal knowing $m$ and assuming that there are no probability ties among the non-zero probability orders. In contrast to $\pi_o$, $p_o(m)$ has a simple analytical expression (Appendix \ref{app:number_of_optimal_arrangements}) 
\begin{equation}
p_o(m) \left\{
          \begin{array}{ll} 
          \frac{1}{60}(6 - m)! & \mbox{if } 1 < m \leq 6 \\
          1 & \mbox{if } m = 1.
          \end{array}
       \right.
\label{eq:probability_of_optimal_by_chance}
\end{equation}
Notice that $\pi_o = p_o(m_i)$ if there are no probability ties among the non-zero probability orders of the trial.

We define $m_i$ as the number of non-zero probability orders of the $i$-th trial. Assuming independence among the $b_i$'s, the probability that $b_i = 1$ is $\pi_o(i)$. 
As $B$ is a sum of independent Bernoulli random variables, it is well-known that $B$ follows a Poisson binomial distribution with parameters $\pi_o(1), \pi_o(i),..., \pi_o(T)$. We define $\mathbb{P}_o$ as the right $p$-value of a Poisson binomial test on $B$. 
$\mathbb{P}_o$ indicates the probability that a value of $B$ that is at least as large as the actual one is obtained by chance. $\mathbb{P}_o$ was calculated using the R package {\tt poisbinom} \citep{Olivella2025a} that implements a method based on discrete Fourier transform \citep{Hong2013a}. 

\subsection{The chance of contiguity}

\begin{table}
\caption{\label{tab:distribution_of_number_of_different_orders} The distribution of $m$, the number of non-zero probability orders in each ensemble (non-reversible, reversible or any). $T(m)$ is the number of $m$ non-zero probability orders in members of an ensemble. }
\centering
\begin{tabular}{llll}
    & \multicolumn{3}{c}{$T(m)$} \\
    \cmidrule(lr){2-4}
$m$ & non-reversible & reversible & any \\
\hline
1 & 0 & 0 & 0 \\
2 & 0 & 0 & 0 \\
3 & 1 & 1 & 2 \\
4 & 2 & 3 & 5 \\
5 & 1 & 0 & 1 \\
6 & 0 & 0 & 0 \\
\end{tabular}
\end{table}

Recall that we say that orders are contiguous when the non-zero probability orders form a path on the permutohedron (\cref{fig:contiguity}).
We wish to investigate the chance of contiguity in the three ensembles above. 
We define $T(m)$ as the number of trials with $m$ non-zero probability orders in an ensemble. Table \ref{tab:distribution_of_number_of_different_orders} shows the values of $T(m)$ in our dataset. We define $p_c(m)$ as the probability that $m$ non-zero probability orders are arranged contiguously on the permutohedron with $n=3$ by chance, which is (Appendix \ref{app:contiguous})
\begin{equation*}
p_c(m) = \left\{ 
             \begin{array}{ll} 
                 \frac{m! (6 - m)!}{5!} & \mbox{if } 1\leq m \leq 5 \\
                 1 & \mbox{if } m \in \{0, 6\}. 
             \end{array}    
          \right.
\end{equation*}
$p_c(m)$ is equivalent to the probability that $m$ randomly chosen vertices form a path in that permutohedron. 

We wish to evaluate how likely it is that the configurations are contiguous for any value of $m$ in an ensemble. To that end, we define $m_i$ as the number of non-zero probability orders of the $i$-th trial and $c_i$ as a Bernoulli variable that indicates if the configuration of the $i$-th trial is contiguous ($c_i = 1$ if contiguous; $c_i = 0$ otherwise).
We also define $C$ as the number of contiguous configurations over $T$ trials, that is
\begin{equation*}
C = \sum_{i=1}^T c_i.
\end{equation*}
Under independence among trials, $c_i$ is a Bernoulli random variable with parameter $p_c(m_i)$ and $C$ follows a Poisson binomial distribution with parameters $p_c(m_1),..., p_c(m_i), ..., p_c(m_T)$. We define $\mathbb{P}_c$ as the right $p$-value of a Poisson binomial test. $\mathbb{P}_c$ is the probability of achieving at least as many contiguous configurations as in the actual data by chance. 
$\mathbb{P}_c$ can be calculated using the same methods used for $\mathbb{P}_o$. 
However, when the non-zero probability orders are arranged consecutively on the permutohedron in all trials ($C = T$),
we have 
\begin{align*}
\mathbb{P}_c & = \prod_{m=1}^6 p_c(m)^{T(m)} \\ 
             & = \prod_{m=2}^4 p_c(m)^{T(m)}.  
\end{align*}
In our dataset, $C = T$ and $T(2) = 0$ reduce the calculation to
\begin{equation}
\mathbb{P}_c =  \left(\frac{3}{10}\right)^{T(3)} \left(\frac{2}{5}\right)^{T(4)}. \label{eq:our_probability_always_consecutive_orders} 
\end{equation}

\section{Results}
\label{sec:results}

\begin{figure}
\centering
\includegraphics[width = \linewidth]{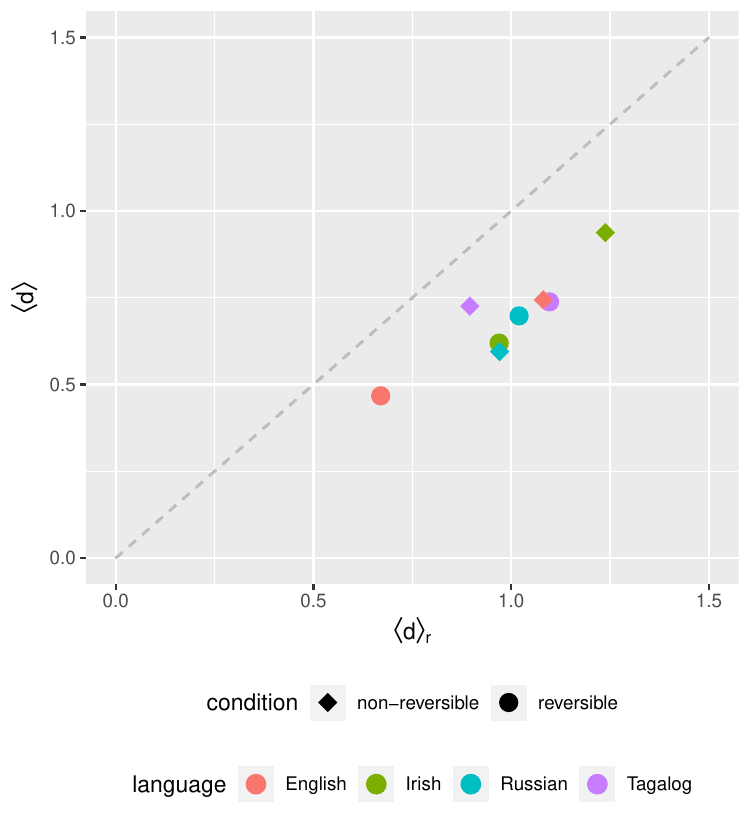}
\caption{\label{fig:swap_distance_minimization_in_crosslinguistic_gestures} The average swap distance ($\averageswapdistance$) as a function of the random baseline ($\averageswapdistance_r$) in crosslinguistic gestures. The dashed line is a control line to indicate identity, i.e. $\averageswapdistance = \averageswapdistance_r$. Points below the control line satisfy $\averageswapdistance < \averageswapdistance_r$. Color point indicates the language of the participants. Point shape indicates the reversibility condition. } 
\end{figure}

Figure \ref{fig:radar_SOV} shows the frequency of every word order. First, notice that word orders are sorted following the structure of the permutohedron (\cref{fig:permutohedron}).
Second, notice that we use logarithmic scale for gesture order frequencies so that the slice of every order that has non-zero frequency can be easily seen. Such a choice is motivated by the fact that the rank distribution of dominant orders in languages decays exponentially \citep{Cysouw2010a, Ferrer2024b}. Third, notice that all non-zero frequency orders are contiguous in the permutohedron. This is a remarkable feature that we will revisit later on.

\subsection{Evidence for swap distance minimization}

Table \ref{tab:statistical_summary_average_swap_distance} summarizes the statistical properties of crosslinguistic gestures.
Figure \ref{fig:swap_distance_minimization_in_crosslinguistic_gestures} shows that $\averageswapdistance < \averageswapdistance_r$ as predicted by the principle of swap distance minimization, independently of the language and the reversibility of the verb. The Wilcoxon signed-rank test indicates that the differences between $\averageswapdistance$ and $\averageswapdistance_r$ are unlikely to be due to chance ($V=0$ and $\mathbb{P}_W = 0.004$ when conditions are combined and $V=0$ and $\mathbb{P}_W = 0.062$ in each condition).

\begin{table*}
\caption{\label{tab:statistical_summary_average_swap_distance} Summary of the statistical information by reversibility condition and participant's language: $F$ (the total frequency), $m$ (the number of non-zero probability orders), $\dominanceindex$ (the dominance index), $\pi_o$ (the probability that a random permutation is optimal), $p_o(m)$ (the probability that a random permutation is optimal assuming no probability ties among the $m$ non-zero probability orders), 
$\averageswapdistance_{min}$ (the minimum value of $\averageswapdistance$ over all permutations), $\averageswapdistance$ (the average swap distance), $\averageswapdistance_{r}$ (the expected value of $\averageswapdistance$ in a random permutation), 
$\averageswapdistance_{max}$ (the maximum value of $\averageswapdistance$) and $\Omega$ (the optimality of $\averageswapdistance$). 
}
\centering
\begin{tabular}{@{} llllllllllll @{}}
Reversibility & Language & $F$ & $m$ & $\dominanceindex$ & $\pi_o$ & $p_o(m)$ & $\averageswapdistance_{min}$ & $\averageswapdistance$ & $\averageswapdistance_{r}$ & $\averageswapdistance_{max}$ & $\Omega$ \\
\hline
reversible & English & 121 & 4 & 0.37 & 0.033 & 0.033 & 0.41 & 0.47 & 0.67 & 1.5 & 0.78\\
 & Russian & 103 & 3 & 0.57 & 0.1 & 0.1 & 0.64 & 0.7 & 1.02 & 1.5 & 0.85\\
 & Irish & 82 & 4 & 0.54 & 0.033 & 0.033 & 0.62 & 0.62 & 0.97 & 1.5 & 1\\
 & Tagalog & 82 & 4 & 0.61 & 0.033 & 0.033 & 0.74 & 0.74 & 1.1 & 1.5 & 1\\
non-reversible & English & 119 & 4 & 0.6 & 0.033 & 0.033 & 0.71 & 0.74 & 1.08 & 1.5 & 0.91\\
 & Russian & 117 & 3 & 0.54 & 0.1 & 0.1 & 0.59 & 0.59 & 0.97 & 1.5 & 1\\
 & Irish & 81 & 4 & 0.69 & 0.033 & 0.033 & 0.94 & 0.94 & 1.24 & 1.5 & 1\\
 & Tagalog & 83 & 5 & 0.5 & 0.033 & 0.017 & 0.7 & 0.73 & 0.9 & 1.5 & 0.86\\
\\
\end{tabular}
\end{table*}

\begin{figure*}
\centering
\includegraphics[width = \linewidth]{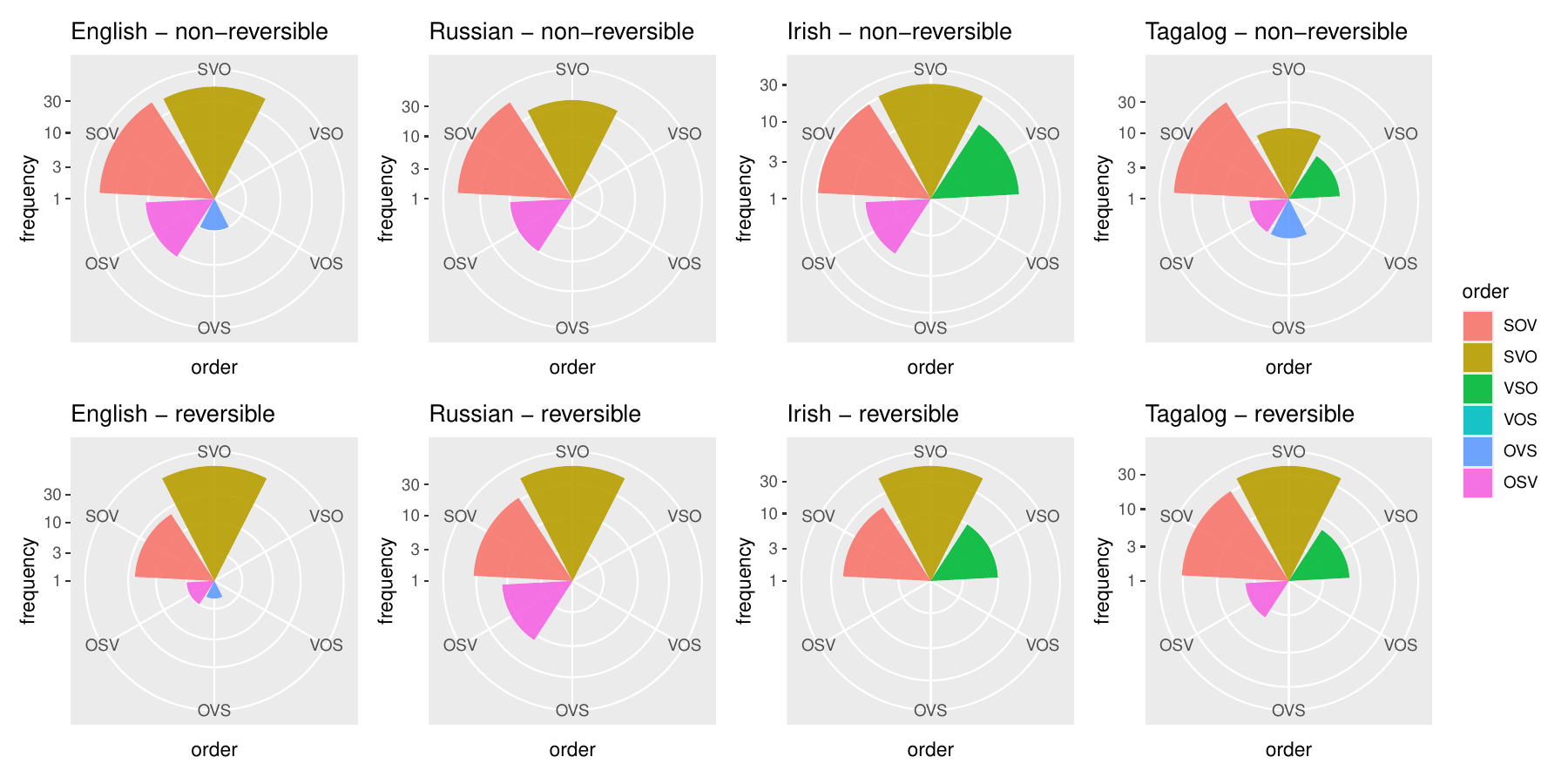}
\caption{\label{fig:radar_SOV} The distribution of gesture orders in each combination of language and reversibility condition. Languages are sorted by dominant order: SVO on the left (English and Russian) and VSO on the right (Irish and Tagalog).}
\end{figure*}
  
\subsection{The degree of optimality of crosslinguistic gestures}
  
Interestingly, we find that $\Omega \geq 0.77$ in all cases (Table \ref{tab:statistical_summary_average_swap_distance}). In half of the cases, the gesture orders are fully optimal ($\Omega = 1$): Irish (in both conditions), Tagalog (in the reversible condition) and Russian (in the non-reversible condition). The number of times the arrangements are optimal is larger than expected by chance according to the Poisson binomial test: $T = 4, B = 2$ and $\mathbb{P}_o= 0.013$ under the non-reversible condition and also under the reversible condition; $T = 8, B = 4$ and $\mathbb{P}_o = 3 \cdot 10^{-4}$ when the two conditions are mixed. Notice that we have $\pi_o = p_o(m)$ except for Tagalog under the non-reversible condition because both OSV and OVS have frequency 4 (Table \ref{tab:statistical_summary_average_swap_distance}).
 
English under the non-reversible condition is suboptimal ($\Omega < 1$) because the 4th most frequent order (OVS) is not at swap distance 1 of the 2nd most frequent order (SVO), breaking the pattern of optimal arrangements. However, the degree of optimality of English under the non-reversible condition is still high ($\Omega = 0.91$). Indeed, English under the non-reversible condition minimizes $\localaverageswapdistance$ when $i$ is SOV. Notice that it matches one of the total orders in \cref{fig:optimal_probability_arrangement} (c) and (d) that in turn correspond to one of the total orders compatible with the total preorder (but not total order) that is obtained when minimizing $\localaverageswapdistance$ (\cref{fig:locally_optimal_probability_arrangement} (b)).

\subsection{The epiphenomena of swap distance minimization}

We check two consequences of swap distance minimization: radiation from the most likely order and contiguity of the non-zero probability orders (Section \ref{sec:theory}). 
As for the latter, notice that all the cases where $\Omega = 1$ and English under the non-reversible condition exhibit perfect radiation from the most likely order: probability decreases as swap distance to the most likely order increases.
As for the former, we investigate if the contiguous arrangement of non-zero probability orders in all cases (\cref{fig:radar_SOV}) could be due to chance. Table \ref{tab:distribution_of_number_of_different_orders} summarizes the information about $T(m)$ in our dataset (Table \ref{tab:statistical_summary_average_swap_distance}). 
Applying Table \ref{tab:distribution_of_number_of_different_orders} to \cref{eq:our_probability_always_consecutive_orders}, we have that, under the reversible condition, 
\begin{equation*}
\mathbb{P}_c = \frac{3}{10} \left(\frac{2}{5}\right)^3 = \frac{12}{625} = 0.0192.
\end{equation*}
Under the non-reversible condition, 
\begin{equation*}
\mathbb{P}_c = \frac{3}{10} \left(\frac{2}{5}\right)^2 = \frac{6}{125} = 0.048.
\end{equation*}
Combining both conditions, 
\begin{equation*}
\mathbb{P}_c = \left(\frac{3}{10}\right)^2 \left(\frac{2}{5}\right)^5 = \frac{72}{78125} = 9.216 \cdot 10^{-4}.
\end{equation*}
Therefore, such as pervasiveness of contiguous configurations is unlikely to be due to chance. 

\section{Discussion}
\label{sec:discussion}

\subsection{The optimal assignment principle}

The quadratic assignment problem (QAP) consists of assigning facilities to locations so as to minimize a cost function \cite{Koopmans1957a}. 
Several problems are indeed particular cases of such a general problem (Appendix \ref{app:quadratic_assignment_problem}). In the minimum linear arrangement problem of computer science, one has to find an optimal assignment of vertices to positions in a sequence \cite{Diaz2002}. 
In the problem of compression with prescribed probabilities and lengths, one has to find an optimal assignment of lengths to probabilities, or the other way around, i.e. an optimal assignment of probabilities to lengths \cite{Ferrer2019c}. Here we have investigated the problem of minimization of average swap distance with prescribed probabilities. 

In the past, each of those subproblems has been paired with optimization principles that operate on languages and other communication systems. \cite{Ferrer2004b} introduced the minimum linear arrangement problem into language research and postulated the principle of Euclidean distance minimization. \cite{Ferrer2019c} introduced the problem of compression with prescribed probabilities and length to provide theoretical foundations for the principle of compression \cite{Ferrer2012d}. 
Here our interest in the swap distance minimization principle has taken us to an optimization problem that turns out to be a particular case QAP.  
Since QAP is a unifying umbrella for all the problems above, here we postulate a general principle of optimal assignment as an umbrella principle for the compression principle, the dependency distance minimization principle and the swap distance minimization principle. The principle of optimal assignment has the potential for explaining a myriad of linguistic phenomena: namely all the phenomena that in the past have been attributed to each of these principles: various linguistic laws \cite{Ferrer2019c} and a wide range of word order patterns \cite{Liu2017a,Temperley2018a,Rios-El-Yazidi2026a}. The optimal assignment principle is not unique to language as it also applies to economic systems \cite{Koopmans1957a} and to biological systems \cite{Semple2021a}. 
 
We have investigated the problem of the minimization of $\averageswapdistance$ under the constraint that the distribution of probabilities is fixed. We have seen that the optimal probability arrangements are obtained by following a traversal of the permutohedron graph and assigning the probabilities in non-increasing order as a new vertex is encountered (\cref{fig:optimal_probability_arrangement} (a) and (b)). Such strategy is reminiscent of the one that is needed to solve the problem of compression, the minimization of the average code length $L$, in the non-singular coding scheme \citep{Ferrer2019c}. In that setting, probabilities are sorted in non-increasing order and assigned to strings following a breadth-first traversal (BFS) of the coding tree (a tree where two strings $s$ and $t$ are joined if adding a symbol at the end of $s$ produces string $t$). For the minimization of $\averageswapdistance$ one also needs to follow a BFS of the permutohedron but not all breadth first traversals produce an optimal arrangement in general (the BFSs in \cref{fig:optimal_probability_arrangement} (a) and (b) do while the BFSs in \cref{fig:optimal_probability_arrangement} (c) and (d) do not). Future research should investigate if a specific BFS also applies to $n > 3$.	

\subsection{Theory construction}

For the sake of parsimony, linguistic theories must separate principles from their manifestations. 
For instance, the low number of edge crossings in syntactic dependency structures is to a large extent a manifestation of the principle of syntactic dependency distance minimization \citep{Ferrer2014c,Gomez2016a,Gomez2019a,Yadav2022a}. One does not need to postulate a principle of minimization of edge crossings to justify the scarcity of crossing dependencies. Similarly, here we have shown that various structural properties: adjacency (of the two most likely orders), radiation (from most likely order) and contiguity (of non-zero probability orders) are implications of a single postulate: swap distance minimization (Section \ref{subsec:overview}; see Appendix \ref{app:optimal_arrangements} for further details). Put differently, these structural properties are epiphenomena of swap distance minimization.
One does not need to postulate an independent tendency for any of those structural properties in linguistic systems as they follow simply from swap distance minimization. 
In crosslinguistic gestures, all configurations exhibit adjacency (the most frequent orders are adjacent in the permutohedron) and are contiguous (\cref{fig:radar_SOV}), as predicted theoretically. Consistently, in most languages that exhibit a couple of primary alternating dominant orders \cite{wals-81}, the orders paired are adjacent in the permutohedron even when none of the orders paired coincides with the two most frequent dominant orders, i.e. SOV and SVO \citep{Ferrer2016c}. 

\subsection{The virtue of our theoretical framework} 

There is a trade-off between generality and predictive power. General theories tend to lack precision or predictive power. However, our theory is an example of a general postulate, swap distance minimization, that has very precise implications that follows combining the permutohedron for S, O and V (\cref{fig:permutohedron} and the minimum $\averageswapdistance$ configurations (\cref{fig:optimal_probability_arrangement} (a) and (b)). If the most likely order is SOV, the prediction about the probabilities of each order is one of the following
\begin{align}
& SOV \geq SVO \geq OSV \geq VSO \geq OVS \geq VOS \nonumber \\
& SOV \geq OSV \geq SVO \geq OVS \geq VSO \geq VOS. \label{eq:SOV_as_source}
\end{align}
If the most likely order is SVO, the prediction is one of the following 
\begin{align}
& SVO \geq VSO \geq SOV \geq VOS \geq OSV \geq OVS \nonumber \\
& SVO \geq SOV \geq VSO \geq OSV \geq VOS \geq OVS. \label{eq:SVO_as_source}
\end{align}

A general challenge of linguistic theory is to overcome the bias towards Indo-European languages or languages from Western, Educated, Industrialized, Rich and Democratic (WEIRD) societies \cite{Blasi2022a}.
A specific challenge of theories of word order is their compatibility with the actual diversity of word orders of languages of the world. 
Traditional models of word order proceed essentially by pure induction from the distribution of dominant orders of the world \cite{Cysouw2008a,Futrell2015b} and thus predict that SOV or SVO should be the most preferred orders, which coincides with the fact that SOV and SVO are the most frequent dominant orders in spoken languages of the world \cite{Hammarstroem2016a}.
However, they lack the generality and precision of \cref{fig:optimal_probability_arrangement} and its implications (e.g., \cref{eq:SOV_as_source} and \cref{eq:SVO_as_source}).
There are languages such that their dominant order is not any of these two orders. This is not a problem for our theoretical framework. If the most likely order is VSO as in many Mayan and Austronesian languages, swap distance minimization predicts one of the following preferences 
\begin{align*}
& VOS \geq OVS \geq VSO \geq OSV \geq SVO \geq SOV\\
& VOS \geq VSO \geq OVS \geq SVO \geq OSV \geq SOV.   
\end{align*} 
This demonstrates that unifying approaches as ours are able to handle languages that deviate from the main trend. 

\subsection{Crosslinguistic gestures}

We have illustrated the utility of our theoretical apparatus by showing that gesture order is highly optimized according to the principle of swap distance minimization. In particular, we have found that $\Omega \geq 0.77$ independently of the participant's language and the reversibility condition. We have shown that it is unlikely that the multiple times where crosslinguistic gestures hit optimality are due to chance. 
    
We consider contiguity (Section \ref{subsec:overview}) to demonstrate the power of our theoretical apparatus for generating predictions. Such property is an implication of optimality: any optimal arrangement must be contiguous (Appendix \ref{app:optimal_arrangements}, Corollary \ref{cor:optimality}). Furthermore, any non-contiguous arrangement has a rearrangement with smaller $\averageswapdistance$ (Appendix \ref{app:contiguous}, Property \ref{prop:contiguity_of_non_zero_probability_orders}). Thus, contiguity is expected from the minimization of $\averageswapdistance$ even when the arrangement is suboptimal. Therefore, it is not surprising that all the arrangements shown in \cref{fig:radar_SOV} are contiguous even when the combination of language and reversibility conditions is not optimal ($\Omega < 1$).
We have indeed shown that the number of contiguous configurations is larger than expected by chance. Adjacency and radiation should be the next targets of future research.

\subsection{Research on spontaneous gesturing}

Unconventional gestures are a window into the origins of signed or spoken languages \cite{Corballis2012a} and into the principles that govern word or gesture order \cite{Goldin-Meadow2008a,Langus2010a}.
The factors that determine order in sequences consisting of a subject (agent), a verb (action) and an object (patient) in spontaneous gesturing experiments have received considerable attention. 
As for the semantics of the event, researchers have been concerned about several distinctions: reversible versus non-reversible events \cite{Hall2014a,Futrell2015b,Schouwstra2022a}, manipulation versus creation events \cite{Christensen2016a}, and extensional versus intensional events \cite{Schouwstra2014a}. 

Here, instead of focusing once again on the conditions that determine the most frequent order, as in most research on word/gesture order, we have taken the opposite perspective: we have examined what remains {\em invariant} across experimental conditions and across all attested orders. We have shown that the variation in gesture order reported in \cite{Futrell2015b} experiments is constrained by the structure of the underlying space of possible permutations across experimental conditions. 
Why a particular order is selected in these gestural experiments is beyond the scope of the present article and should be the subject of future research. 

\subsection{The optimality of languages}

This article is part of a research program on quantifying the degree of optimality of languages with respect to distinct optimization principles using the template of $\Omega$ (\cref{eq:optimality_score}): compression \cite{Petrini2022a}, syntactic dependency distance minimization \cite{Ferrer2020b} and swap distance minimization \cite{Rios-El-Yazidi2026b}. 
With respect to compression, languages are optimized to 62 or 67 percent on average (depending on the source) when word lengths are measured in characters, and to 65 percent on average when word lengths are measured in time \cite{Petrini2022a}. With respect to syntactic dependency distance minimization, languages are optimized to 66\% on average in the Universal Dependencies (UD) collection and that percentage grows to 71\% when using Surface-Syntactic Universal Dependencies (SUD) annotation style \cite{Ferrer2020b}. The high degree of optimization of $\Omega$ with respect to swap distance minimization in spontaneous gestures ($\Omega \geq 0.77$) is in contrast with the lower degree of optimality of spoken languages according to a $99\%$ confidence interval obtained with stratified sampling over languages with linguistic families as strata \cite{Rios-El-Yazidi2026b}. In the New Testament, the interval is $[0.66, 0.76]$. In the UD collection, the interval is $[0.39, 0.67]$ when using UD annotation style and $[0.37, 0.63]$ when using SUD annotation style. Therefore, spontaneous gestures are, in general, more optimized than conventional spoken languages. The reason of such a difference should be the subject of future research. Our reanalysis of \cite{Futrell2015b} has covered only gesturers from four spoken languages. A wider sample of spoken languages should we used for fairer comparison.

\subsection{Type-based versus token-based typology}

Traditionally, languages have been classified by their dominant order, namely the most common order displayed by a language under certain conditions \cite{wals-81}. Then Turkish is an SOV language, English is an SVO language and Cebuano is a VSO language. This approach, that reduces the actual distribution of word orders to the mode is an example of the so-called type-based typology \cite{Levshina2019a}. The alternative is the so-called token-based typology, that emphasizes that the distribution of word orders is captured more accurately by diversity indices such as entropy \cite{Levshina2023a}. The limits of type-based typology with respect to token-based typology have been criticized \cite{Levshina2019a,Levshina2023a}. Our mathematical analysis reconciles and enriches these two apparently opposite approaches. First, it is token-based for being based on an a diversity index, $\averageswapdistance$. However, $\averageswapdistance$ is more powerful than entropy for incorporating the structure of the space of permutations \cite{Franco2024a,Rios-El-Yazidi2026a}. Second, it establishes that the most frequent order has a critical role in the variation of word order: when swap distance minimization is the only force, all others emanate from the most frequent order, in particular, the frequency of other orders decreases in a particular way as their distance to the most frequent order increases (\cref{sec:theory}). Put differently, the dominant order is more than the mode. 

\subsection{Swap distance minimization and cognitive ease}

\begin{figure}
\centering
\includegraphics[width = \linewidth]{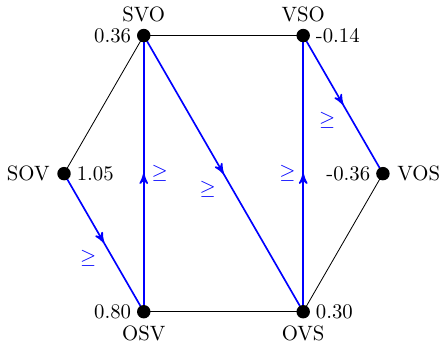}
\caption{\label{fig:Malayalam} The Hasse diagram of induced from acceptability ratings in Malayalam (blue) on the permutohedron (black). A blue arrow from vertex $i$ to vertex $j$ indicates that $p_i \geq p_j$ where $p_i$ is the acceptability rating of vertex $i$. The acceptability rating each order is shown near each vertex. 
}
\end{figure}

In our mathematical framework, probability can be replaced by a normalized cost and the mathematical results remain the same. That is, $p_i$ can be a probability or a normalized measure of cognitive ease. To illustrate the broader applicability of our framework, we leave gesture experiments and consider the case of word order acceptability judgments in Malayalam, a language with default SOV order \citep{Namboodiripad2017a}. \Cref{fig:Malayalam} shows the acceptability of each word order according to \citep[Table 2.7]{Namboodiripad2017a}. The arrangement of word order acceptability judgments in \cref{fig:Malayalam} corresponds to one of the two optimal arrangements such that SOV is the most likely order (\cref{fig:optimal_probability_arrangement}). A strong correspondence can be established algebraically: these acceptability ratings are distinct real numbers and thus induce a total order of the vertices.
Therefore, the acceptability ratings induce a Hasse diagram that is shown in \cref{fig:Malayalam} and that turns out to coincide with one of the Hasse diagrams of minimum $\averageswapdistance$ (\cref{fig:optimal_probability_arrangement}  (a) and (b)).
Thus, our theory predicts the distribution of cognitive cost on the permutohedron with higher precision than \cite{Ferrer2023a}.

The probability of an optimal arrangement of acceptability ratings has been produced by chance assuming that the most likely order can be any (that is ignoring that SOV is the default order in Malayalam) is (\cref{eq:probability_of_optimal_by_chance} with $m = 6$)
\begin{equation*}
p_o(6) = \frac{1}{60} = 0.01\bar{6}.
\end{equation*}
Knowing that SOV is the default order in Malayalam, the probability of hitting the arrangement displayed by Malayalam (and not any of the 5 alternative arrangements with another order as the most frequent) is
\begin{equation*}
\frac{1}{6}p_o(6) = \frac{1}{360} = 0.002\bar{7}.
\end{equation*}
Thus, it is unlikely that our theory makes a successful prediction by chance. 

This example has just been used as an illustration of the power of our theoretical framework. Further research in other contexts is necessary. 

\subsection{Conclusion}

To sum up, we have paved the way for (a) a large scale analysis of the degree of optimality of word or gesture order with respect to swap distance minimization, (b) the investigation of epiphenomenal properties of swap distance minimization (e.g., contiguity) in communication systems and (c) research on cognitive cost from a swap distance minimization perspective. 
Finally, notice that our mathematical framework can be applied to any permutation of three elements. $\averageswapdistance$ is a serious alternative to the Bandt-Pompe permutation entropy for the analysis of time series \cite{Bandt2002a} for incorporating the hidden structure of the permutation space \cite{Rios-El-Yazidi2026a}. Therefore, our optimality index ($\Omega$) and the underlying theoretical framework provides new tools for Ordinal Analysis \cite{Leyva2022a}.

\section*{Acknowledgments}

We are grateful to M. Serna for insights from theoretical computer science and helpful discussions. We also thank J. Rios El-Yazidi for helpful comments, A. Punnen for advice on QAP and M. Arias for calling our attention to adjusted similarity indices. This research is supported by the grant PID2024-155946NB-I00 funded by Ministerio de Ciencia, Innovación y Universidades (MICIU), Agencia Estatal de Investigación (AEI/10.13039/501100011033) and the European Social Fund Plus (ESF+).

\iftoggle{PRE}{

\section*{Data availability statement}

The code and data that support the findings of this article are openly available \cite{Ferrer2025c_open_data}.
}{
}

\appendix
\renewcommand{\theequation}{A\arabic{equation}}
\setcounter{equation}{0}


\section{Quadratic assignment problem}

\label{app:quadratic_assignment_problem}


Consider the matrix of flows $W = \{w_{ij}\}$ and the matrix of distances $D = \{d_{ij}\}$ \cite{Loiola2007a}.
$w_{ij}$ is the flow between facilities $i$ and $j$ and $d_{ij}$ is the distance between locations $i$ and $j$. $\sigma$ is a permutation on $[1, N]$. $\sigma(i)$ is the location of facility $i$. Then the Quadratic Assignment Problem (QAP) \cite{Loiola2007a} can be defined as the following optimization problem 
\begin{equation}
\min_{\sigma \in {\cal S}} \sum_{i=1}^N \sum_{j=1}^N w_{ij} d_{\sigma(i)\sigma(j)},
\label{eq:quadratic_assignment_problem_on_permutations}
\end{equation}
where ${\cal S}$ is the set of all permutation functions on $[1, N]$.                           
\Cref{eq:quadratic_assignment_problem_on_permutations} is one of the definitions of QAP when formulated by means of permutations
\cite[Section 2.1.3]{Loiola2007a}. \Cref{eq:quadratic_assignment_problem_on_permutations} defines the most common version of QAP, that is known as Koopmans–Beckmann formulation \cite{Koopmans1957a}, shortly QAP-KB. That formulation of QAP should not be confused with alternative definitions of QAP \cite{Wang2021a}.

A class of QAP-KB problems is obtained when $d_{ij}$ stands for the distance between vertices $i$ and $j$ in a graph $G$ \cite{Alvarez2007a}. We refer to it as QAP-KB-G.

\subsection{The minimum linear arrangement problem}

If $w_{ij}$ is the adjacency matrix of a graph and $G$ is the path graph (hence $d_{ij} = |i - j|$), it is well-known that QAP-KB becomes the minimum linear arrangement problem \cite[p. 218]{Garey1979}. To illustrate it, the insertion of \cref{eq:sum_of_dependency_distance_for_a_permutation} into \cref{eq:minimum_linear_arrangement_problem} yields
\begin{equation*}
D_{min} = \frac{1}{2} \min_{\sigma \in {\cal S}} \sum_{i=1}^N \sum_{j=1}^N a_{ij} |\sigma(i) - \sigma(j)|
\end{equation*}
and then $w_{ij} = a_{ij}$ and $d_{\sigma(i)\sigma(j)} = |\sigma(i) - \sigma(j)|$. 
In this setting, the expected value of $D$ over all linear arrangements (\cref{eq:average_linear_arrangement_problem}) is \cite{Ferrer2018a}
\begin{equation*}
D_r = m \frac{N + 1}{3},
\end{equation*}
where $m$ is the number of edges of the graph. 
See \cite{Alemany2021a,Alemany2021b,Alemany2022c} for the calculation $D_{min}$ and $D_r$ on trees when constraints on the possible linear arrangements are imposed. 

\subsection{Average swap distance minimization}

We will show that the calculation of $\averageswapdistance_{min}$ (\cref{eq:average_swap_distance_min_QAP}) is a special case of QAP-KB-G where $G$ is the permutohedron and $w_{ij} = p_i p_j$.
By plugging the definition of $\averageswapdistance(\sigma)$ (\cref{eq:average_swap_distance_permutation}) into the definition of $\averageswapdistance_{min}$ (\cref{eq:average_swap_distance_min_QAP}),
we obtain 
\begin{align*}
\averageswapdistance_{min} & = \min_{\sigma \in {\cal S}} \sum_{i=1}^N \sum_{j=1}^N p_{\sigma(i)}p_{\sigma(j)} d_{ij} \\
                           & = \min_{\sigma \in {\cal S}} \sum_{i=1}^N \sum_{j=1}^N p_i p_j d_{\sigma^{-1}(i)\sigma^{-1}(j)} \\
                           & = \min_{\sigma \in {\cal S}} \sum_{i=1}^N \sum_{j=1}^N p_i p_j d_{\sigma(i)\sigma(j)} \\
                           & = \min_{\sigma \in {\cal S}} \sum_{i=1}^N \sum_{j=1}^N w_{ij} d_{\sigma(i)\sigma(j)}, 
\end{align*}
where $w_{ij} = p_i p_j$.
Thus, the minimization of $\averageswapdistance$ over the space of all permutations is equivalent to a particular case of QAP-KB.
That is, the minimization of $\averageswapdistance$ is equivalent to finding an optimal assignment of probabilities to individual vertices of the permutohedron.   
 
In our variant of QAP-KB, $W$ results from the outer product of the vector $\mathbf{p}$ with itself. If $\mathbf{p}$ is a column vector, then 
\begin{equation*}
W = \mathbf{p} \mathbf{p}^T.
\end{equation*}
That implies that our $W = \{w_{ij}\}$ is symmetric and has rank 1. In our application, we have considered that $\mathbf{p}$ is a vector of probabilities and then our application is a special case of $w_{ij} = p_{ij}$ where $p_{ij}$ is a joint probability. Our particular case follows by assuming independence, that is $p_{ij} = p_i p_j$.
Besides, $D = \{d_{ij}\}$ corresponds to vertex-vertex distance on the permutohedron graph, that is a particular case of regular graph (all vertices have degree $n - 1$) with potentially useful additional features \cite[Appendix B]{Franco2024a}.

\subsection{The problem of compression}

In classic information theory, the average length of codes is defined as 
\begin{equation*}
L = \sum_{i = 1}^N p_i l_i,
\end{equation*}
where $p_i$ is the probability of a number and $l_i$ is the length of its code \cite{Cover2006a}.
The problem of compression is defined as the minimization of $L$ \cite{Cover2006a}.
In a linguist setting, $p_i$ can be interpreted as the probability of a word and $l_i$ its length in graphemes or phonemes. Classic information theory is concerned about the minimization of $L$ 
assuming that the probabilities are prescribed and the codes and their corresponding lengths can be {\em a priori} any within a coding scheme, i.e. unique decodability \cite{Cover2006a} or a less restrictive constraint that is non-singular coding \cite{Ferrer2019c}. The problem of compression has been generalized to the case where $l_i$ can be a positive real number, e.g., $l_i$ the duration of a call type in a non-human species, but imposing the constraint that the possible $l_i$'s are prescribed, namely they must be a permutation of a multiset of size $N$ \cite{Ferrer2019c}. 
In such a setting, the problem of compression can be formulated as the calculation of 
\begin{equation*}
L_{min} = \min_{\sigma \in {\cal S}} \sum_{i=1}^N p_i l_{\sigma(i)}.
\end{equation*}
It is easy to see that this is a radically simple case of QAP-KB.
First, notice that $L$ can be expressed in vectorial form as
\begin{equation*}
L = \mathbf{p} \mathbf{l}.
\end{equation*}
Assuming that $\mathbf{p}$ and $\mathbf{l}$ are column vectors, $L$ can be expressed in matrix form as $1 \times N$ row matrix $W = \mathbf{p}^T$ and 
a $N \times 1$ column matrix $D = \mathbf{l}$.
The solution to this minimization problem is \cite{Ferrer2019c}
\begin{equation}
L_{min} = \overleftarrow{\mathbf{p}} \overrightarrow{\mathbf{l}} = \overrightarrow{\mathbf{p}} \overleftarrow{\mathbf{l}},
\label{eq:minimum_average_magnitude}
\end{equation}
where $\overrightarrow{\mathbf{v}}$ is a sorting of vector $\mathbf{v}$ in increasing order and $\overleftarrow{\mathbf{v}}$ is a sorting of vector $\mathbf{v}$ in decreasing order.
$L_{min}$ (\cref{eq:minimum_average_magnitude}) is a straightforward consequence of a classic result by Hardy, Littlewood and Polya \cite{Hardy1934a} that is revisited in this article as Theorem \ref{theo:Hardy_et_al}.
In this setting, the expected value of $L$ over all assignments is simply the unweighted average \cite{Petrini2022b} 
\begin{equation*}
L_r = \frac{1}{N} \sum_{i=1}^N l_i.
\end{equation*}
\cite{Ferrer2019c} investigated generalizations of the compression problem above where the $l_i's$ are extracted from a multiset of size $N$ or greater.  

\section{Upper and lower bounds}
\label{app:bounds}

\subsection{Preliminaries}

Recall that $m$ is the number of non-zero probability orders, $1 \leq m \leq N$, and that $\pi_i$ is the value of the $i$-th largest value among the $p_i$'s. 
\begin{property}
If $1 \leq i \leq m$, we have 
\begin{equation}
\pi_i \geq \frac{1 - (i - 1)\pi_1}{m - i + 1}
\label{eq:ranked_probability_lower_bound}
\end{equation}
and 
\begin{equation}
\pi_1 \geq \frac{1}{m} \geq \frac{1}{N}.
\label{eq:maximum_probability_lower_bound}
\end{equation}
\end{property}
\begin{proof}
We have 
\begin{align*}
T & =    \sum_{i = 1}^N \pi_i \\
  & =    \sum_{j = 1}^{i - 1} \pi_j + \sum_{j = i}^m \pi_j \\
  & \leq (i-1) \pi_1 + (m - i + 1)\pi_i. 
\end{align*}
The condition $T = 1$ produces \cref{eq:ranked_probability_lower_bound} and then $\pi_1 \geq 1/m$.
\end{proof}
\Cref{eq:ranked_probability_lower_bound} is the finite support set complementary of Debowski's harmonic bound, i.e. $\pi_i \leq 1/i$ \cite{Debowski2025a}.

We revisit a useful result by Hardy, Littlewood and Polya \cite{Hardy1934a}.
Consider two vectors 
\begin{align*}
\mathbf{a} & = (a_1,..., a_i,..., a_N) \\
\mathbf{b} & = (b_1,..., b_i,..., b_N)
\end{align*}
of real numbers. 
A rearrangement of a vector $\mathbf{a}$ is a new vector resulting by applying a permutation function $\sigma$ to $\mathbf{a}$, that is
\begin{equation*}
\mathbf{a}' = (a_{\sigma(1)},..., a_{\sigma(i)},..., a_{\sigma(N)}).
\end{equation*}
We use $\overrightarrow{\mathbf{a}}$ and $\overleftarrow{\mathbf{a}}$ to refer to a rearrangement of $\mathbf{a}$ in ascending order or descending order, respectively. That is
\begin{align*}
& \overrightarrow{a}_1 \leq \overrightarrow{a}_2 \leq ... \leq\overrightarrow{a}_N\\
& \overleftarrow{a}_1 \geq \overleftarrow{a}_2 \geq ... \geq \overleftarrow{a}_N.
\end{align*}
\begin{theorem}{Rearrangement inequality.}
\label{theo:Hardy_et_al}
Given two vectors of real numbers $\mathbf{a}$ and $\mathbf{b}$ \cite[Theorem 368]{Hardy1934a},
\begin{equation*}
\sum_{i=1}^N \overrightarrow{a}_i \overleftarrow{b}_i \leq \mathbf{a'} \mathbf{b'} = \sum_{i=1}^N a'_i b'_i \leq \sum_{i=1}^N \overrightarrow{a}_i \overrightarrow{b}_i.
\end{equation*}  
\end{theorem}

\subsection{Bounds of $\averageswapdistance$}

It has been shown that the range of variation of $\averageswapdistance$ satisfies \citep{Franco2024a}
\begin{equation*}
\averageswapdistance_{low} \leq \averageswapdistance \leq \averageswapdistance_{up} 
\end{equation*}
with 
\begin{align}
\averageswapdistance_{low} & = 0 \nonumber \\
\averageswapdistance_{up}  & = \min\left(\averageswapdistance_{max}, d_{max} \dominanceindex\right),
\label{eq:Franco_Sanchez_et_al_bounds}
\end{align}
where $\averageswapdistance_{max}$ is the maximum value of $\averageswapdistance_{max}$ for any $p_i$'s, $d_{max}$ is the diameter of the permutohedron and $\dominanceindex = 1 - S$, where $S$ is the Simpson index (\cref{eq:Simpson_index}).
Here we will refine the previous bounds and provide tighter bounds for $n = 3$.

$\averageswapdistance$ can be rewritten equivalently as
\begin{equation*}
\averageswapdistance = \sum_{d=0}^{d_{max}} P(d)d = \sum_{d=1}^{d_{max}} P(d)d,
\end{equation*}
where $P(d)$ is the probability mass of a certain distance $d$, i.e.  
\begin{equation}
P(c) = \sum_{i=1}^{N} \sum_{\substack{j=1 \\ d_{ij} = c}}^{N} p_i p_j.
\label{eq:probability_of_distance}
\end{equation}
and 
\begin{equation}    
\sum_{d=0}^{d_{max}} P(d) = 1.
\label{eq:normalization}
\end{equation}
We have $P(0) = S$.

The fact that $1/m \leq S \leq 1$ \citep{Franco2024a} has useful consequences. First, the range of variation of $\dominanceindex$ is \citep{Franco2024a}
\begin{equation}
0 \leq \dominanceindex \leq 1 - \frac{1}{m} \leq 1 - \frac{1}{N}.
\label{eq:range_of_variation_of_dominance_index}
\end{equation}
Second, 
\begin{align*}    
\sum_{d=1}^{d_{max}} P(d) & =    1 - P(0) \\
                          & \leq 1 - 1/m. 
\end{align*}
Hence 
\begin{align}
P(d) & \leq 1 - 1/m \nonumber \\
     & \leq 1 - 1/N \label{eq:probability_upper_bound}
\end{align}
for $1 \leq d \leq d_{max}$.
 
$\averageswapdistance$ can be expressed equivalently as \citep[Proof of Property C.3]{Franco2024a}
\begin{align}
\averageswapdistance & = d_{max}(1 - P(0)) - \sum_{d=1}^{d_{max} - 1} P(d) (d_{max} - d) \nonumber \\
                     & = d_{max}\dominanceindex - \sum_{d=1}^{d_{max} - 1} P(d) (d_{max} - d).
\label{eq:alternative2_average_swap_distance}                 
\end{align} 
Hence, it is obvious that $\averageswapdistance \leq d_{max}\dominanceindex$ (\cref{eq:Franco_Sanchez_et_al_bounds}). Next we provide useful expressions to bound $\averageswapdistance$ above and below.  

\begin{property}
\begin{align}
\averageswapdistance & = 2\dominanceindex + (d_{max} - 2)P(d_{max}) - P(1) + \nonumber \\
                     &   \sum_{d=3}^{d_{max} - 1} P(d)(d - 2) \label{eq:averageswapdistance_for_bounding_2_times_dominance_index} \\
                     & = \dominanceindex + (d_{max} - 1)P(d_{max}) + \nonumber \\ 
                     &   \sum_{d=2}^{d_{max} - 1} P(d)(d - 1). \label{eq:averageswapdistance_for_bounding_1_time_dominance_index}
\end{align}
\end{property}
\begin{proof}
Thanks to \cref{eq:normalization}, we have 
\begin{align*}
P(d) & = 1 - P(0) - \sum_{d'=1}^{d - 1} P(d') - \sum_{d'=d+1}^{d_{max}} P(d') \\
     & = \dominanceindex - \sum_{d'=1}^{d - 1} P(d') - \sum_{d'=d+1}^{d_{max}} P(d') 
\end{align*}
for $1 \leq d \leq d_{max}$. Then 
\begin{align}
P(1) & = \dominanceindex - \sum_{d=2}^{d_{max}} P(d) \label{eq:from_normalization1} \\
P(2) & = \dominanceindex - P(1) - \sum_{d=3}^{d_{max}} P(d). \label{eq:from_normalization2}
\end{align}
\Cref{eq:alternative2_average_swap_distance} gives 
\begin{align*}
\averageswapdistance & = d_{max}\dominanceindex - (d_{max} - 1)P(1) - (d_{max} - 2)P(2) - \\
                     &   \sum_{d = 3}^{d_{max} - 1} P(d) (d_{max} - d). 
\end{align*}
The application of \cref{eq:from_normalization2} leads to
\begin{align*}
\averageswapdistance & = 2 \dominanceindex - P(1) +  (d_{max} - 2)P(d_{max}) +  \\
                     &   (d_{max} - 2) \sum_{d=3}^{d_{max} - 1} P(d) - \sum_{d=3}^{d_{max} - 1} P(d)(d_{max} - d) \\
                     & = 2 \dominanceindex + (d_{max} - 2)P(d_{max}) - P(1) + \\
                     &   \sum_{d=3}^{d_{max} - 1} P(d)(d - 2).
\end{align*}
\Cref{eq:alternative2_average_swap_distance} also gives 
\begin{equation*}
\averageswapdistance = d_{max}\dominanceindex - (d_{max} - 1)P(1) - \sum_{d = 2}^{d_{max} - 1} P(d) (d_{max} - d).
\end{equation*}
The application of \cref{eq:from_normalization1} leads to 
\begin{align*}
\averageswapdistance & = \dominanceindex + (d_{max} - 1) P(d_{max}) + \\ 
                     &   (d_{max} - 1) \sum_{d=2}^{d_{max} - 1} P(d) - \sum_{d=2}^{d_{max} - 1} (d_{max} - d) \\
                     & = \dominanceindex + (d_{max} - 1)P(d_{max}) + \sum_{d=2}^{d_{max} - 1} P(d)(d - 1).\\
\end{align*}

\end{proof}

When $n=3$, \cref{eq:probability_upper_bound} gives $P(1), P(2), P(3) \leq 5/6$. The next property presents better upper bounds.

\begin{property}
\label{prop:probability_of_distance_upper_bound}
When $n = 3$,
\begin{align*}
4 Z \leq P(1), P(2) & \leq \min(2S, 1 - S) \leq \frac{2}{3} \\
2 Z \leq P(3)       & \leq \min(S, 1 - S)  \leq \frac{1}{2},
\end{align*}
where 
\begin{align*}
S & = P(0) \\
Z & = \pi_1 \pi_6 + \pi_2 \pi_5 + \pi_3 \pi_4.
\end{align*} 
\end{property}

\begin{proof}
It follows from \cref{eq:normalization} that
\begin{equation*}
P(1), P(2), P(3) \leq 1 - P(0) = 1 - S.
\end{equation*}
Following the conventions in \cref{fig:permutohedron_1_to_6} for $n = 3$, it is easy to see that (\cref{eq:probability_of_distance})
\begin{align*}
P(1) & = 2(p_1 p_2 + p_2 p_3 + p_3 p_4 + p_4 p_5 + p_5 p_6 + p_6 p_1) \\
P(2) & = 2(p_1 p_3 + p_2 p_4 + p_3 p_5 + p_4 p_6 + p_5 p_1 + p_6 p_2) \\
P(3) & = p_1 p_4 + p_2 p_5 + p_3 p_6 + p_4 p_1 + p_5 p_2 + p_6 p_3 \\
     & = 2(p_1 p_4 + p_2 p_5 + p_3 p_6).
\end{align*}
When $1 \leq d \leq 3$, $P(d)$ is of the form (\cref{eq:average_swap_distance})
\begin{equation} 
P(d) = \alpha \mathbf{p} \mathbf{p}', \label{eq:distance_probability_template}
\end{equation}
where $\alpha \in \{1, 2\}$,  
\begin{equation*} 
\mathbf{p} = (p_1,...,p_i,...,p_6),
\end{equation*}
and $\mathbf{p}'$ is a suitable rearrangement of $\mathbf{p}$.  
For $P(1)$, $\alpha = 2$ and $\mathbf{p}' = (p_2, p_3, p_4, p_5, p_6, p_1)$.
For $P(2)$, $\alpha = 2$ and $\mathbf{p}' = (p_3, p_4, p_5, p_6, p_1, p_2)$.
For $P(3)$, $\alpha = 1$ and $\mathbf{p}' = (p_4, p_5, p_6, p_1, p_2, p_3)$.
Applying Theorem \ref{theo:Hardy_et_al} to \cref{eq:distance_probability_template}, we obtain 
\begin{equation*}
2\alpha Z \leq P(d) \leq \alpha P(0) =  \alpha S,
\end{equation*}
which gives
\begin{align*}
4 Z & \leq P(1), P(2) \\
2 Z & \leq P(3).
\end{align*}
Combining the results obtained so far and $P(d) \leq 1 - P(0)$, we obtain 
\begin{align*}
P(1), P(2) & \leq \min(2 P(0), 1 - P(0)) = \min(2S, 1 - S) \\
P(3)       & \leq \min(P(0), 1 - P(0)) = \min(S, 1 - S).
\end{align*}
A simple graphical analysis shows
\begin{align*}
\min(2S, 1 - S) & \leq \frac{2}{3} \\
\min(S, 1 - S)  & \leq \frac{1}{2}.
\end{align*} 
\end{proof}

Thanks to \cref{eq:averageswapdistance_for_bounding_1_time_dominance_index}, it is obvious that $\averageswapdistance \geq \dominanceindex$ and hence we can refine \cref{eq:Franco_Sanchez_et_al_bounds} with upper and lower bounds of $\averageswapdistance$ that are both functions of $\dominanceindex$ as
\begin{align}
\averageswapdistance_{low}(\dominanceindex) & = \dominanceindex \nonumber \\ 
\averageswapdistance_{up}(\dominanceindex)  & = \min\left(\averageswapdistance_{max}, d_{max} \dominanceindex\right).
\label{eq:Franco_Sanchez_et_al_bounds_refined}
\end{align}

The following property shows the dependence on the Simpson index of the range of variation of $\averageswapdistance$ (\cref{eq:Franco_Sanchez_et_al_bounds}) when $n = 3$.

\begin{property}
When $n = 3$, 
\begin{align}
\averageswapdistance_{low}(\dominanceindex) & = \max\left(\dominanceindex, 2\dominanceindex - 2/3 \right) \nonumber \\ 
\averageswapdistance_{up}(\dominanceindex)  & = \min\left(\dominanceindex + 1, 2\dominanceindex + 1/2, \frac{3}{2}\right).
\label{eq:Franco_Sanchez_et_al_bounds_refined_n_3}
\end{align}
\end{property}
\begin{proof}
When $n = 3$, equations \cref{eq:averageswapdistance_for_bounding_2_times_dominance_index} 
and \cref{eq:averageswapdistance_for_bounding_1_time_dominance_index} give
\begin{align}
\averageswapdistance & = 2 \dominanceindex - P(1) + P(3) \label{eq:averageswapdistance_for_bounding_2_times_dominance_index_n_3} \\
\averageswapdistance & = \dominanceindex + P(2) + 2P(3) \label{eq:averageswapdistance_for_bounding_1_time_dominance_index_n_3}
\end{align}
while Property \ref{prop:probability_of_distance_upper_bound} gives
\begin{align}
P(1), P(2)   & \leq 2/3 \nonumber \\
P(3)         & \leq 1/2 \label{eq:easy_inequalities}
\end{align}
and
\begin{align*}
P(2) + 2P(3) & =    P(2) + P(3) + P(3) \\
             & \leq P(2) + P(3) + P(0) \\ 
             &      \mbox{       ($P(3) \leq P(0)$ by Property \ref{prop:probability_of_distance_upper_bound})} \\
             & =    1 - P(1). 
\end{align*}
Given \cref{eq:averageswapdistance_for_bounding_2_times_dominance_index_n_3}, $P(1) \geq 0$ and $P(3) \leq 1/2$ yield $\averageswapdistance \leq 2 \dominanceindex + 1/2$ whereas $P(1) \leq 2/3$ and $P(3) \geq 0$ yield 
$\averageswapdistance \geq 2 \dominanceindex - 2/3$.
Given \cref{eq:averageswapdistance_for_bounding_1_time_dominance_index_n_3}, $P(2) + 2P(3) \leq 1 - P(1)$ and $P(1) \geq 0$ yield $\averageswapdistance \leq \dominanceindex + 1$ whereas $P(1), P(3) \geq 0$ yield 
$\averageswapdistance \geq \dominanceindex$.
Combining the results above, we retrieve \cref{eq:Franco_Sanchez_et_al_bounds_refined_n_3} noting that the range of variation of $\dominanceindex$ is $0 \leq \dominanceindex \leq 5/6$ (\cref{eq:range_of_variation_of_dominance_index}). 

Notice it is possible to refine the bounds in \cref{eq:Franco_Sanchez_et_al_bounds_refined_n_3} by involving other expressions for $\averageswapdistance$ or replacing \cref{eq:easy_inequalities} by more precise bounds for $P(d)$ from Property \ref{prop:probability_of_distance_upper_bound}. However, the analysis by cases becomes unnecessarily complicated for our simple goal of showing that the variation of $\averageswapdistance$ is limited by simple functions of $S$. 
\end{proof}

\subsection{Lower bounds of $\Omega$}

\begin{figure}
\includegraphics[width = \linewidth]{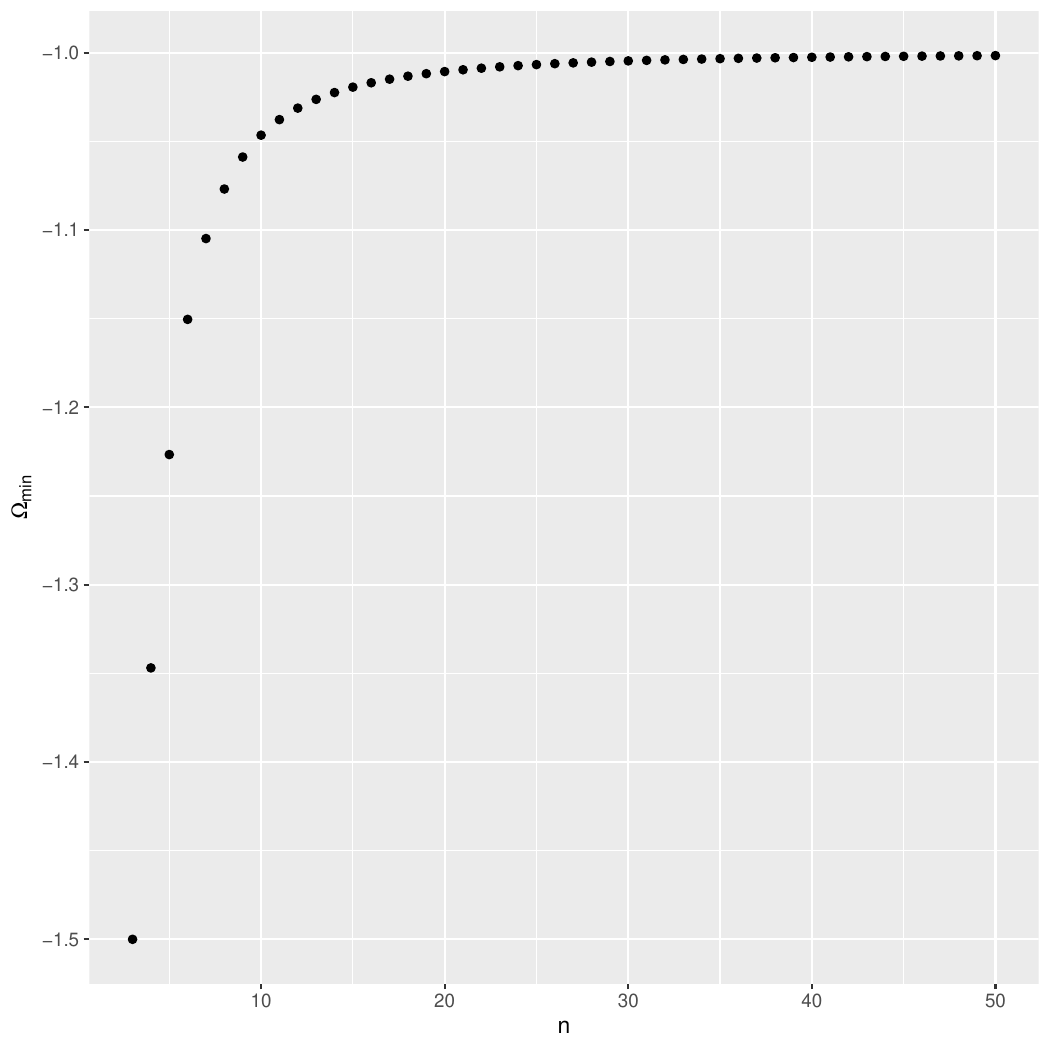}
\caption{\label{fig:Omega_min_m_2} $\Omega_{min}$, a tight lower bound of $\Omega$, when $m$, the number of non-zero probability orders, is $2$.}
\end{figure}

The following property shows sheds light on the minimum value of $\Omega$ in specific conditions.

\begin{property}
Consider $\Omega_{min}$(n), or shortly $\Omega_{min}$, as the minimum value of $\Omega$ that an arrangement can achieve on sequences of length $n$.
\begin{itemize}
\item
When $n \geq 3$ and $m = 2$, we have $\Omega = \Omega_{min}$ if and only if the two non-zero probability orders are at distance $d_{max}$, independently of the value of $\pi_1$ and $\pi_2$. In addition, $\Omega_{min}$ is a monotonically increasing function of $n$ (\cref{fig:Omega_min_m_2})
such that
\begin{equation*}
-\frac{3}{2} \leq \Omega_{min}(n) \leq - 1,
\end{equation*}
and 
\begin{align*}
\Omega_{min}(3) & = -\frac{3}{2}, \Omega_{min}(4) = -\frac{66}{49} \\ 
\Omega_{min}(5) & = -\frac{590}{481}.
\end{align*}
\item
When $n = 3$ and $m = 3$, $\Omega_{min}(3) \rightarrow -3/2^+$ at the bottom edge of the line $\pi_2 = 1 - \pi_1$ for $1/2 < \pi_1 < 1$ (\cref{fig:Omega_min_map}). 
\end{itemize}
\end{property}

\begin{proof}
By definition (\cref{eq:optimality_template}),  
\begin{equation}
\Omega \geq \Omega_{min} = \frac{\averageswapdistance_{r} - \averageswapdistance_{max}}{\averageswapdistance_{r} - \averageswapdistance_{min}}. 
\label{eq:optimality_lower_bound}
\end{equation}

When $n \geq 3$ and $m = 2$ (equations \cref{eq:average_swap_distance} and \cref{eq:Simpson_index}), 
\begin{align*}
\averageswapdistance & = 2\delta\pi_1(1 - \pi_1) \\
S                    & = \pi_1^2 + (1 - \pi_1)^2 \\
\dominanceindex      & = 2\pi_1(1 - \pi_1),
\end{align*}
where $\delta$ is the swap distance between the two non-zero probability orders. 
As $1\leq \delta \leq d_{max}$ (where $d_{max}$ is the diameter of the permutohedron), it follows that
\begin{align*}
\averageswapdistance_{min} & = 2\pi_1(1 - \pi_1) \\
\averageswapdistance_{max} & = 2 d_{max} \pi_1(1 - \pi_1).
\end{align*}
Besides, \cref{eq:expected_average_swap_distance_random_permutation} gives
\begin{equation*}
\averageswapdistance_{r} = \frac{N}{N - 1} d_{max}\pi_1(1 - \pi_1).
\end{equation*}
Then \cref{eq:optimality_lower_bound} with the shorthands 
\begin{align*}
c       & = \frac{N}{N - 1} \\
d_{max} & = \frac{n(n-1)}{2} \mbox{~~~(\cref{eq:diameter_of_permutohedron})}
\end{align*}
becomes
\begin{align*}
\Omega_{min}(n) & = \frac{\left(  c - 2 \right) d_{max} \pi_1(1 - \pi_1)}{\left( c d_{max} - 2 \right)  \pi_1(1 - \pi_1)} \\
                & = \frac{\left(  c - 2 \right) d_{max}}{c d_{max} - 2}
\end{align*}                         
independently of the value of $\pi_1$.
It is easy to see that $\Omega_{min}(n)$ is a monotonically increasing function of $n$ (\cref{fig:Omega_min_m_2}). 
$\Omega_{min}(n)$ is minimized at $n = 3$ ($\Omega_{min}(3) = -\frac{3}{2}$) and from there onward it converges to $-1$ as $n \rightarrow \infty$.

Consider the case $n = 3$. When $m = 3$, $\averageswapdistance$ is of the form
\begin{equation*}
\averageswapdistance =  2 \mathbf{d} \mathbf{q},
\end{equation*}
where the vector $\mathbf{d}$ is any permutation of the multiset of swap distances $\{d_1, d_2, d_3\}$ and the vector $\mathbf{q}$ is 
\begin{equation*}
\mathbf{q} = (\pi_1 \pi_2, \pi_1 \pi_3, \pi_2 \pi_3). 
\end{equation*}
It is easy to see that $\mathbf{q}$ is sorted in descending order, that is $\pi_1 \pi_2 \geq \pi_1 \pi_3$ and $\pi_1 \pi_3 \geq \pi_2 \pi_3$ by the definition of $\pi_i$.

To minimize $\averageswapdistance$, the arrangement must be contiguous and then $\{d_1, d_2, d_3\} = \{1, 1, 2\}$ (Appendix \ref{app:contiguous}) while $\mathbf{d}$ must be ascending since $q$ is descending (Theorem \ref{theo:Hardy_et_al}), hence 
\begin{align*}
\averageswapdistance_{min} & = 2 (1, 1, 2) (\pi_1 \pi_2, \pi_1 \pi_3, \pi_2 \pi_3) \\
                           & = 2 (\pi_1 \pi_2 + \pi_1 \pi_3 + 2\pi_2 \pi_3).
\end{align*} 
To maximize $\averageswapdistance$, the arrangement must be discontiguous and then $\{d_1, d_2, d_3\} = \{2, 2, 2\}$ or $\{d_1, d_2, d_3\} = \{1, 2, 3\}$ (Appendix \ref{app:contiguous}) while $\mathbf{d}$ must be descending since $q$ is descending (Theorem \ref{theo:Hardy_et_al}), hence we obtain two candidates for $\averageswapdistance_{max}$
\begin{align*}
\averageswapdistance_{max}^1 & = 2 (2, 2, 2) (\pi_1 \pi_2, \pi_1 \pi_3, \pi_2 \pi_3) \\
                             & = 4 (\pi_1 \pi_2 + \pi_1 \pi_3 + 2\pi_2 \pi_3) \\
\averageswapdistance_{max}^2 & = 2 (3, 2, 1) (\pi_1 \pi_2, \pi_1 \pi_3, \pi_2 \pi_3) \\
                             & = 2 (3\pi_1 \pi_2 + 2\pi_1 \pi_3 + \pi_2 \pi_3).                             
\end{align*} 
It is easy to see that $\averageswapdistance_{max}^2 \geq \averageswapdistance_{max}^1$. Notice 
\begin{equation*}
\averageswapdistance_{max}^2 - \averageswapdistance_{max}^1 \geq 2(\pi_1\pi_2 - \pi_2\pi_3) \geq 0
\end{equation*}
by the definition of $\pi_i$. Therefore, $\averageswapdistance_{max} = \averageswapdistance_{max}^2$. 

\begin{figure}
\includegraphics[width = \linewidth]{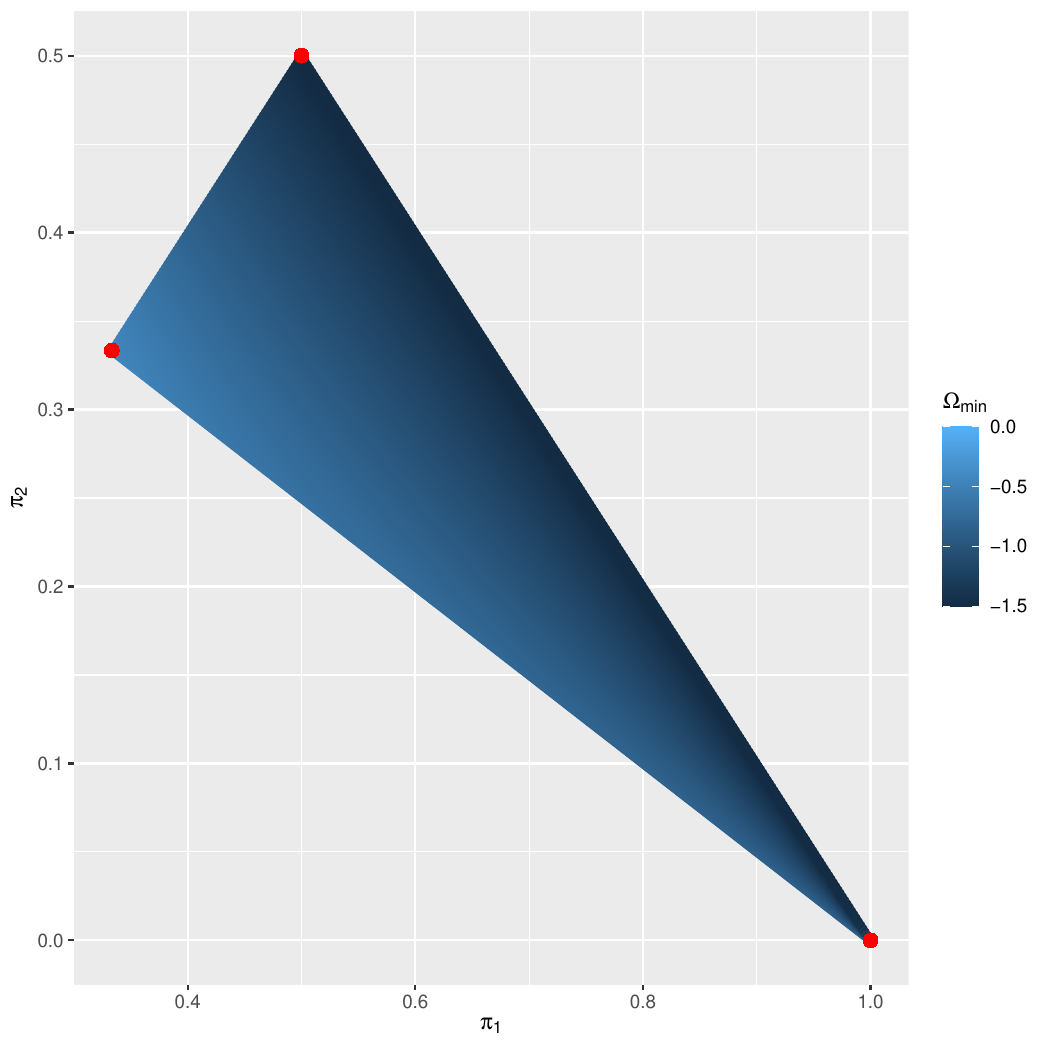}
\caption{\label{fig:Omega_min_map} $\Omega_{min}$ (\cref{eq:optimality_lower_bound}) as a function of $\pi_1$ and $\pi_2$. Red circles are used to indicate the corners of the triangle at $(1/3, 1/3)$, $(1/2, 1/2)$ and $(1, 0)$. Within the triangle, $\Omega_{min}$ was calculated by scanning exhaustively with a step of $10^{-3}$ for both variables. }
\end{figure}

As $\pi_3 = 1 - \pi_1 - \pi_2$, we can express $\averageswapdistance_{min}$, $\averageswapdistance_{max}$ and $\averageswapdistance_r$ as functions of $\pi_1$ and $\pi_2$ only, that is
\begin{align*}
\averageswapdistance_{min} & = 2\left(\pi_1 \pi_2 + (1 - \pi_1 - \pi_2)(\pi_1 + 2\pi_2)\right) \\
\averageswapdistance_{max} & = 2\left(3\pi_1 \pi_2 + (1 - \pi_1 - \pi_2)(2\pi_1 + \pi_2)\right) \\
\averageswapdistance_r     & = \frac{9}{5}\left(1 - \pi_1^2 - \pi_2^2 - (1 - \pi_1 - \pi_2)^2\right).
\end{align*}  
Then we plug the previous results into $\Omega_{min}$ (\cref{eq:optimality_lower_bound}) and minimize $\Omega_{min}(3)$ within the domain defined by the following inequalities
\begin{enumerate}
\item
$1/3 \leq \pi_1 < 1$. The lower bound follows from \cref{eq:maximum_probability_lower_bound}. The upper bound follows from the fact that $\pi_1 = 1$ would contradict $m = 3$. $\Omega$ is not defined for $m = 1$ (section \ref{subsec:how}).
\item
$(1 - \pi_1)/2 \leq \pi_2 < \min(\pi_1, 1 - \pi_1)$. $(1 - \pi_1)/2 \leq \pi_2$ follows from $\pi_3 = 1 - \pi_1 - \pi_2 \leq \pi_2$. $\pi_2 < \min(\pi_1, 1 - \pi_1)$ follows from combining $\pi_2 \leq \pi_1$ according to the definition of $\pi_i$ and $\pi_1 + \pi_2 < 1$. Notice that $\pi_1 + \pi_2 = 1$ would contradict $m = 3$. 
\end{enumerate}
Then, it is easy to see that the variation of $(\pi_1, \pi_2)$ is confined within a triangle (\cref{fig:Omega_min_map}).
$\Omega_{min}(3)$ is minimized at the bottom edge of the segment that joins the corners $(1/2, 1/2)$ and $(1, 0)$ of the triangle, more precisely the line $\pi_2 = 1 - \pi_1$ for $1/2 < \pi_1 < 1$ (\cref{fig:Omega_min_map}). On that line, one has $\pi_3 = 0$ which implies $m = 2$. Therefore $\Omega_{min}(3) = -3/2$ on that line according to analysis of the case $m = 2$. On the below edge of that line, that is in turn defined by the line, $\pi_2 = 1 - \pi_1 - \epsilon$, where $\epsilon$ is a small positive number (e.g. $\epsilon = 10^{-6}$), it is easy to check numerically that $\Omega_{min}(3) \rightarrow -3/2^+$.


\end{proof}

Consider the case $n = 3$. In light of the results above, its easy to conclude that negative values of $\Omega$ close to $-3/2 + \epsilon$, where $\epsilon$ is some small positive number, can be obtained for any $m$, not only $m = 2$ or $m = 3$. For $m=6$, it suffices to assign two non-zero probabilities totaling $1-\zeta$, where $\zeta$ is a small positive number, to a couple of vertices at distance $3$ in the permutohedron and then set the probabilities of the remainder of the vertices to $\zeta/4$ and choose a sufficiently small $\zeta$. Similar arguments can be made for $m=4$ or $m=5$. Thus the open problem is if $\Omega$ can be smaller than $-3/2$ when $n = 3$ and $m > 3$.   

\section{Optimal arrangements}
\label{app:optimal_arrangements}

\subsection{Preliminaries}

$p_i$ is the probability of the order corresponding to vertex $i$. We also assume that the distribution of order probabilities is fixed, i.e. the probability vector 
\begin{equation*}
\mathbf{p}=(p_1,..., p_i,...,p_N)
\end{equation*} 
can only be one of the permutations of some probability vector $\pi$ such that 
\begin{equation*}
\boldsymbol{\pi}=(\pi_1,..., \pi_i,..., \pi_N)
\end{equation*}
and $\pi_i \geq \pi_{i+1}$ for $1 \leq i \leq N - 1$.  
Put differently, the multiset of probabilities of any vector $\mathbf{p}$ is constant, it is always $\{\pi_1,..., \pi_i,..., \pi_N\}$. An arrangement is some permutation of $\boldsymbol{\pi}$, that is
\begin{equation*}
(\pi_{\sigma(1)},..., \pi_{\sigma(i)},..., \pi_{\sigma(6)}).
\end{equation*}
We will use $\mathbf{p}$, $\mathbf{p}'$, $\mathbf{p}''$ to refer to arrangements.  

\subsection{Local optimality}

In this subsection we are interested in the local optima, namely the optimal arrangements with respect to $\localaverageswapdistance$ (\cref{eq:local_average_swap_distance}) given some $i$. 

We define $\localaverageswapdistance(\sigma)$ as the value of $\localaverageswapdistance$ for some permutation $\sigma$ as
\begin{equation*}
\localaverageswapdistance = \sum_{j=1}^N p_\sigma(j) d_{ij}.
\end{equation*}
Then, the minimum and the maximum of $\localaverageswapdistance$ over all possible permutations are
\begin{align*}
\localaverageswapdistance_{min} & = \min_{\sigma \in {\cal S}} \localaverageswapdistance(\sigma) \\
\localaverageswapdistance_{max} & = \max_{\sigma \in {\cal S}} \localaverageswapdistance(\sigma).
\end{align*}

The next property presents tight bounds for $\localaverageswapdistance$ and loose bounds for $\averageswapdistance$.
\begin{property}
\label{prop:local_optimality}
\begin{equation*}
\localaverageswapdistance_{min} = \boldsymbol{\pi}\overrightarrow{\boldsymbol{d}} \leq \averageswapdistance, \localaverageswapdistance \leq \localaverageswapdistance_{max} = \boldsymbol{\pi}\overleftarrow{\boldsymbol{d}}.  
\end{equation*}
When $n = 3$,
\begin{align}
\averageswapdistance, \localaverageswapdistance & \geq \localaverageswapdistance_{min} = \pi_2 + \pi_3 + 2(\pi_4 + \pi_5) + 3\pi_6 \nonumber \\ 
\averageswapdistance, \localaverageswapdistance & \leq \localaverageswapdistance_{max} = 3\pi_1 + 2(\pi_2 + \pi_3) + \pi_4 + \pi_5. \label{eq:local_average_swap_distance_min_max_n=3}     
\end{align}
Following the vertex labeling in \cref{fig:permutohedron_1_to_6}, and assuming that $i = 1$, the minimum arrangements satisfy the total preorder of vertices in \cref{fig:locally_optimal_probability_arrangement}, that is given by
\begin{align}
& p_1 \geq p_2 \geq p_6 \geq p_3 \geq p_5 \geq p_4 \nonumber \\  
& p_6 \geq p_2 \nonumber \\ 
& p_4 \geq p_5. \label{eq:total_preorder_local_minimum}
\end{align}
The maximum arrangements satisfy the total preorder of vertices that is given by
\begin{align}
& p_1 \leq p_2 \leq p_6 \leq p_3 \leq p_5 \leq p_4 \nonumber \\
& p_6 \leq p_2 \nonumber \\
& p_5 \leq p_3. \label{eq:total_preorder_local_maximum}
\end{align}
\end{property}
\begin{proof}
We define $\mathbf{d}$ as vector containing a permutation of the multiset of the swap distances of a vertex, say $i$, to any other vertex.
$\localaverageswapdistance$ can be expressed as $\localaverageswapdistance = \boldsymbol{\pi}\mathbf{d}$.
By theorem \ref{theo:Hardy_et_al},
\begin{equation*}
\localaverageswapdistance_{min} = \boldsymbol{\pi}\overrightarrow{\boldsymbol{d}} \leq \localaverageswapdistance \leq \localaverageswapdistance_{max} = \boldsymbol{\pi}\overleftarrow{\boldsymbol{d}}. 
\end{equation*}
When $n = 3$, the multiset is $\{0, 1, 1, 2, 2, 3\}$ and then equation \cref{eq:local_average_swap_distance_min_max_n=3} follows immediately. Then \cref{eq:average_swap_distance_using_local_average_swap_distance} gives
\begin{align*}
\averageswapdistance & \geq \sum_{i=1}^N p_i \localaverageswapdistance_{min} \\
                     & \geq \localaverageswapdistance_{min} \sum_{i=1}^N p_i \\ 
                     & \geq \localaverageswapdistance_{min}.                      
\end{align*}
Analogously, $\averageswapdistance \leq \localaverageswapdistance_{max}$. 
\cref{eq:total_preorder_local_minimum} and \cref{eq:total_preorder_local_maximum} follow by noting that the multiset induces a ranking of the vertices with ties which in turn induces a total preorder of the vertices of the permutohedron. The directed graph that defines the total preorder has cycles, that is the edges $(2,6)$ and $(6,2)$ form a cycle; the edges (3,5) and (5,3) form another cycle (\cref{fig:locally_optimal_probability_arrangement} (a)) and then the transitive reduction is not unique \cite{Aho1972a}. \Cref{fig:locally_optimal_probability_arrangement} (b) shows only one of the Hasse diagrams. The transitive reduction was computed with the tool {\tt tred} of the Graphviz toolset (\url{https://graphviz.org/}).
\end{proof}

The previous result has some immediate implications summarized in the following corollary. 
\begin{corollary}{Consequences of local optimality.}
When $n = 3$, the arrangements that minimize $\localaverageswapdistance$ satisfy the following conditions
\begin{enumerate}
\item
{\em Radiation from vertex $i$}. Probability remains the same or decreases as one moves away from vertex $i$ (\cref{fig:locally_optimal_probability_arrangement}). More precisely, vertex $i$ is assigned $\pi_1$. The vertices at distance 1 are assigned probabilities $\pi_2$ and $\pi_3$, the vertex at distance $2$ are assigned probabilities $\pi_4$ and $\pi_5$ and the vertex at distance $3$ is assigned $\pi_6$.
\item 
{\em Adjacency of the most likely orders}. The two most likely orders are linked in the permutohedron (\cref{fig:locally_optimal_probability_arrangement}).
\item
{\em Contiguity of the non-zero probability orders}. The $m$ non-zero probability orders are contiguous in the permutohedron, namely they form a path $m$ vertices in the permutohedron (\cref{fig:contiguity}).
\end{enumerate}
\label{cor:local_optimality}
\end{corollary}

\begin{proof}
1 and 2 follow trivially from $\localaverageswapdistance_{min} = \boldsymbol{\pi}\overrightarrow{\boldsymbol{d}}$ (Property \ref{prop:local_optimality}). 3 follows from 1 because a discontiguity would violate one of the inequalities described in \cref{fig:locally_optimal_probability_arrangement}.
\end{proof}

Figure \ref{fig:optimal_probability_arrangement} shows the four total orders of vertices that minimize $\localaverageswapdistance$ assuming that $i$ is the left-most vertex.
Later on (Appendix \ref{subsec:structure_of_optimal_arrangements}), we will show that the arrangements that minimize $\averageswapdistance$ correspond to the two total orders in \cref{fig:optimal_probability_arrangement} (a) and (b). Thus Corollary \ref{cor:local_optimality} also applies to arrangements that minimize $\averageswapdistance$. A critical difference is that the minima of $\localaverageswapdistance$ define total preorders (\cref{eq:total_preorder_local_minimum}) while the minima of $\averageswapdistance$ define total orders.
 
\subsection{Swapping probabilities}

We assume $n=3$. We label the vertices of the permutohedron with numbers from 1 to 6 in a clockwise sense as in \cref{fig:permutohedron_1_to_6}. 

\begin{figure}
\centering
\includegraphics[width = \lonelypermutohedron \linewidth]{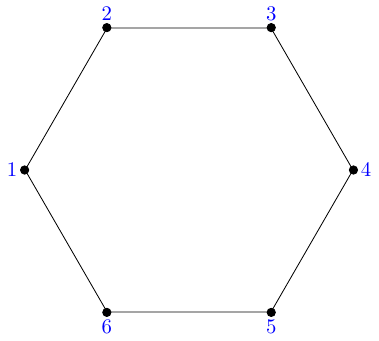}
\caption{\label{fig:permutohedron_1_to_6} The permutohedron of order 3. Vertices are labeled with consecutive numbers in a clockwise sense.} 
\end{figure}

An arrangement $\mathbf{p}$ can be transformed into a another arrangement $\mathbf{p}'$ by swapping the probabilities of a pair of vertices $x$ and $y$. 
$p_i'$, the probability of $i$ after that swap, is 
\begin{equation}
p_i' = \left\{
          \begin{array}{ll}
          p_y & \mbox{if } i = x \\ 
          p_x & \mbox{if } i = y \\ 
          p_i & \mbox{if } i \notin \{x, y \}.
          \end{array}
       \right.
       \label{eq:new_probability_after_swap}
\end{equation}           
We also use $\averageswapdistance'$ to denote the new value of $\averageswapdistance$ after a swap of a pair of probabilities and $\averageswapdistance''$ to denote the increment in $\averageswapdistance$ after two swaps of a pair of probabilities. The following property gives the increment in $\averageswapdistance$ after one or two of such swaps. 

\begin{property}
\label{prop:increment_in_average_swap_distance_after_swap}
If the probabilities of vertex $x$ and vertex $y$ are swapped, the variation in $\averageswapdistance$ after the swap is
\begin{equation}
\averageswapdistance' - \averageswapdistance = 2 (p_y - p_x)\sum_{j \in [1,6]\setminus\{x,y\}}  p_j (d_{xj} - d_{yj})
\label{eq:increment_in_average_swap_distance_after_swap} 
\end{equation}
\end{property}

\begin{proof}   
Recall \cref{eq:average_swap_distance}. We can see $\averageswapdistance$ as the sum of the square matrix $A = \left\{ d_{ij} p_i p_j \right\}$ and $\averageswapdistance'$ as the sum of another square matrix $A'$.
Then
\begin{eqnarray*}
\averageswapdistance' - \averageswapdistance & = & \underbrace{\sum_{j \in [1,6]} (p_x'p_j' - p_xp_j)d_{xj}}_{i=x} + \\ 
                                             &   & \underbrace{\sum_{j \in [1,6]} (p_y'p_j' - p_yp_j)d_{yj}}_{i=y} + \\ 
                                             &   & \underbrace{\sum_{i \in [1,6]\setminus\{x,y\}} (p_i'p_x' - p_ip_x)d_{ix}}_{j=x} + \\
                                             &   & \underbrace{\sum_{i \in [1,6]\setminus\{x,y\}} (p_i'p_y' - p_ip_y)d_{iy}}_{j=y} \\
                                             & = & \sum_{j \in [1,6]\setminus\{x\}} (p_x'p_j' - p_xp_j)d_{xj} + \\
                                             &   & \sum_{j \in [1,6]\setminus\{y\}} (p_y'p_j' - p_yp_j)d_{yj} + \\ 
                                             &   & \sum_{j \in [1,6]\setminus\{x,y\}} (p_x'p_j' - p_xp_i)d_{xj} + \\
                                             &   & \sum_{j \in [1,6]\setminus\{x,y\}} (p_y'p_j' - p_yp_j)d_{yj} \\
                                             &   & \mbox{ ($d_{jj} = 0$ and change of index} \\
                                             &   & \mbox{  in 3rd and 4th summations)} \\  
                                             & = & 2\sum_{j \in [1,6]\setminus\{x,y\}} \left[ (p_x'p_j' - p_xp_j)d_{xj} + \right. \\
                                             &   & \left. (p_y'p_j' - p_yp_j)d_{yj} \right] + \\
                                             &   & 2(p_x'p_y' - p_xp_y)d_{xy} \\
                                             &   & \mbox{~~~~~~(regrouping and factoring out).}
\end{eqnarray*}
Thanks to \cref{eq:new_probability_after_swap}, $p_i' p_j' = p_i p_j$ if $|\{i, j\} \cap \{x, y \}| \neq 1$. 
Then
\begin{eqnarray*}                                               
\averageswapdistance' - \averageswapdistance & = & 2\sum_{j \in [1,6]\setminus\{x,y\}} \left[ (p_y p_j - p_xp_j)d_{xj} + \right. \\
                                             &   & \left(p_x p_j - p_yp_j)d_{yj} \right] \\
                                             &   & \mbox{~~~~~~($p_x'=p_y$, $p_y'=p_x$ and $p_j' = p_j$} \\
                                             &   & \mbox{~~~~~~~for $j \notin \{x,y\}$)} \\
                                             & = & 2(p_y - p_x)\sum_{j \in [1,6]\setminus\{x,y\}} p_j (d_{xj} - d_{yj}) \\ 
                                             &   & \mbox{~~~~~~(factoring out).}                                                                                                                                                                              
\end{eqnarray*} 
\end{proof}

\begin{property}
\label{prop:increment_in_average_swap_distance_after_2_swaps}
Consider two pairs of vertices of the permutohedron, $(x, y)$ and $(w, z)$, such that 
\begin{enumerate}
\item
$x \neq y$, $w \neq z$ 
\item
They are independent, i.e. $\{x, y\} \cap \{w, z\} = \emptyset$. 
\end{enumerate}
Let $\averageswapdistance''$ be the new value of $\averageswapdistance$ after the probabilities of vertices of the pair $(x, y)$ are swapped and the probabilities of the vertices in the pair $(w, z)$ are also swapped.
The vertices that are not swapped are $\gamma$ and $\delta$, that is $\{\gamma, \delta\} = [1, 6]\setminus\{w, x, y, z\}$.
Then
\begin{eqnarray*}
\frac{\averageswapdistance'' - \averageswapdistance}{2} & = & (p_y - p_x) \left[ p_\gamma (d_{x\gamma} - d_{y\gamma}) + p_\delta (d_{x\delta} - d_{y\delta}) + \right. \\
                                                        &   & \left. p_w (d_{xw} - d_{yw}) + p_z (d_{xz} - d_{yz}) \right] + \\   
                                                        &   & (p_z - p_w) \left[ p_\gamma (d_{w\gamma} - d_{z\gamma}) + p_\delta (d_{w\delta} - d_{z\delta}) + \right. \\
                                                        &   & \left. p_x (d_{yw} - d_{yz}) + p_y (d_{xw} - d_{xz}) \right] 
\end{eqnarray*} 
regardless of the order of the swaps. 
\end{property}

\begin{proof} 
We define $\Delta_i$ as the increment in $\averageswapdistance$ after the $i$-th swap.
Then 
\begin{eqnarray*}
\averageswapdistance' & = & \averageswapdistance + \Delta_1 \\
\averageswapdistance'' & = & \averageswapdistance + \Delta_2 \\
\averageswapdistance'' - \averageswapdistance & = & \Delta_1 + \Delta_2.
\end{eqnarray*} 
Without loss of generality, suppose that the 1st swap is on $(x,y)$. We define $p_i'$ as the new value of $p_i$ after the 1st swap. 
Then Property \ref{prop:increment_in_average_swap_distance_after_swap} yields
\begin{widetext}
\begin{eqnarray*}
\frac{\Delta_1}{2} & = & (p_y  - p_x)  \sum_{j \in [1,6]\setminus\{x,y\}}     p_j  (d_{xj} - d_{yj}) \\
                   & = & (p_y  - p_x)  \left[\sum_{j \in \{\gamma, \delta\}}  p_j  (d_{xj} - d_{yj}) + p_w(d_{xw} - d_{yw}) + p_z(d_{xz} - d_{yz})\right] \\
\frac{\Delta_2}{2} & = & (p_z' - p_w') \sum_{j \in [1,6]\setminus\{w,z\}}     p_j' (d_{wj} - d_{zj}) \\ 
                   & = & (p_z  - p_w)  \left[\sum_{j \in \{\gamma, \delta\}}  p_j  (d_{wj} - d_{zj}) + p_x(d_{wy} - d_{zy}) + p_y(d_{wx} - d_{zx})\right] \\ 
                   &   & \mbox{(\cref{eq:new_probability_after_swap}).}
\end{eqnarray*}
Finally,
\begin{eqnarray*}
\frac{\Delta_1 + \Delta_2}{2} & = & \sum_{j \in \{\gamma, \delta\}}  p_j \left[ (p_y - p_x) (d_{xj} - d_{yj}) + (p_z - p_w) (d_{wj} - d_{zj}) \right] +\\
                              &   & (p_y - p_x) \left[ p_w (d_{xw} - d_{yw}) + p_z (d_{xz} - d_{yz}) \right] + (p_z - p_w) \left[ p_x (d_{yw} - d_{yz}) + p_y (d_{xw} - d_{xz}) \right] \\
                              & = & (p_y - p_x) p_\gamma (d_{x\gamma} - d_{y\gamma}) + (p_z - p_w) p_\gamma (d_{w\gamma} - d_{z\gamma}) + (p_y - p_x) p_\delta (d_{x\delta} - d_{y\delta}) + (p_z - p_w) p_\delta (d_{w\delta} - d_{z\delta}) +\\                              
                              &   & (p_y - p_x) \left[ p_w (d_{xw} - d_{yw}) + p_z (d_{xz} - d_{yz}) \right] + (p_z - p_w) \left[ p_x (d_{yw} - d_{yz}) + p_y (d_{xw} - d_{xz}) \right] \\                              
                              & = & (p_y - p_x) \left[ p_\gamma (d_{x\gamma} - d_{y\gamma}) + p_\delta (d_{x\delta} - d_{y\delta}) + p_w (d_{xw} - d_{yw}) + p_z (d_{xz} - d_{yz}) \right] + \\
                              &   & (p_z - p_w) \left[ p_\gamma (d_{w\gamma} - d_{z\gamma}) + p_\delta (d_{w\delta} - d_{z\delta}) + p_x (d_{yw} - d_{yz}) + p_y (d_{xw} - d_{xz}) \right].                                     
\end{eqnarray*}
\end{widetext}
\end{proof}

\subsection{Creation of $/$-structures and $\wedge$-structures}

Assuming $n = 3$, we examine the optimality of two structures that can be formed in an arrangement by swapping the probabilities of vertices: $/$-structures and $\wedge$-structures. 
 
An arrangement has a slash-structure, shortly $/$-structure, formed by vertices $u$ and $v$ if 
\begin{enumerate}
\item
$u$ and $v$ are adjacent in the permutohedron.
\item
$v$ follows $u$ in a clockwise sense, that is $v = (u + 1) \bmod 6$.
\item
$u$ and $v$ have the two largest probabilities, that is 
\begin{equation*}
p_u, p_v \geq \pi_2.
\end{equation*}
If there are no probability ties, $u$ and $v$ are the two vertices with the two highest probabilities. 
\end{enumerate}
Languages that have a pair of dominant orders that are adjacent in the permutohedron have a $/$-structure \citep{Ferrer2016c}.

An arrangement has a $\wedge$-structure formed by vertices $t$, $u$, $v$ if   
\begin{enumerate}
\item
$t$, $u$, $v$ form a path in the permutohedron. That is $t$ and $v$ are neighbors of $u$ in the permutohedron.  
\item
They are consecutive in a clockwise sense, that is $u = (t + 1) \bmod 6$ and $v = (u + 1) \bmod 6$.  
\item
These three vertices have the three largest probabilities,
\begin{equation*}
\min(p_t, p_u, p_v) \geq \pi_3. 
\end{equation*}
If there are no probability ties, $s$, $u$ and $v$ are the three vertices with the three highest probabilities.
\item
$u$, the vertex in the middle, has maximum probability ($p_u = \pi_1$).
\item
The vertex in the middle is at least as likely as the two vertices at the ends,
$p_u \geq p_t, p_v$. Hence the term wedge-up structure, since probability may increase (but never decrease) from $t$ to $u$ and may drop (but never increase) from $u$ to $v$. 
\end{enumerate}
If an arrangement has a $\wedge$-structure formed by $t$, $u$ and $v$, then $t$ and $u$ or $u$ and $v$ form a $/$-structure. 

The following lemma shows that an arrangement that does not have a $/$-structure can be transformed into an arrangement that has a $/$-structure and lower or equal $\averageswapdistance$ in just one swap provided that certain conditions are satisfied. 

\begin{lemma}[Creation of a $/$-structure with lower or equal $\averageswapdistance$ in one swap]
\label{lem:creation_of_slash_structure_with_1_swap}
Without loss of generality, suppose that $1$ is a vertex with maximum probability ($p_1 = \pi_1$). Consider a vertex $x$, a neighbor of $1$ in the permutohedron. That is $x \in \{2, 6\}$. 
Consider another vertex $y$ that has the second largest probability and is not a neighbor of $1$. That is $p_y = \pi_2$ and $y \in [1, 6]\setminus\{1, 2, 6\}$.
By swapping the probabilities of the pair $(x, y)$ one will produce a $/$-structure formed by $1$ and $x$ such that $\averageswapdistance' \leq \averageswapdistance$ if and only if one the following conditions is met
\begin{enumerate}
\item
$y = 4$.
\item
$y = 3$. [$x = 2$ and $p_4 + p_5 \leq p_1 + p_6$] or [$x = 6$ and $p_2 + p_4 \leq p_1 + p_5$].
\item
$y = 5$. [$x = 2$ and $p_4 + p_6 \leq p_1 + p_3$] or [$x = 6$ and $p_3 + p_4 \leq p_1 + p_2$].
\end{enumerate}
We have $\averageswapdistance' = \averageswapdistance$ if and only if
\begin{enumerate}
\item
$y = 4$. [$x = 2$ and ($p_1 = p_5$ or $p_2 = p_4$)] or [$x = 6$ and ($p_1 = p_3$ or $p_4 = p_6$)].
\item
$y = 3$. [$x = 2$ and ($p_2 = p_3$ or $p_4 + p_5 = p_1 + p_6$)] or [$x = 6$ and ($p_3 = p_6$ or $p_2 + p_4 = p_1 + p_5$)].
\item
$y = 5$. [$x = 2$ and ($p_2 = p_5$ or $p_4 + p_6 = p_1 + p_3$)] or [$x = 6$ and ($p_5 = p_6$ or $p_3 + p_4 = p_1 + p_2$)].
\end{enumerate}
\end{lemma}
\begin{proof}
We aim to swap the probabilities of a pair of vertices $(x, y)$ to obtain a $/$-structure formed by vertices $6$, $1$ or by vertices $1$ and $2$. 
We examine all possible vertices that can act as $y$.
If $y$ is $2$ or $6$ then $6$ and $1$ or $1$ and $2$ already form a $/$-structure. 

First, we consider the case $y = 4$ and hence $d_{1y} = 3$.
If $x = 2$, \cref{eq:increment_in_average_swap_distance_after_swap} gives
\begin{eqnarray*}                                               
\averageswapdistance' - \averageswapdistance & = & 4\underbrace{(p_4 - p_2)}_{\geq 0} \underbrace{(p_5 - p_1)}_{\leq 0}.                                             
\end{eqnarray*}
Since $p_2 \leq p_4$ and $p_1 \geq p_5$, we have $\averageswapdistance' - \averageswapdistance \leq 0$ with equality if and only if $p_1 = p_5$ or $p_2 = p_4$.
If $x = 6$, we obtain
\begin{eqnarray*}   
\averageswapdistance' - \averageswapdistance & = & 4\underbrace{(p_4 - p_6)}_{\geq 0} \underbrace{(p_3 - p_1)}_{\leq 0}.                                             
\end{eqnarray*}
Since $p_6 \leq p_4$ and $p_1 \geq p_3$ we have $\averageswapdistance' - \averageswapdistance \leq 0$ with equality if and only if $p_1 = p_3$ or $p_4 = p_6$.
Therefore, a $/$-structure with smaller or equal $\averageswapdistance$ is always formed by swapping the probability of $4$ with that of one of the neighbors of $1$.

Second, we consider the case $y = 3$.
If $x=2$, \cref{eq:increment_in_average_swap_distance_after_swap} gives
\begin{eqnarray*}                                               
\averageswapdistance' - \averageswapdistance & = & 2\underbrace{(p_3 - p_2)}_{\geq 0} (p_4 + p_5 - p_1 - p_6).                                             
\end{eqnarray*}
Since $p_2 \leq p_3$, we have $\averageswapdistance' - \averageswapdistance \leq 0$ if and only if $p_4 + p_5 \leq p_1 + p_6$.
We have $\averageswapdistance = \averageswapdistance'$ if and only if $p_2 = p_3$ or $p_4 + p_5 = p_1 + p_6$.
If $x=6$, \cref{eq:increment_in_average_swap_distance_after_swap} gives
\begin{eqnarray*}                                               
\averageswapdistance' - \averageswapdistance & = & 2\underbrace{(p_3 - p_6)}_{\geq 0}(p_4 + p_2 - p_1 - p_5).                                              
\end{eqnarray*}
Since $p_6 \leq p_3$, we have $\averageswapdistance' - \averageswapdistance \leq 0$ if and only if $p_4 + p_2 \leq p_1 + p_5$.  
We have $\averageswapdistance = \averageswapdistance'$ if and only if $p_3 = p_6$ or $p_4 + p_2 = p_1 + p_5$.


Third, we consider the case $y = 5$. If $x = 2$, \cref{eq:increment_in_average_swap_distance_after_swap} gives
\begin{eqnarray*}
\averageswapdistance' - \averageswapdistance & = & 2\underbrace{(p_5 - p_2)}_{\geq 0}(p_4 + p_6 - p_1 - p_3).
\end{eqnarray*}
Since $p_2 \leq p_5$, we have $\averageswapdistance' - \averageswapdistance \leq 0$ if and only if $p_4 + p_6 \leq p_1 + p_3$.
We have $\averageswapdistance = \averageswapdistance'$ if and only if $p_2 = p_5$ or $p_4 + p_6 = p_1 + p_3$.
If $x = 6$, \cref{eq:increment_in_average_swap_distance_after_swap} gives
\begin{eqnarray*}
\averageswapdistance' - \averageswapdistance & = & 2\underbrace{(p_5 - p_6)}_{\geq 0} (p_3 + p_4 - p_1 - p_2 ).                                             
\end{eqnarray*}
Since $p_6 \leq p_5$, we have $\averageswapdistance' - \averageswapdistance \leq 0$ if and only if $p_3 + p_4 \leq p_1 + p_2$.
We have $\averageswapdistance = \averageswapdistance'$ if and only if $p_5 = p_6$ or $p_3 + p_4 = p_1 + p_2$.
\end{proof}

The following corollary provides an interpretation for the previous Lemma \ref{lem:creation_of_slash_structure_with_1_swap}.

\begin{corollary}[Creation of a $/$-structure with lower or equal $\averageswapdistance$ in one swap]
\label{cor:creation_of_slash_structure_with_1_swap}
Consider the setting of Lemma \ref{lem:creation_of_slash_structure_with_1_swap}. 
Suppose that vertex $1$ does not form a $/$-structure with any of its neighbors ($2$ or $6$). Then we have $\averageswapdistance' = \averageswapdistance$ if and only if
\begin{enumerate}
\item
$y = 4$. [$x = 2$ and $p_1 = p_5$] or [$x = 6$ and $p_1 = p_3$].
\item
$y = 3$. [$x = 2$ and $p_4 + p_5 = p_1 + p_6$] or [$x = 6$ and $p_2 + p_4 = p_1 + p_5$].
\item
$y = 5$. [$x = 2$ and $p_4 + p_6 = p_1 + p_3$)] or [$x = 6$ and $p_3 + p_4 = p_1 + p_2$)].
\end{enumerate}
In addition, suppose that $y$ does not form a $/$-structure with any of its neighbors. Then we have 
\begin{enumerate}
\item
$\averageswapdistance' < \averageswapdistance$ for $y = 4$ 
\item
$\averageswapdistance' = \averageswapdistance$ if and only if
  \begin{enumerate}
  \item
  $y = 3$. [$x = 2$ and $p_4 + p_5 = p_1 + p_6$] or [$x = 6$ and $p_2 + p_4 = p_1 + p_5$].
  \item
  $y = 5$. [$x = 2$ and $p_4 + p_6 = p_1 + p_3$)] or [$x = 6$ and $p_3 + p_4 = p_1 + p_2$)].
  \end{enumerate}
\end{enumerate}  
\end{corollary}
\begin{proof}
The proof consists of examining the conditions for $\averageswapdistance' = \averageswapdistance$ in the 2nd part of \cref{lem:creation_of_slash_structure_with_1_swap} and annotating the pair of vertices that form a $/$-structure when a certain condition is satisfied by means of a curly brace below the condition. Then
\begin{enumerate}
\item
$y = 4$. [$x = 2$ and ($\underbrace{p_1 = p_5}_{(4, 5)}$ or $\underbrace{p_2 = p_4}_{(1, 2)}$)] or [$x = 6$ and ($\underbrace{p_1 = p_3}_{(3,4)}$ or $\underbrace{p_4 = p_6}_{(1, 6)}$)].
\item
$y = 3$. [$x = 2$ and ($\underbrace{p_2 = p_3}_{(1, 2), (2, 3)}$ or $p_4 + p_5 = p_1 + p_6$)] or [$x = 6$ and ($\underbrace{p_3 = p_6}_{(1, 6)}$ or $p_2 + p_4 = p_1 + p_5$)].
\item
$y = 5$. [$x = 2$ and ($\underbrace{p_2 = p_5}_{(1, 2)}$ or $p_4 + p_6 = p_1 + p_3$)] or [$x = 6$ and ($\underbrace{p_5 = p_6}_{(1, 6), (5, 6)}$ or $p_3 + p_4 = p_1 + p_2$)].
\end{enumerate}
The previous information allows one to generate the conditions in the 1st part or the 2nd part of the statement of this corollary by dropping the false conditions.
\end{proof}

The next lemma analyses the impact on $\averageswapdistance$ of destroying a $\wedge$-structure by swapping the probabilities of the neighbors of its central vertex with those of vertices outside the $\wedge$-structure, or equivalently, the consequences of creating a $\wedge$-structure by 
\begin{enumerate}
\item
Locating a vertex with maximum probability, say $u$, that will become the center of the $\wedge$-structure 
\item
Swapping the probabilities of the neighbors of $u$ with those of two vertices with the 2nd and 3rd largest probability that are not adjacent to $u$.     
\end{enumerate}

\begin{lemma}[Creation or destruction of a $\wedge$-structure in two independent swaps]
\label{lem:creation_or_destruction_of_wedge_structure_with_2_swaps}
Consider an arrangement that has a $\wedge$-structure formed by vertices $t$, $u$ and $v$ and average swap distance $\averageswapdistance_\wedge$.
The pairs of swaps of probabilities of the form $(t, y)$ and $(v, z)$ such that $y \neq z$ and $y, z \in [1, 6]\setminus \{t, u, v\}$ yield an arrangement with average swap distance $\averageswapdistance$ such that $\averageswapdistance \geq \averageswapdistance_\wedge$. 
\end{lemma}
\begin{proof}
We define $\gamma$ and $\delta$ as the only two vertices that are not involved in any swap, that is $\{\gamma, \delta\} = [1, 6]\setminus \{t, v, y, z\}$. 
Without any loss of generality, we set $\gamma = u = 1$, hence $t = 6$ and $v = 2$. By swapping the probabilities of the pair $(2, y)$ and those of the pair $(6, z)$ with $y \neq z$, Property \ref{prop:increment_in_average_swap_distance_after_2_swaps} gives
\begin{align*}
\frac{\averageswapdistance - \averageswapdistance_\wedge }{2} = & (p_y - p_2) \left[ p_1 (1 - d_{1y}) + p_\delta (d_{2\delta} - d_{y\delta}) + \right. \nonumber \\ 
                                 & \left. p_6 (2 - d_{6y}) + p_z (d_{2z} - d_{yz}) \right] + \nonumber \\   
                                 & (p_z - p_6) \left[ p_1 (1 - d_{1z}) + p_\delta (d_{6\delta} - d_{z\delta}) + \right. \nonumber \\ 
                                 & \left. p_2 (d_{6y} - d_{yz}) + p_y (2 - d_{2z}) \right]. 
\end{align*}
Now we consider all possible values of $\delta$. Given $\gamma$ and $\delta$, the values of $y$ and $z$ follow. 
Suppose that $\delta = 3$. $y=4$ and $z = 5$ lead to
\begin{eqnarray*}
\frac{\averageswapdistance - \averageswapdistance_\wedge }{2} & = & 2(p_4 - p_2)(p_5 - p_1) + \\
                              &   & (p_5 - p_6)(p_2 + p_3 - p_1 - p_4) \\
                              & = & 2(p_4 - p_2)(p_5 - p_1) + (p_5 - p_6)(p_3 - p_1) + \\ 
                              &   & (p_5 - p_6)(p_2 - p_4) \\
                              & = & \underbrace{(p_4 - p_2)}_{\leq 0}\underbrace{(p_5 + p_6 - 2p_1)}_{\leq 0} + \\
                              &   & \underbrace{(p_5 - p_6)}_{\leq 0}\underbrace{(p_3 - p_1)}_{\leq 0} \\
                              & \geq & 0
\end{eqnarray*} 
with equality if and only if
\begin{equation*}
(p_2 = p_4 \mbox{ or } p_5 + p_6 = 2p_1) \mbox{ and } (p_5 = p_6 \mbox{ or } p_1 = p_3).
\end{equation*}
$y=5$ and $z = 4$ lead to
\begin{eqnarray*}
\frac{\averageswapdistance - \averageswapdistance_\wedge }{2} & = & \underbrace{(p_5 - p_2)}_{\leq 0}(p_4 + p_6 - p_1 - p_3) + \\ 
                              & = & 2\underbrace{(p_4 - p_6)}_{\leq 0}\underbrace{(p_3 - p_1)}_{\leq 0} \\
                              & \geq & 0
\end{eqnarray*}
if and only if $p_4 + p_6 \leq p_1 + p_3$. Then we have $(\averageswapdistance - \averageswapdistance_\wedge )/2 = 0$ if and only if 
\begin{equation*}
(p_2 = p_5 \mbox{ or } p_4 + p_6 = p_1 + p_3) \mbox{ and } (p_4 = p_6 \mbox{ or } p_1 = p_3).
\end{equation*}

Suppose that $\delta = 4$. $y = 3$ and $z = 5$ lead to
\begin{eqnarray*}
\frac{\averageswapdistance - \averageswapdistance_\wedge }{2} & = & (p_3 - p_2)(p_4 - p_1 - p_6) + \\ 
                              &   & (p_5 - p_6)(p_2 + p_4 - p_1 - p_3) \\
                              & = & (p_3 - p_2)(p_4 - p_1 - p_6) +  \\ 
                              &   & (p_5 - p_6)(p_2 - p_3) + (p_5 - p_6)(p_4 - p_1) \\
                              & = & \underbrace{(p_3 - p_2)}_{\leq 0}\underbrace{(p_4 - p_1 - p_5)}_{\leq 0} + \\ 
                              &   & \underbrace{(p_5 - p_6)}_{\leq 0}\underbrace{(p_4 - p_1)}_{\leq 0} \\
                              & \geq & 0.
\end{eqnarray*}
$y = 5$ and $z = 3$ lead to
\begin{eqnarray*}
\frac{\averageswapdistance - \averageswapdistance_\wedge }{2} & = & (p_5 - p_2)(p_4 + p_6 - p_1 - p_3) + \\ 
                              &   & (p_3 - p_6)(p_4 + p_5 - p_1 - p_2) \\
                              & = & (p_5 - p_2)(p_4 - p_1) + (p_5 - p_2)(p_6 - p_3) + \\ 
                              &   & (p_3 - p_6)(p_4 - p_1) + (p_3 - p_6)(p_5 - p_2) \\
                              & = & \underbrace{(p_5 + p_3 - p_2 - p_6)}_{\leq 0}\underbrace{(p_4 - p_1)}_{\leq 0} \\
                              & \geq & 0
\end{eqnarray*}
with equality if and only if 
\begin{equation*}
p_3 + p_5 = p_2 + p_6 \mbox{ or } p_1 = p_4.
\end{equation*}

Suppose that $\delta = 5$. We have $\averageswapdistance - \averageswapdistance_\wedge \geq 0$ by symmetry with respect to the case $\delta = 3$. 

\end{proof}

\begin{theorem}[Creation of a $\wedge$-structure with lower or equal $\averageswapdistance$ in one or two swaps]
\label{theo:creation_of_wedge_structure_with_1_or_2_swaps}
Every arrangement with average swap distance $\averageswapdistance$ that does not have a $\wedge$-structure can be transformed into an arrangement that has a $\wedge$-structure and average swap distance $\averageswapdistance_\wedge$ such that $\averageswapdistance_\wedge \leq \averageswapdistance$ 
by swapping the probabilities of one or two pairs of vertices. 
\end{theorem}
\begin{proof}
Without loss of generality, suppose that $1$ is a vertex with maximum probability. Its neighbors in the permutohedron are $2$ and $6$. Since there is no $\wedge$-structure, $p_2 \leq \pi_4$ or $p_6 \leq \pi_4$.
We consider the following cases 
\begin{enumerate}  
\item
Neither $6$ and $1$ nor $1$ and $2$ form a $/$-structure, i.e. $p_2 < \pi_2$ and $p_6 < \pi_2$. Then we apply Lemma \ref{lem:creation_or_destruction_of_wedge_structure_with_2_swaps} to obtain an arrangement with a $\wedge$-structure such that $\averageswapdistance_\wedge \leq \averageswapdistance$ by swapping the probabilities of two pairs of vertices, i.e. $(2, y)$ and $(6, z)$ where $y$ and $z$ are two vertices such that $y, z \in \{3, 4, 5\}$, $y \neq z$ and $\pi_3 \leq p_y, p_z \leq \pi_2$.
\item
$6$ and $1$ form a $/$-structure but $1$ and $2$ do not. Then $p_6 = \pi_2$ but $p_2 \leq \pi_3$. We will build a $\wedge$-structure
by swapping the probabilities of a pair of vertices of the form $(2, y)$ where $y$ is a vertex such that $y = \pi_3$ and $y \in \{3, 4, 5\}$. 
We consider the following cases
  \begin{enumerate}
  \item
  $y = 4$. Then Lemma \ref{lem:creation_of_slash_structure_with_1_swap} indicates that $\averageswapdistance$ cannot increase.
  \item
  $y = 3$. Then the condition required by Lemma \ref{lem:creation_of_slash_structure_with_1_swap} to not increase $\averageswapdistance$ is $p_4 + p_5 \leq p_1 + p_6$, which holds trivially since $p_1 + p_6 = \pi_1 + \pi_2$ while $p_4 + p_5 \leq \pi_1 + \pi_2$. 
  \item 
  $y = 5$. Then the condition required by Lemma \ref{lem:creation_of_slash_structure_with_1_swap} to not increase $\averageswapdistance$ is $p_4 + p_6 \leq p_1 + p_3$, which becomes $p_4 + \pi_2 \leq \pi_1 + p_3$ in this context.
  If $p_4 \leq p_3$, we are done. If $p_4 \geq p_3$, we look for an alternative way of producing a $\wedge$-structure. The path formed by vertices $(5, 6, 1)$ is such that the corresponding probabilities for each vertex in the path are $(\pi_3, \pi_2, \pi_1)$. Therefore, we build a $\wedge$-structure formed by vertices $5$, $6$, and $1$ by swapping the probabilities of vertices $1$ and $6$. Then the increment in $\averageswapdistance$ is
  \begin{align*}
  \averageswapdistance' - \averageswapdistance & = 2(p_1 - p_6) (p_2 + p_3 - p_4 - p_5) \\
                                               & = 2(\underbrace{\pi_1 - \pi_2}_{\geq 0}) (\underbrace{p_2 + p_3 - p_4 - \pi_3}_{\leq 0} \leq 0
  \end{align*}  
  since $p_2 \leq \pi_4$ and $p_3 \leq p_4$.
  \end{enumerate}  
\item
$1$ and $2$ form a $/$-structure but $6$ and $1$ do not. This case is the symmetric of the previous case.
\end{enumerate}
\end{proof}

\subsection{The structure of optimal arrangements}

\label{subsec:structure_of_optimal_arrangements}

The following theorem characterizes the minimum $\averageswapdistance$ arrangements, that is arrangements that minimize $\averageswapdistance$ over all possible arrangements.
An eccentric vertex of a vertex $u$ is a vertex $v$ that is at maximum distance of $u$ ($d_{uv}$ is maximum).
 
\begin{theorem}
\label{theo:optimal_arrangements_theorem}
When $n = 3$, the minimum $\averageswapdistance$ arrangements are such that
\begin{enumerate}
\item
They contain a $\wedge$-structure formed by vertices $t, u, v$ and hence $p_u = \pi_1$.
\item
$w$, the eccentric vertex of $u$ (i.e. $d_{u,w} = 3$) has minimum probability, i.e. $p_w = \pi_6$. 
\item
If $p_v = \pi_2$, then $p_t = \pi_3$, the neighbors of $v$ have probabilities $\pi_1$ and $\pi_4$ and the neighbors of $t$ have probabilities $\pi_1$ and $\pi_5$.
If $p_v = \pi_3$, then $p_t = \pi_2$, the neighbors of $v$ have probabilities $\pi_1$ and $\pi_5$ and the neighbors of $t$ have probabilities $\pi_1$ and $\pi_4$.
\end{enumerate}
Suppose that the vertices of the permutohedron are labeled as in \cref{fig:permutohedron_1_to_6}. 
In any minimum  $\averageswapdistance$ arrangement, there is a vertex, say vertex 1, such that satisfies the total order of vertices
\begin{equation}
(1,2,6,3,5,4)
\label{eq:optimal_total_order1}
\end{equation}
given by
\begin{equation*}
p_1 \geq p_2 \geq p_6 \geq p_3 \geq p_5 \geq p_4,
\end{equation*}
or its symmetric
\begin{equation}
(1,6,2,5,3,4)
\label{eq:optimal_total_order2}
\end{equation}
given by
\begin{equation*}
p_1 \geq p_6 \geq p_2 \geq p_5 \geq p_3 \geq p_4
\end{equation*}
as illustrated in \cref{fig:optimal_probability_arrangement} (a) and (b).
\end{theorem}

\begin{proof}
We wish to find the arrangements that minimize $\averageswapdistance$. There are $6! = 720$ possible arrangements. Instead of working with arrangements, we work with total orders of vertices of the permutohedron.  
If there are no probability ties, there is a one-to-one correspondence between total orders and arrangements. In general, every total order may correspond to more than one arrangement. 

The strategy of the proof consists of incrementally defining the total orders that minimize $\averageswapdistance$ until only two total orders with same $\averageswapdistance$ remain (equations \cref{eq:optimal_total_order1} and \cref{eq:optimal_total_order2}). The process consists of (a) building partial orders of increasing strength until only four total orders are possible and then (b) selecting the only total orders that minimize $\averageswapdistance$. The derivation follows. 

\begin{enumerate}
\item
{\em Creation of a $\wedge$-structure between vertices $t$, $u$ and $v$.} If the arrangement does not have a $\wedge$-structure, we locate a vertex with maximum probability, say $u$, and create a $\wedge$-structure formed by vertices $t$, $u$, $v$ by swapping the probabilities of one or two pairs of vertices following the procedure described in Theorem \ref{theo:creation_of_wedge_structure_with_1_or_2_swaps}. Then, without loss of generality, suppose that $u = 1$. Then, by the definition of the $\wedge$-structure, we have a partial order of the vertices given by
\begin{equation}
p_1 \geq p_2, p_6 \geq p_3, p_4, p_5.
\label{eq:relationship_among_probabilities_in_wedge_structure}
\end{equation}
\item
{\em Assigning the minimum probability to the eccentric vertex of $u$.} We consider each of the possible least likely vertices, that is $3$, $4$ or $5$. If $4$ is a least likely vertex, we are done because $d_{1,4} = 3$. We will show that $\averageswapdistance$ cannot increase if $3$ or $5$ are swapped with $4$ so that vertex 4 becomes a least likely vertex.
If we swap 3 and 4, Property \ref{prop:increment_in_average_swap_distance_after_swap} gives that the increment in $\averageswapdistance$ after the swap is
\begin{equation*}
\averageswapdistance' - \averageswapdistance = 2(\underbrace{p_4 - p_3}_{\geq 0})(\underbrace{p_5 + p_6 - p_1 - p_2}_{\leq 0}) \leq 0
\end{equation*}
since $p_3 \leq p_4$, $p_6 \leq p_1$ and $p_5 \leq p_2$ (\cref{eq:relationship_among_probabilities_in_wedge_structure}). If we swap 5 and 4, Property \ref{prop:increment_in_average_swap_distance_after_swap} gives that the increment in $\averageswapdistance$ after the swap is
\begin{equation*}
\averageswapdistance' - \averageswapdistance = 2(\underbrace{p_4 - p_5}_{\geq 0})(\underbrace{p_2 + p_3 - p_1 - p_6}_{\leq 0}) \leq 0
\end{equation*}
since $p_5 \leq p_4$, $p_2 \leq p_1$ and $p_3 \leq p_6$ (\cref{eq:relationship_among_probabilities_in_wedge_structure}). Therefore the minimum $\averageswapdistance$ arrangements are given by 
\begin{equation*}
p_1 \geq p_2, p_6 \geq p_3, p_5 \geq p_4.
\end{equation*}
\item
{\em Selection of the total orders that minimize $\averageswapdistance$.}
The previous equation defines a partial order of vertices that is compatible with four total orders. Two of the total orders are defined in \cref{eq:optimal_total_order1} and \cref{eq:optimal_total_order2}. Two other total orders are given by 
\begin{align}
& p_1 \geq p_6 \geq p_2 \geq p_3 \geq p_5 \geq p_4 \label{eq:optimal_total_order3} \\
& p_1 \geq p_2 \geq p_6 \geq p_5 \geq p_3 \geq p_4. \label{eq:optimal_total_order4}
\end{align}
We wish to determine which of the four total orders minimizes $\averageswapdistance$. To do so, we treat the total order in \cref{eq:optimal_total_order1} as canonical and produce the other total orders by swapping pairs of probabilities. If we swap both probabilities of the pairs $(2, 6)$ and $(3, 5)$ on the canonical total order (\cref{eq:optimal_total_order1}), we obtain an arrangement equivalent to the total order (\cref{eq:optimal_total_order2}). It is easy to see that both arrangements have the same $\averageswapdistance$ by symmetry. 
The argument can be validated mechanically applying Property \ref{prop:increment_in_average_swap_distance_after_2_swaps} to the pairs $(2, 6)$ and $(3, 5)$.
To obtain a probability arrangement equivalent to the total order in \cref{eq:optimal_total_order3}, we swap the probabilities of the pair $(2, 6)$. Then Property 
\ref{prop:increment_in_average_swap_distance_after_swap} gives
\begin{eqnarray*}
\averageswapdistance' - \averageswapdistance & = & 4 \underbrace{(p_2 - p_6)}(p_3 - p_5) \\
                                             & = & 4 \underbrace{(\pi_2 - \pi_3)}_{\geq 0}(\pi_4 - \pi_5)\geq 0 
\end{eqnarray*}
with equality if and only if $p_2 = p_6$. To obtain a probability arrangement equivalent to the total order in \cref{eq:optimal_total_order4}, we swap the probabilities of the pair $(3, 5)$. Proceeding similarly, we obtain
\begin{eqnarray*}
\averageswapdistance' - \averageswapdistance & = & 4 \underbrace{(p_5 - p_3)}\underbrace{(p_6 - p_2)} \\
\averageswapdistance' - \averageswapdistance & = & 4 \underbrace{(\pi_5 - \pi_4)}_{\leq 0}\underbrace{(\pi_3 - \pi_2)}_{\leq 0} \geq 0
\end{eqnarray*} 
with equality if and only if $p_3 = p_5$ and $p_2 = p_6$.
\end{enumerate}

\end{proof}

The previous theorem has various interesting consequences about the structure of optimal arrangements. 

\begin{corollary}
In minimum $\averageswapdistance$ arrangements,
\begin{align}
\min(2P(0), 1 - P(0)) \geq P(1) \geq P(2) \geq 2P(3) \geq 2Z \nonumber \\
P(3) \leq \min\left(P(0), 1 - P(0), \frac{P(2)}{2}\right), \label{eq:probability_of_distance_3_n_3}
\end{align}
where 
\begin{equation*}
Z = \pi_1 \pi_6 + \pi_2 \pi_5 + \pi_3 \pi_4.
\end{equation*}
Let $\pi_i$ be the $i$-th largest value among the $p_i$'s. Then 
\begin{align}
\averageswapdistance_{min} = & 3\dominanceindex - 2\left[ \right. \nonumber \\
                             & \pi_1(2\pi_2 + \pi_4) + \pi_2(2\pi_4 + \pi_6) + \nonumber \\ 
                             & \pi_3(2\pi_1 + \pi_2) + \pi_4(2\pi_6 + \pi_5) + \nonumber \\ 
                             & \left. \pi_5(2\pi_3 + \pi_1) + \pi_6(2\pi_5 + \pi_3) \right] \label{eq:average_swap_distance_min}
\end{align} 
\end{corollary}
\begin{proof}
Thanks to Property \ref{prop:probability_of_distance_upper_bound}, we know that
\begin{align*}
P(1), P(2) \leq \min(2P(0), 1 - P(0)) \\
2Z \leq P(3) \leq \min(P(0), 1 - P(0)).
\end{align*}
Now we aim to show that $2P(3) \leq P(1), P(2)$. 
Applying Theorem \ref{theo:optimal_arrangements_theorem} to \cref{eq:probability_of_distance} using \cref{fig:optimal_probability_arrangement} as a guide, we obtain 
\begin{align}
P(1) & = 2\boldsymbol{\pi}(\pi_2,\pi_4, \pi_1, \pi_6, \pi_3, \pi_5) \nonumber \\
     & = 2\boldsymbol{\pi}(\pi_3,\pi_1, \pi_5, \pi_2, \pi_6, \pi_4) \nonumber \\
P(2) & = 2\boldsymbol{\pi}(\pi_4,\pi_6, \pi_2, \pi_5, \pi_1, \pi_3) \nonumber \\
     & = 2\boldsymbol{\pi}(\pi_5,\pi_3, \pi_6, \pi_1, \pi_4, \pi_2) \nonumber \\
P(3) & = \boldsymbol{\pi} \overrightarrow{\boldsymbol{\pi}}. \label{eq:probability_distances_minimum_n_3}
\end{align}
By Theorem \ref{theo:Hardy_et_al}, $2P(3) \leq P(1), P(2)$.

Expanding \cref{eq:probability_distances_minimum_n_3}, we obtain
\begin{align}
\frac{P(1)}{2} = & \pi_1 \pi_2 + \pi_2 \pi_4 + \pi_4 \pi_6 + \pi_5 \pi_6 + \pi_3 \pi_5 + \nonumber \\
                 & \pi_1 \pi_3 \label{eq:probability_distance_1_minimum_n_3} \\
\frac{P(2)}{2} = & \pi_1 \pi_4 + \pi_2 \pi_6 + \pi_4 \pi_5 + \pi_3 \pi_6 + \pi_1 \pi_5 + \nonumber \\
                 & \pi_2 \pi_3 \label{eq:probability_distance_2_minimum_n_3}\\
\frac{P(3)}{2} = & \pi_1 \pi_6 + \pi_2 \pi_5 + \pi_3 \pi_4. \nonumber
\end{align}

To prove that $P(1) \geq P(2)$, 
we define 
\begin{equation*}
\Delta = \frac{1}{2}(P(1) - P(2)).
\end{equation*}
We will show that $\Delta \geq 0$. Plugging the expansions in \cref{eq:probability_distance_1_minimum_n_3} and \cref{eq:probability_distance_2_minimum_n_3} into the definition of $\Delta$, we get
\begin{align*}
\Delta & = (\pi_1 - \pi_6)(\pi_2 - \pi_4 + \pi_3 - \pi_5) + (\pi_2 - \pi_5)(\pi_4 - \pi_3) \\
       & = (\underbrace{\pi_2 - \pi_5}_{\geq 0})(\underbrace{\pi_1 - \pi_6 + \pi_4 - \pi_3}_{\geq 0}) + (\underbrace{\pi_1 - \pi_6}_{\geq 0})(\underbrace{\pi_3 - \pi_4}_{\geq 0}) \\
       & \geq 0
\end{align*}
Knowing that $2P(3) \leq P(1), P(2)$ and $P(1) \geq P(2)$, we get $2P(3) \leq \min(P(1), P(2)) = P(2)$.

Combining the results obtained so far, we retrieve \cref{eq:probability_of_distance_3_n_3}.

\Cref{eq:alternative2_average_swap_distance} with $n = 3$ yields 
\begin{equation*}
\averageswapdistance = 3\dominanceindex - 2P(1) - P(2).
\end{equation*}
The application of equations \cref{eq:probability_distance_1_minimum_n_3} and \cref{eq:probability_distance_2_minimum_n_3} eventually produces \cref{eq:average_swap_distance_min} after some algebra. 
\end{proof}

The following Corollary is the correlate of Corollary \ref{cor:local_optimality} for global optimality.

\begin{corollary}{Consequences of global optimality.}
Assume $n=3$. In a minimum $\averageswapdistance$ arrangement, there is a vertex $u$ with maximum probability, such that arrangement satisfies the following conditions
\begin{enumerate}
\item
{\em Radiation from $u$.} Probability remains the same or decreases as one moves away from $u$ on the permutohedron (\cref{fig:optimal_probability_arrangement} (a-b)). 
Consider $u$ as the reference point. The vertices at distance 1 are assigned probabilities $\pi_2$ and $\pi_3$, the vertex at distance $2$ are assigned probabilities $\pi_4$ and $\pi_5$ and the vertex at distance $3$ is assigned $\pi_6$.
\item
{\em Adjacency of the most likely orders}. Same definition as in Corollary \ref{cor:local_optimality}. 
\item
{\em Contiguity of the non-zero probability orders}. Same definition as in Corollary \ref{cor:local_optimality}. 
\end{enumerate}
\label{cor:optimality}
\end{corollary}

\begin{proof}
1 and 2 follow trivially from Theorem \ref{theo:optimal_arrangements_theorem}. 3 follows from 1 because a discontiguity would violate one of the inequalities described in \cref{fig:optimal_probability_arrangement} (a-b).
\end{proof}

We have shown that an optimal arrangement cannot be non-contiguous. Appendix \ref{app:contiguous} examines contiguity more closely showing that contiguity can be obtained without imposing optimality, simply looking for arrangements with smaller $\averageswapdistance$. 

\subsection{Counting the number of optimal arrangements}
\label{app:number_of_optimal_arrangements}

Consider $N_o(m)$, the number of probability arrangements that are optimal knowing $m$, the number of non-zero probability orders. The case $m = 1$ is special. We have $N_o(1) = 6!$, since all probability arrangements have the same $\averageswapdistance$ when $m=1$. 
When $1 < m \leq 6$ and assuming that there are no probability ties among non-zero probability orders, we have 
\begin{equation}
N_o(m) = 2 \cdot 6 \cdot (6 - m)!,
\label{eq:number_of_optimal_arrangements}
\end{equation}
where the $2$ factor comes from the two total orders that minimize $\averageswapdistance$ (Theorem \ref{theo:optimal_arrangements_theorem}), 6 accounts for the 6 vertices that can be assigned a maximum probability and $(6 - m)!$ accounts for the arrangements such that the zero probabilities are assigned to the same subset of vertices. Therefore,
the probability that a random permutation of probabilities is optimal as function of $m$ is 
$p_o(m) = N_o(m)/6!$, which gives \cref{eq:probability_of_optimal_by_chance}
if there are no probability ties among the non-zero probability vertices. Table \ref{tab:probability_n3} shows the values of $p_o(m)$. Notice that $p_o(5) = p_o(6)$.

\section{Contiguous arrangements}

\label{app:contiguous}

\subsection{Counting the number of contiguous arrangements}

We define $p_c(m)$ as the probability that a random permutation of probabilities produces a contiguous arrangement knowing that the number of non-zero probabilities is $m$. $p_c(m)$ is equivalent to the probability that $m$ randomly chosen vertices form a path in the permutohedron.
Consider $N_c(m)$, the number of permutations where the non-zero probability orders are contiguous. We treat $N_c(m)$ as the ways of distributing $m$ elements on a sequence of $6$ elements forming a subsequence of $m$ consecutive elements with boundary conditions, that is, the next element of the element in position $6$ is the element in position 1 and the previous element of the element in position 1 is the element in position $6$.
The number of ways of distributing the $m$ elements forming such a subsequence of $m$ elements is as follows.  
The subsequence has $6$ possible start positions in the sequence. For each start position, the elements in the segment are interchangeable and also the elements out of the segment are interchangeable.
Therefore, we have
\begin{equation}
N_c(m) = 6 m! (6 - m)!,
\label{eq:number_of_contiguous_arrangements}
\end{equation}
and then $p_c(m) = N_c(m)/6!$ gives
\begin{eqnarray*}
p_c(m) & = \frac{6 m! (6 - m)!}{6!} \\
       & = \frac{m! (6- m)!}{5!}
\end{eqnarray*}
for $1\leq m \leq 5$.
When $m = 6$, the notion of start position does not apply but it is easy to infer that $p_c(m) = 1$ in that case. Finally, 
\begin{equation}
p_c(m) = \left\{ 
             \begin{array}{ll} 
                 \frac{m! (6 - m)!}{120} & \mbox{if } 1\leq m \leq 5 \\
                 1 & \mbox{if } m = 6. 
             \end{array}    
          \right.
          \label{eq:probability_of_contiguous}   
\end{equation} 
Notice that $p_c(5) = p_c(6)$. It is easy to see that $p_c(m)$ satisfies a symmetry property, that is $p_c(m) = p_c(6 - m)$ for $1 \leq m \leq 5$. Table \ref{tab:probability_n3} shows the values of $p_c(m)$.

\begin{table}
\caption{\label{tab:probability_n3} $p_o(m)$, the probability that random arrangement is optimal assuming that there are no probability ties among the $m$ non-zero probability orders (\cref{eq:probability_of_optimal_by_chance}), and $p_c(m)$, the probability that a random arrangement is contiguous when the number non-zero probability orders is $m$ (\cref{eq:probability_of_contiguous}).}
\centering
\begin{tabular}{lll}
$m$ & $p_o(m)$ & $p_c(m)$ \\
\hline
1   & $1$    & $1$ \\
2   & $2/5$  & $2/5$ \\
3   & $1/10$ & $3/10$ \\
4   & $1/30$ & $2/5$ \\
5   & $1/60$ & $1$ \\
6   & $1/60$ & $1$ \\
\end{tabular}
\end{table}

\subsection{Contiguity versus optimality}

We will show that contiguity does not imply optimality in general. 

\begin{property}
\label{prop:when_optimality_implies_contiguity}
If all non-zero probabilities are distinct, contiguity implies optimality ($\averageswapdistance = \averageswapdistance_{min}$) if and only if $m \leq 2$.
\end{property}
\begin{proof}
Contiguity implies optimality if and only if $N_c(m) = N_o(m)$. 
This is trivially true for $m = 1$.
For $m > 1$, we recall \cref{eq:number_of_optimal_arrangements} and \cref{eq:number_of_contiguous_arrangements}. Hence, the ratio 
\begin{equation*}
\frac{N_c(m)}{N_o(m)} = \frac{2}{m!}
\end{equation*}
hits $1$ if and only if $m = 2$.
\end{proof}

\subsection{The advantage of transforming a non-contiguous arrangement into a contiguous one}

We will show that any non-contiguous arrangement has a contiguous rearrangement with smaller $\averageswapdistance$. The existence of non-contiguous arrangements requires $m > 1$. The point is that such contiguous arrangement does not need to be optimal when $m > 2$ (Property \ref{prop:when_optimality_implies_contiguity}). 

We define $\averageswapdistance_c$ as the value of $\averageswapdistance$ of a contiguous arrangement, namely all non-zero probability orders are consecutive in the permutohedron. 
We use $\averageswapdistance_{nc}$ to refer to the value of $\averageswapdistance$ of a non-contiguous arrangement.

Notice that, when $m = 1$, $m = 5$ or $m = 6$, all arrangements are contiguous. 
The next Property deals with the values of $m$ where non-contiguous arrangements are possible. 

\begin{property}
\label{prop:contiguity_of_non_zero_probability_orders}
~\\ 
\begin{enumerate}
\item
When $m = 2$, for any contiguous arrangement and for any non-contiguous arrangement, $\averageswapdistance_c < \averageswapdistance_{nc}$.
\item
When $m = 3$ or $m = 4$, for any non-contiguous arrangement, there is always a contiguous arrangement with a value of $\averageswapdistance$ that is $\averageswapdistance_c^*$ such that
$\averageswapdistance_c^* < \averageswapdistance_{nc}$.
\item
When $m = 3$ and the non-contiguous arrangement is one where all pairs of non-zero probability orders are at swap distance 2 (\cref{fig:arrangements} (e)) then, 
for any contiguous arrangement, $\averageswapdistance_c < \averageswapdistance_{nc}$.
\end{enumerate}
\end{property}

\begin{proof}

We use letters $w$, $x$, $y$, $z$ to refer to vertices of the permutohedron that have non-zero probability, that is $p_w, p_x, p_y, p_z > 0$. We assume that the distribution of probabilities is constant. 

When $m=2$, $p_y = 1 - p_y$ and \cref{eq:average_swap_distance} becomes the function
\begin{equation*}
\averageswapdistance(d_{xy}) = 2p_x(1 - p_x) d_{xy}.
\end{equation*}
We have 
$\averageswapdistance_c = \averageswapdistance(1)$ and $\averageswapdistance_{nc} = \averageswapdistance(\delta)$ with $\delta \in \{2, 3\}$.
Obviously, $\averageswapdistance_c < \averageswapdistance_{nc}$. It has been shown that $\averageswapdistance(3)$ with $p_x = 1/2$ maximizes $\averageswapdistance$ for any $m$ when $n = 3$ \citep[Appendix C.1]{Franco2024a}. 

\begin{figure}
\centering
\includegraphics[width = \linewidth]{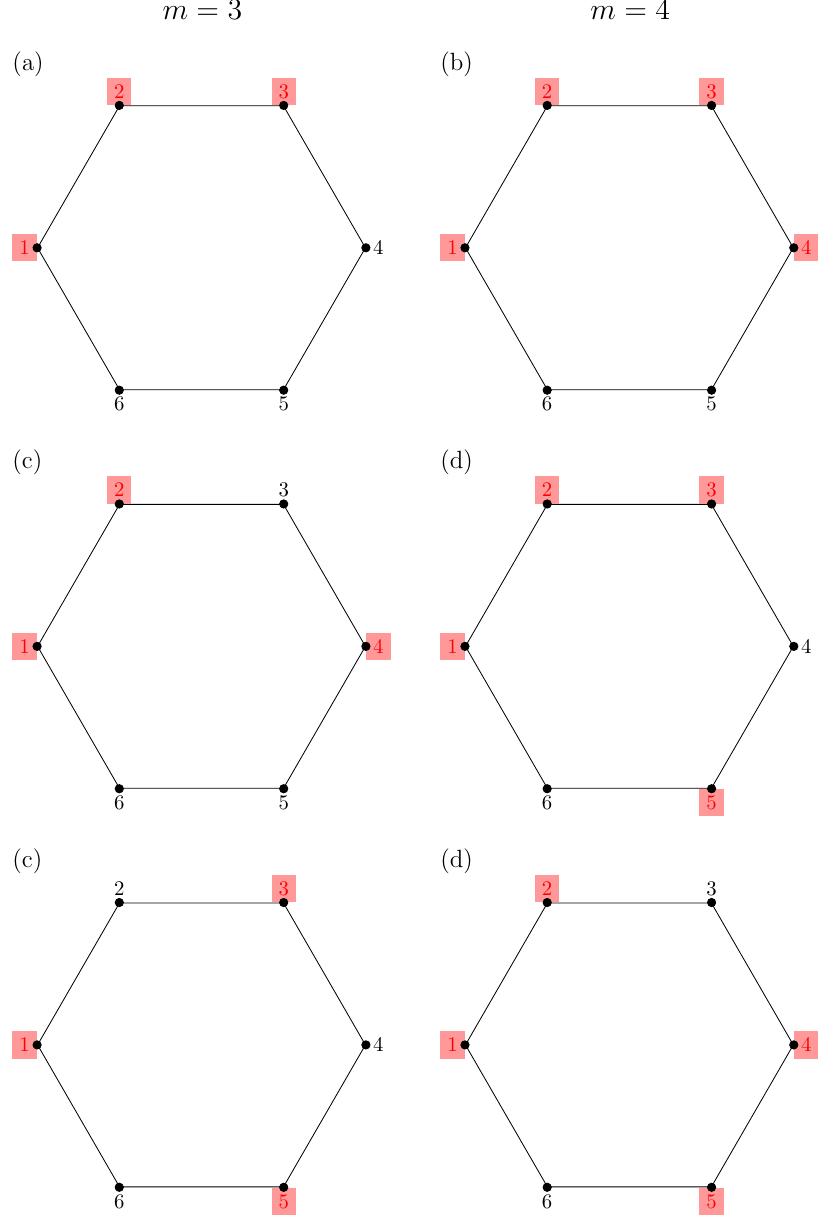}
\caption{\label{fig:arrangements} Distinct ways of arranging $m$ non-zero probability orders on the permutohedron ($m = 3$ on the left and $m=4$ on the right). The orders with non-zero probability are marked in red. 
(a) A contiguous arrangement where the multiset of swap distances is $\{1, 1, 2\}$.
(b) A contiguous arrangement where the multiset of swap distances is $\{1, 1, 1, 2, 2, 3\}$.
(c) A non-contiguous arrangement where the multiset of swap distances is $\{1, 2, 3\}$.
(d) A non-contiguous arrangement where the multiset of swap distances is $\{1, 1, 2, 2, 2, 3\}$. 
(e) A non-contiguous arrangement where the multiset of swap distances is $\{2, 2, 2\}$.
(f) Another non-contiguous arrangement where the multiset of swap distances is $\{1, 1, 2, 2, 3, 3\}$. 
} 
\end{figure}

When $m=3$, \cref{eq:average_swap_distance} becomes 
\begin{equation*}
\averageswapdistance = 2(p_x p_y d_{xy} + p_x p_z d_{xz} + p_y p_z d_{yz}).
\label{eq:average_swap_distance_n_3}
\end{equation*}
In a contiguous arrangement, the multiset of swap distances formed by $d_{xy}$, $d_{xz}$ and $d_{yz}$, i.e. $\{d_{xy}, d_{xz}, d_{yz}\}$ equals the multiset $\{1, 1, 2\}$ (\cref{fig:arrangements} (a)).
In the non-contiguous arrangements, one has that $\{d_{xy}, d_{xz}, d_{yz}\} = \{1, 2, 3\}$ (\cref{fig:arrangements} (c)) or $\{d_{xy}, d_{xz}, d_{yz}\} = \{2, 2, 2\}$ (\cref{fig:arrangements} (e)). 

Suppose a non-contiguous arrangement with multiset $\{1, 2, 3\}$. Without loss of generality suppose that such arrangement has a value of $\averageswapdistance$ that is (\cref{eq:average_swap_distance_n_3})
\begin{equation*}
\averageswapdistance_{nc} = 2(p_x p_y + 2p_x p_z + 3p_y p_z).
\end{equation*}
Then there is a contiguous arrangement such that its value of $\averageswapdistance$ is
\begin{equation}
\averageswapdistance_{c}^* = 2(p_x p_y + p_x p_z + 2p_y p_z) \label{eq:contiguous_arrangement_n_3}
\end{equation}  
and $\averageswapdistance_{c}^* < \averageswapdistance_{nc}$. Thus, any non-contiguous arrangement with multiset $\{1, 2, 3\}$ has a contiguous arrangement with smaller $\averageswapdistance$.

Suppose a non-contiguous arrangement with multiset $\{2, 2, 2\}$. Such arrangement has a value of $\averageswapdistance$ that is 
\begin{equation*}
\averageswapdistance_{nc} = 2(2p_x p_y + 2p_x p_z + 2p_y p_z).
\end{equation*}
Without loss of generality, suppose a contiguous arrangement such that its value of $\averageswapdistance$ is $\averageswapdistance_{c}^*$ (\cref{eq:contiguous_arrangement_n_3}) and then
\begin{equation*}
\averageswapdistance_{nc} = \averageswapdistance_{c}^* + 2(p_x p_y + p_x p_z).
\end{equation*} 
Obviously, $\averageswapdistance_{c}^* < \averageswapdistance_{nc}$. Thus, any non-contiguous arrangement with multiset $\{2, 2, 2\}$ has a value $\averageswapdistance$ that is greater than that of any contiguous arrangement. 

When $m=4$, \cref{eq:average_swap_distance} becomes 
\begin{eqnarray*}
\frac{\averageswapdistance}{2} & = & p_x p_y d_{xy} + p_x p_z d_{xz} + p_x p_z d_{xw} + p_y p_z d_{yz} +  \\	
                               &   & p_y p_w d_{yw} + p_z p_w d_{zw}.
\end{eqnarray*}
In a contiguous arrangement, the multiset of swap distances among pairs of non-zero probability orders is $\{1, 1, 1, 2, 2, 3\}$ (\cref{fig:arrangements} (b)).
In the non-contiguous arrangements, the multiset is either $\{1, 1, 2, 2, 2, 3\}$ (\cref{fig:arrangements} (d)) or $\{1, 1, 2, 2, 3, 3\}$ (\cref{fig:arrangements} (f)). 

Suppose a non-contiguous arrangement with multiset $\{1, 1, 2, 2, 2, 3\}$. Without loss of generality suppose that such arrangement has a value of $\averageswapdistance$ that is 
\begin{equation*}
\frac{\averageswapdistance_{nc}}{2} = p_x p_y + p_x p_z + 2 p_x p_z + 2 p_y p_z + 2 p_y p_w + 3 p_z p_w.
\end{equation*}
Then there is a contiguous arrangement such that its value of $\averageswapdistance$ is
\begin{equation*}
\frac{\averageswapdistance_{c}^*}{2} = p_x p_y + p_x p_z + p_x p_z + 2 p_y p_z + 2 p_y p_w + 3 p_z p_w
\end{equation*}  
and $\averageswapdistance_{c}^* < \averageswapdistance_{nc}$. Thus, any non-contiguous arrangement with multiset $\{1, 1, 2, 2, 2, 3\}$ has a contiguous arrangement with smaller $\averageswapdistance$. Reasoning analogously, we reach the same conclusion for any non-contiguous arrangement with multiset $\{1, 1, 2, 2, 3, 3\}$.

\end{proof}

\bibliography{../../biblio/rferrericancho,../../biblio/complex,../../biblio/ling,../../biblio/cl,../../biblio/cs,../../biblio/ml,../../biblio/maths,../../biblio/Encyclopedia_of_Language_and_Linguistics}

\begin{thebibliography}{66}%
\makeatletter
\providecommand \@ifxundefined [1]{%
 \@ifx{#1\undefined}
}%
\providecommand \@ifnum [1]{%
 \ifnum #1\expandafter \@firstoftwo
 \else \expandafter \@secondoftwo
 \fi
}%
\providecommand \@ifx [1]{%
 \ifx #1\expandafter \@firstoftwo
 \else \expandafter \@secondoftwo
 \fi
}%
\providecommand \natexlab [1]{#1}%
\providecommand \enquote  [1]{``#1''}%
\providecommand \bibnamefont  [1]{#1}%
\providecommand \bibfnamefont [1]{#1}%
\providecommand \citenamefont [1]{#1}%
\providecommand \href@noop [0]{\@secondoftwo}%
\providecommand \href [0]{\begingroup \@sanitize@url \@href}%
\providecommand \@href[1]{\@@startlink{#1}\@@href}%
\providecommand \@@href[1]{\endgroup#1\@@endlink}%
\providecommand \@sanitize@url [0]{\catcode `\\12\catcode `\$12\catcode
  `\&12\catcode `\#12\catcode `\^12\catcode `\_12\catcode `\%12\relax}%
\providecommand \@@startlink[1]{}%
\providecommand \@@endlink[0]{}%
\providecommand \url  [0]{\begingroup\@sanitize@url \@url }%
\providecommand \@url [1]{\endgroup\@href {#1}{\urlprefix }}%
\providecommand \urlprefix  [0]{URL }%
\providecommand \Eprint [0]{\href }%
\providecommand \doibase [0]{https://doi.org/}%
\providecommand \selectlanguage [0]{\@gobble}%
\providecommand \bibinfo  [0]{\@secondoftwo}%
\providecommand \bibfield  [0]{\@secondoftwo}%
\providecommand \translation [1]{[#1]}%
\providecommand \BibitemOpen [0]{}%
\providecommand \bibitemStop [0]{}%
\providecommand \bibitemNoStop [0]{.\EOS\space}%
\providecommand \EOS [0]{\spacefactor3000\relax}%
\providecommand \BibitemShut  [1]{\csname bibitem#1\endcsname}%
\let\auto@bib@innerbib\@empty
\bibitem [{\citenamefont {{Franco-Sánchez}}\ \emph {et~al.}(2026)\citenamefont
  {{Franco-Sánchez}}, \citenamefont {{Martí-Llobet}},\ and\ \citenamefont
  {{Ferrer-i-Cancho}}}]{Franco2024a}%
  \BibitemOpen
  \bibfield  {author} {\bibinfo {author} {\bibfnamefont {V.}~\bibnamefont
  {{Franco-Sánchez}}}, \bibinfo {author} {\bibfnamefont {A.}~\bibnamefont
  {{Martí-Llobet}}},\ and\ \bibinfo {author} {\bibfnamefont {R.}~\bibnamefont
  {{Ferrer-i-Cancho}}},\ }\bibfield  {title} {\bibinfo {title} {Swap distance
  minimization beyond entropy minimization in word order variation},\ }\href
  {https://doi.org/10.1080/09296174.2025.2585611} {\bibfield  {journal}
  {\bibinfo  {journal} {Journal of Quantitative Linguistics}\ }\textbf
  {\bibinfo {volume} {33}},\ \bibinfo {pages} {305} (\bibinfo {year}
  {2026})}\BibitemShut {NoStop}%
\bibitem [{\citenamefont {{Ferrer-i-Cancho}}\ and\ \citenamefont
  {Namboodiripad}(2023)}]{Ferrer2023a}%
  \BibitemOpen
  \bibfield  {author} {\bibinfo {author} {\bibfnamefont {R.}~\bibnamefont
  {{Ferrer-i-Cancho}}}\ and\ \bibinfo {author} {\bibfnamefont {S.}~\bibnamefont
  {Namboodiripad}},\ }\bibfield  {title} {\bibinfo {title} {Swap distance
  minimization in {SOV} languages. {Cognitive} and mathematical foundations},\
  }\href {https://doi.org/10.53482/2023_55_412} {\bibfield  {journal} {\bibinfo
   {journal} {Glottometrics}\ }\textbf {\bibinfo {volume} {55}},\ \bibinfo
  {pages} {59} (\bibinfo {year} {2023})}\BibitemShut {NoStop}%
\bibitem [{\citenamefont {{Ferrer-i-Cancho}}(2016)}]{Ferrer2016c}%
  \BibitemOpen
  \bibfield  {author} {\bibinfo {author} {\bibfnamefont {R.}~\bibnamefont
  {{Ferrer-i-Cancho}}},\ }\bibfield  {title} {\bibinfo {title} {Kauffman's
  adjacent possible in word order evolution},\ }in\ \href
  {https://arxiv.org/pdf/1512.05582} {\emph {\bibinfo {booktitle} {The
  evolution of language: Proceedings of the 11th International Conference
  (EVOLANG11)}}},\ \bibinfo {editor} {edited by\ \bibinfo {editor}
  {\bibfnamefont {S.}~\bibnamefont {Roberts}}, \bibinfo {editor} {\bibfnamefont
  {C.}~\bibnamefont {Cuskley}}, \bibinfo {editor} {\bibfnamefont
  {L.}~\bibnamefont {McCrohon}}, \bibinfo {editor} {\bibfnamefont
  {L.}~\bibnamefont {Barcel\'o-Coblijn}}, \bibinfo {editor} {\bibfnamefont
  {O.}~\bibnamefont {Feher}},\ and\ \bibinfo {editor} {\bibfnamefont
  {T.}~\bibnamefont {Verhoef}}}\ (\bibinfo {address} {New Orleans, USA},\
  \bibinfo {year} {2016})\ \bibinfo {note} {{Evolution} of {Language}
  {Conference} ({Evolang} 2016), March 21-24}\BibitemShut {NoStop}%
\bibitem [{\citenamefont {{Ferrer-i-Cancho}}(2015)}]{Ferrer2013e}%
  \BibitemOpen
  \bibfield  {author} {\bibinfo {author} {\bibfnamefont {R.}~\bibnamefont
  {{Ferrer-i-Cancho}}},\ }\bibfield  {title} {\bibinfo {title} {The placement
  of the head that minimizes online memory. {A} complex systems approach},\
  }\href {https://doi.org/10.1163/22105832-00501007} {\bibfield  {journal}
  {\bibinfo  {journal} {Language Dynamics and Change}\ }\textbf {\bibinfo
  {volume} {5}},\ \bibinfo {pages} {114} (\bibinfo {year} {2015})}\BibitemShut
  {NoStop}%
\bibitem [{\citenamefont {Somerfield}\ \emph {et~al.}(2008)\citenamefont
  {Somerfield}, \citenamefont {Clarke},\ and\ \citenamefont
  {Warwick}}]{Sommerfield2008a}%
  \BibitemOpen
  \bibfield  {author} {\bibinfo {author} {\bibfnamefont {P.}~\bibnamefont
  {Somerfield}}, \bibinfo {author} {\bibfnamefont {K.}~\bibnamefont {Clarke}},\
  and\ \bibinfo {author} {\bibfnamefont {R.}~\bibnamefont {Warwick}},\
  }\bibfield  {title} {\bibinfo {title} {Simpson index},\ }in\ \href
  {https://doi.org/10.1016/B978-008045405-4.00133-6} {\emph {\bibinfo
  {booktitle} {Encyclopedia of Ecology}}},\ \bibinfo {editor} {edited by\
  \bibinfo {editor} {\bibfnamefont {S.~E.}\ \bibnamefont {Jørgensen}}\ and\
  \bibinfo {editor} {\bibfnamefont {B.~D.}\ \bibnamefont {Fath}}}\ (\bibinfo
  {publisher} {Academic Press},\ \bibinfo {address} {Oxford},\ \bibinfo {year}
  {2008})\ pp.\ \bibinfo {pages} {3252--3255}\BibitemShut {NoStop}%
\bibitem [{\citenamefont {{Rios-El-Yazidi}}\ and\ \citenamefont
  {{Ferrer-i-Cancho}}(2026)}]{Rios-El-Yazidi2026a}%
  \BibitemOpen
  \bibfield  {author} {\bibinfo {author} {\bibfnamefont {J.}~\bibnamefont
  {{Rios-El-Yazidi}}}\ and\ \bibinfo {author} {\bibfnamefont {R.}~\bibnamefont
  {{Ferrer-i-Cancho}}},\ }\bibfield  {title} {\bibinfo {title} {Swap distance
  minimization shapes the order of subject, object and verb in languages of the
  world},\ }\href {http://arxiv.org/abs/2604.26726} {\bibfield  {journal}
  {\bibinfo  {journal} {Europhysics Letters}\ ,\ \bibinfo {pages} {in press}}
  (\bibinfo {year} {2026})}\BibitemShut {NoStop}%
\bibitem [{\citenamefont {{Rios-El-Yazidi}}(2026)}]{Rios-El-Yazidi2026b}%
  \BibitemOpen
  \bibfield  {author} {\bibinfo {author} {\bibfnamefont {J.}~\bibnamefont
  {{Rios-El-Yazidi}}},\ }\emph {\bibinfo {title} {The Quadratic Assignment
  Problem in word order}},\ \href {https://hdl.handle.net/2117/473378}
  {Master's thesis},\ \bibinfo  {school} {Barcelona School of Informatics},
  \bibinfo {address} {Barcelona} (\bibinfo {year} {2026})\BibitemShut {NoStop}%
\bibitem [{\citenamefont {{Ferrer-i-Cancho}}\ \emph
  {et~al.}(2022{\natexlab{a}})\citenamefont {{Ferrer-i-Cancho}}, \citenamefont
  {{G\'omez-Rodr\'iguez}}, \citenamefont {Esteban},\ and\ \citenamefont
  {{Alemany-Puig}}}]{Ferrer2020b}%
  \BibitemOpen
  \bibfield  {author} {\bibinfo {author} {\bibfnamefont {R.}~\bibnamefont
  {{Ferrer-i-Cancho}}}, \bibinfo {author} {\bibfnamefont {C.}~\bibnamefont
  {{G\'omez-Rodr\'iguez}}}, \bibinfo {author} {\bibfnamefont {J.~L.}\
  \bibnamefont {Esteban}},\ and\ \bibinfo {author} {\bibfnamefont
  {L.}~\bibnamefont {{Alemany-Puig}}},\ }\bibfield  {title} {\bibinfo {title}
  {Optimality of syntactic dependency distances},\ }\href
  {https://doi.org/10.1103/PhysRevE.105.014308} {\bibfield  {journal} {\bibinfo
   {journal} {Physical Review E}\ }\textbf {\bibinfo {volume} {105}},\ \bibinfo
  {pages} {014308} (\bibinfo {year} {2022}{\natexlab{a}})}\BibitemShut
  {NoStop}%
\bibitem [{\citenamefont {Petrini}\ \emph {et~al.}(2026)\citenamefont
  {Petrini}, \citenamefont {{Casas-i-Muñoz}}, \citenamefont
  {{Cluet-i-Martinell}}, \citenamefont {Wang}, \citenamefont {Bentz},\ and\
  \citenamefont {{Ferrer-i-Cancho}}}]{Petrini2022a}%
  \BibitemOpen
  \bibfield  {author} {\bibinfo {author} {\bibfnamefont {S.}~\bibnamefont
  {Petrini}}, \bibinfo {author} {\bibfnamefont {A.}~\bibnamefont
  {{Casas-i-Muñoz}}}, \bibinfo {author} {\bibfnamefont {J.}~\bibnamefont
  {{Cluet-i-Martinell}}}, \bibinfo {author} {\bibfnamefont {M.}~\bibnamefont
  {Wang}}, \bibinfo {author} {\bibfnamefont {C.}~\bibnamefont {Bentz}},\ and\
  \bibinfo {author} {\bibfnamefont {R.}~\bibnamefont {{Ferrer-i-Cancho}}},\
  }\bibfield  {title} {\bibinfo {title} {The optimality of word lengths.
  {Theoretical} foundations and an empirical study},\ }\href
  {https://doi.org/10.53482/2026_60_434} {\bibfield  {journal} {\bibinfo
  {journal} {Glottometrics}\ }\textbf {\bibinfo {volume} {60}},\ \bibinfo
  {pages} {77} (\bibinfo {year} {2026})}\BibitemShut {NoStop}%
\bibitem [{\citenamefont {Hubert}\ and\ \citenamefont
  {Arabie}(1985)}]{Hubert1985a}%
  \BibitemOpen
  \bibfield  {author} {\bibinfo {author} {\bibfnamefont {L.}~\bibnamefont
  {Hubert}}\ and\ \bibinfo {author} {\bibfnamefont {P.}~\bibnamefont
  {Arabie}},\ }\bibfield  {title} {\bibinfo {title} {Comparing partitions},\
  }\href {https://doi.org/10.1007/bf01908075} {\bibfield  {journal} {\bibinfo
  {journal} {Journal of Classification}\ }\textbf {\bibinfo {volume} {2}},\
  \bibinfo {pages} {193–218} (\bibinfo {year} {1985})}\BibitemShut {NoStop}%
\bibitem [{\citenamefont {Vinh}\ \emph {et~al.}(2010)\citenamefont {Vinh},
  \citenamefont {Epps},\ and\ \citenamefont {Bailey}}]{Vinh2010a}%
  \BibitemOpen
  \bibfield  {author} {\bibinfo {author} {\bibfnamefont {N.~X.}\ \bibnamefont
  {Vinh}}, \bibinfo {author} {\bibfnamefont {J.}~\bibnamefont {Epps}},\ and\
  \bibinfo {author} {\bibfnamefont {J.}~\bibnamefont {Bailey}},\ }\bibfield
  {title} {\bibinfo {title} {Information theoretic measures for clusterings
  comparison: Variants, properties, normalization and correction for chance},\
  }\href {http://jmlr.org/papers/v11/vinh10a.html} {\bibfield  {journal}
  {\bibinfo  {journal} {Journal of Machine Learning Research}\ }\textbf
  {\bibinfo {volume} {11}},\ \bibinfo {pages} {2837} (\bibinfo {year}
  {2010})}\BibitemShut {NoStop}%
\bibitem [{\citenamefont {{Ferrer-i-Cancho}}(2004)}]{Ferrer2004b}%
  \BibitemOpen
  \bibfield  {author} {\bibinfo {author} {\bibfnamefont {R.}~\bibnamefont
  {{Ferrer-i-Cancho}}},\ }\bibfield  {title} {\bibinfo {title} {{Euclidean}
  distance between syntactically linked words},\ }\href
  {https://doi.org/10.1103/PhysRevE.70.056135} {\bibfield  {journal} {\bibinfo
  {journal} {Physical Review E}\ }\textbf {\bibinfo {volume} {70}},\ \bibinfo
  {pages} {056135} (\bibinfo {year} {2004})}\BibitemShut {NoStop}%
\bibitem [{\citenamefont {Liu}(2008)}]{Liu2008a}%
  \BibitemOpen
  \bibfield  {author} {\bibinfo {author} {\bibfnamefont {H.}~\bibnamefont
  {Liu}},\ }\bibfield  {title} {\bibinfo {title} {Dependency distance as a
  metric of language comprehension difficulty},\ }\href
  {https://doi.org/10.17791/jcs.2008.9.2.159} {\bibfield  {journal} {\bibinfo
  {journal} {Journal of Cognitive Science}\ }\textbf {\bibinfo {volume} {9}},\
  \bibinfo {pages} {159} (\bibinfo {year} {2008})}\BibitemShut {NoStop}%
\bibitem [{\citenamefont {Futrell}\ \emph
  {et~al.}(2015{\natexlab{a}})\citenamefont {Futrell}, \citenamefont
  {Mahowald},\ and\ \citenamefont {Gibson}}]{Futrell2015a}%
  \BibitemOpen
  \bibfield  {author} {\bibinfo {author} {\bibfnamefont {R.}~\bibnamefont
  {Futrell}}, \bibinfo {author} {\bibfnamefont {K.}~\bibnamefont {Mahowald}},\
  and\ \bibinfo {author} {\bibfnamefont {E.}~\bibnamefont {Gibson}},\
  }\bibfield  {title} {\bibinfo {title} {Large-scale evidence of dependency
  length minimization in 37 languages},\ }\href
  {https://doi.org/https://doi.org/10.1073/pnas.1502134112} {\bibfield
  {journal} {\bibinfo  {journal} {Proceedings of the National Academy of
  Sciences USA}\ }\textbf {\bibinfo {volume} {112}},\ \bibinfo {pages} {10336}
  (\bibinfo {year} {2015}{\natexlab{a}})}\BibitemShut {NoStop}%
\bibitem [{\citenamefont {Muñoz-Ortiz}\ \emph {et~al.}(2024)\citenamefont
  {Muñoz-Ortiz}, \citenamefont {Gómez-Rodríguez},\ and\ \citenamefont
  {Vilares}}]{Munoz-Ortiz2024a}%
  \BibitemOpen
  \bibfield  {author} {\bibinfo {author} {\bibfnamefont {A.}~\bibnamefont
  {Muñoz-Ortiz}}, \bibinfo {author} {\bibfnamefont {C.}~\bibnamefont
  {Gómez-Rodríguez}},\ and\ \bibinfo {author} {\bibfnamefont
  {D.}~\bibnamefont {Vilares}},\ }\bibfield  {title} {\bibinfo {title}
  {Contrasting linguistic patterns in human and {LLM}-generated news text},\
  }\bibfield  {journal} {\bibinfo  {journal} {Artificial Intelligence Review}\
  }\textbf {\bibinfo {volume} {57}},\ \href
  {https://doi.org/10.1007/s10462-024-10903-2} {10.1007/s10462-024-10903-2}
  (\bibinfo {year} {2024})\BibitemShut {NoStop}%
\bibitem [{\citenamefont {Futrell}\ \emph
  {et~al.}(2015{\natexlab{b}})\citenamefont {Futrell}, \citenamefont {Hickey},
  \citenamefont {Lee}, \citenamefont {Lim}, \citenamefont {Luchkina},\ and\
  \citenamefont {Gibson}}]{Futrell2015b}%
  \BibitemOpen
  \bibfield  {author} {\bibinfo {author} {\bibfnamefont {R.}~\bibnamefont
  {Futrell}}, \bibinfo {author} {\bibfnamefont {T.}~\bibnamefont {Hickey}},
  \bibinfo {author} {\bibfnamefont {A.}~\bibnamefont {Lee}}, \bibinfo {author}
  {\bibfnamefont {E.}~\bibnamefont {Lim}}, \bibinfo {author} {\bibfnamefont
  {E.}~\bibnamefont {Luchkina}},\ and\ \bibinfo {author} {\bibfnamefont
  {E.}~\bibnamefont {Gibson}},\ }\bibfield  {title} {\bibinfo {title}
  {Cross-linguistic gestures reflect typological universals: a subject-initial,
  verb-final bias in speakers of diverse languages},\ }\href
  {https://doi.org/10.1016/j.cognition.2014.11.022} {\bibfield  {journal}
  {\bibinfo  {journal} {Cognition}\ }\textbf {\bibinfo {volume} {136}},\
  \bibinfo {pages} {215} (\bibinfo {year} {2015}{\natexlab{b}})}\BibitemShut
  {NoStop}%
\bibitem [{\citenamefont {Goldin-Meadow}\ \emph {et~al.}(2008)\citenamefont
  {Goldin-Meadow}, \citenamefont {So}, \citenamefont {\"Ozy\"urek},\ and\
  \citenamefont {Mylander}}]{Goldin-Meadow2008a}%
  \BibitemOpen
  \bibfield  {author} {\bibinfo {author} {\bibfnamefont {S.}~\bibnamefont
  {Goldin-Meadow}}, \bibinfo {author} {\bibfnamefont {W.~C.}\ \bibnamefont
  {So}}, \bibinfo {author} {\bibfnamefont {A.}~\bibnamefont {\"Ozy\"urek}},\
  and\ \bibinfo {author} {\bibfnamefont {C.}~\bibnamefont {Mylander}},\
  }\bibfield  {title} {\bibinfo {title} {The natural order of events: how
  speakers of different languages represent events nonverbally},\ }\href@noop
  {} {\bibfield  {journal} {\bibinfo  {journal} {Proceedings of the National
  Academy of Sciences}\ }\textbf {\bibinfo {volume} {105}},\ \bibinfo {pages}
  {9163} (\bibinfo {year} {2008})}\BibitemShut {NoStop}%
\bibitem [{\citenamefont {Langus}\ and\ \citenamefont
  {Nespor}(2010)}]{Langus2010a}%
  \BibitemOpen
  \bibfield  {author} {\bibinfo {author} {\bibfnamefont {A.}~\bibnamefont
  {Langus}}\ and\ \bibinfo {author} {\bibfnamefont {M.}~\bibnamefont
  {Nespor}},\ }\bibfield  {title} {\bibinfo {title} {Cognitive systems
  struggling for word order},\ }\href@noop {} {\bibfield  {journal} {\bibinfo
  {journal} {Cognitive Psychology}\ }\textbf {\bibinfo {volume} {60}},\
  \bibinfo {pages} {291} (\bibinfo {year} {2010})}\BibitemShut {NoStop}%
\bibitem [{\citenamefont {Dryer}(2013)}]{wals-81}%
  \BibitemOpen
  \bibfield  {author} {\bibinfo {author} {\bibfnamefont {M.~S.}\ \bibnamefont
  {Dryer}},\ }\bibfield  {title} {\bibinfo {title} {Order of subject, object
  and verb},\ }in\ \href {http://wals.info/chapter/81} {\emph {\bibinfo
  {booktitle} {The World Atlas of Language Structures Online}}},\ \bibinfo
  {editor} {edited by\ \bibinfo {editor} {\bibfnamefont {M.~S.}\ \bibnamefont
  {Dryer}}\ and\ \bibinfo {editor} {\bibfnamefont {M.}~\bibnamefont
  {Haspelmath}}}\ (\bibinfo  {publisher} {Max Planck Institute for Evolutionary
  Anthropology},\ \bibinfo {address} {Leipzig},\ \bibinfo {year}
  {2013})\BibitemShut {NoStop}%
\bibitem [{\citenamefont {D\'iaz}\ \emph {et~al.}(2002)\citenamefont {D\'iaz},
  \citenamefont {Petit},\ and\ \citenamefont {Serna}}]{Diaz2002}%
  \BibitemOpen
  \bibfield  {author} {\bibinfo {author} {\bibfnamefont {J.}~\bibnamefont
  {D\'iaz}}, \bibinfo {author} {\bibfnamefont {J.}~\bibnamefont {Petit}},\ and\
  \bibinfo {author} {\bibfnamefont {M.}~\bibnamefont {Serna}},\ }\bibfield
  {title} {\bibinfo {title} {A survey of graph layout problems},\ }\href@noop
  {} {\bibfield  {journal} {\bibinfo  {journal} {ACM Computing Surveys}\
  }\textbf {\bibinfo {volume} {34}},\ \bibinfo {pages} {313} (\bibinfo {year}
  {2002})}\BibitemShut {NoStop}%
\bibitem [{\citenamefont {Petit}(2011)}]{Petit2011a}%
  \BibitemOpen
  \bibfield  {author} {\bibinfo {author} {\bibfnamefont {J.}~\bibnamefont
  {Petit}},\ }\bibfield  {title} {\bibinfo {title} {Addenda to the survey of
  layout problems},\ }\href@noop {} {\bibfield  {journal} {\bibinfo  {journal}
  {Bulletin of the European Association for Theoretical Computer Science}\
  }\textbf {\bibinfo {volume} {105}},\ \bibinfo {pages} {177} (\bibinfo {year}
  {2011})}\BibitemShut {NoStop}%
\bibitem [{\citenamefont {Gómez-Rodríguez}\ and\ \citenamefont
  {Alemany-Puig}(2026)}]{Gomez2026a_Encyclopedia}%
  \BibitemOpen
  \bibfield  {author} {\bibinfo {author} {\bibfnamefont {C.}~\bibnamefont
  {Gómez-Rodríguez}}\ and\ \bibinfo {author} {\bibfnamefont {L.}~\bibnamefont
  {Alemany-Puig}},\ }\bibfield  {title} {\bibinfo {title} {Formal constraints
  on dependency syntax},\ }in\ \href
  {https://doi.org/https://doi.org/10.1016/B978-0-323-95504-1.01024-3} {\emph
  {\bibinfo {booktitle} {International Encyclopedia of Language and Linguistics
  (Third Edition)}}},\ \bibinfo {editor} {edited by\ \bibinfo {editor}
  {\bibfnamefont {H.}~\bibnamefont {Nesi}}\ and\ \bibinfo {editor}
  {\bibfnamefont {P.}~\bibnamefont {Milin}}}\ (\bibinfo  {publisher}
  {Elsevier},\ \bibinfo {address} {London},\ \bibinfo {year} {2026})\ \bibinfo
  {edition} {third edition}\ ed.,\ pp.\ \bibinfo {pages} {221--232}\BibitemShut
  {NoStop}%
\bibitem [{Note1()}]{Note1}%
  \BibitemOpen
  \bibinfo {note} {Notice that $S = \protect \frac {1}{N}$ (\protect \cref
  {eq:Simpson_index}) and then (\protect \cref
  {eq:expected_average_swap_distance_random_permutation}) \begin {equation*}
  \left < d \right >_r = (1 - S)\protect \frac {N}{N-1}\left < d \right >_{max}
  = \left < d \right >_{max}. \end {equation*}}\BibitemShut {NoStop}%
\bibitem [{\citenamefont {Loiola}\ \emph {et~al.}(2007)\citenamefont {Loiola},
  \citenamefont {de~Abreu}, \citenamefont {Boaventura-Netto}, \citenamefont
  {Hahn},\ and\ \citenamefont {Querido}}]{Loiola2007a}%
  \BibitemOpen
  \bibfield  {author} {\bibinfo {author} {\bibfnamefont {E.~M.}\ \bibnamefont
  {Loiola}}, \bibinfo {author} {\bibfnamefont {N.~M.~M.}\ \bibnamefont
  {de~Abreu}}, \bibinfo {author} {\bibfnamefont {P.~O.}\ \bibnamefont
  {Boaventura-Netto}}, \bibinfo {author} {\bibfnamefont {P.}~\bibnamefont
  {Hahn}},\ and\ \bibinfo {author} {\bibfnamefont {T.}~\bibnamefont
  {Querido}},\ }\bibfield  {title} {\bibinfo {title} {A survey for the
  quadratic assignment problem},\ }\href
  {https://doi.org/10.1016/j.ejor.2005.09.032} {\bibfield  {journal} {\bibinfo
  {journal} {European Journal of Operational Research}\ }\textbf {\bibinfo
  {volume} {176}},\ \bibinfo {pages} {657–690} (\bibinfo {year}
  {2007})}\BibitemShut {NoStop}%
\bibitem [{\citenamefont {Wang}\ \emph {et~al.}(2021)\citenamefont {Wang},
  \citenamefont {Yang}, \citenamefont {Punnen}, \citenamefont {Tian},
  \citenamefont {Yin},\ and\ \citenamefont {L\"{u}}}]{Wang2021a}%
  \BibitemOpen
  \bibfield  {author} {\bibinfo {author} {\bibfnamefont {Y.}~\bibnamefont
  {Wang}}, \bibinfo {author} {\bibfnamefont {W.}~\bibnamefont {Yang}}, \bibinfo
  {author} {\bibfnamefont {A.~P.}\ \bibnamefont {Punnen}}, \bibinfo {author}
  {\bibfnamefont {J.}~\bibnamefont {Tian}}, \bibinfo {author} {\bibfnamefont
  {A.}~\bibnamefont {Yin}},\ and\ \bibinfo {author} {\bibfnamefont
  {Z.}~\bibnamefont {L\"{u}}},\ }\bibfield  {title} {\bibinfo {title} {The
  rank-one quadratic assignment problem},\ }\href
  {https://doi.org/10.1287/ijoc.2020.1003} {\bibfield  {journal} {\bibinfo
  {journal} {INFORMS Journal on Computing}\ }\textbf {\bibinfo {volume} {33}},\
  \bibinfo {pages} {979–996} (\bibinfo {year} {2021})}\BibitemShut {NoStop}%
\bibitem [{\citenamefont {Clough}\ \emph {et~al.}(2014)\citenamefont {Clough},
  \citenamefont {Gollings}, \citenamefont {Loach},\ and\ \citenamefont
  {Evans}}]{Clough2014a}%
  \BibitemOpen
  \bibfield  {author} {\bibinfo {author} {\bibfnamefont {J.~R.}\ \bibnamefont
  {Clough}}, \bibinfo {author} {\bibfnamefont {J.}~\bibnamefont {Gollings}},
  \bibinfo {author} {\bibfnamefont {T.~V.}\ \bibnamefont {Loach}},\ and\
  \bibinfo {author} {\bibfnamefont {T.~S.}\ \bibnamefont {Evans}},\ }\bibfield
  {title} {\bibinfo {title} {{Transitive reduction of citation networks}},\
  }\href {https://doi.org/10.1093/comnet/cnu039} {\bibfield  {journal}
  {\bibinfo  {journal} {Journal of Complex Networks}\ }\textbf {\bibinfo
  {volume} {3}},\ \bibinfo {pages} {189} (\bibinfo {year} {2014})}\BibitemShut
  {NoStop}%
\bibitem [{\citenamefont {Warshall}(1962)}]{Warshall1962a}%
  \BibitemOpen
  \bibfield  {author} {\bibinfo {author} {\bibfnamefont {S.}~\bibnamefont
  {Warshall}},\ }\bibfield  {title} {\bibinfo {title} {A theorem on boolean
  matrices},\ }\href {https://doi.org/10.1145/321105.321107} {\bibfield
  {journal} {\bibinfo  {journal} {Journal of the ACM}\ }\textbf {\bibinfo
  {volume} {9}},\ \bibinfo {pages} {11} (\bibinfo {year} {1962})}\BibitemShut
  {NoStop}%
\bibitem [{\citenamefont {Aho}\ \emph {et~al.}(1972)\citenamefont {Aho},
  \citenamefont {Garey},\ and\ \citenamefont {Ullman}}]{Aho1972a}%
  \BibitemOpen
  \bibfield  {author} {\bibinfo {author} {\bibfnamefont {A.~V.}\ \bibnamefont
  {Aho}}, \bibinfo {author} {\bibfnamefont {M.~R.}\ \bibnamefont {Garey}},\
  and\ \bibinfo {author} {\bibfnamefont {J.~D.}\ \bibnamefont {Ullman}},\
  }\bibfield  {title} {\bibinfo {title} {The transitive reduction of a directed
  graph},\ }\href {https://doi.org/10.1137/0201008} {\bibfield  {journal}
  {\bibinfo  {journal} {SIAM Journal on Computing}\ }\textbf {\bibinfo {volume}
  {1}},\ \bibinfo {pages} {131} (\bibinfo {year} {1972})}\BibitemShut {NoStop}%
\bibitem [{\citenamefont {Olivella}\ and\ \citenamefont
  {Shiraito}(2025)}]{Olivella2025a}%
  \BibitemOpen
  \bibfield  {author} {\bibinfo {author} {\bibfnamefont {S.}~\bibnamefont
  {Olivella}}\ and\ \bibinfo {author} {\bibfnamefont {Y.}~\bibnamefont
  {Shiraito}},\ }\href {https://doi.org/10.32614/CRAN.package.poisbinom}
  {\bibinfo {title} {A faster implementation of the {Poisson-binomial}
  distribution}} (\bibinfo {year} {2025}),\ \bibinfo {note} {version
  1.0.2}\BibitemShut {NoStop}%
\bibitem [{\citenamefont {Hong}(2013)}]{Hong2013a}%
  \BibitemOpen
  \bibfield  {author} {\bibinfo {author} {\bibfnamefont {Y.}~\bibnamefont
  {Hong}},\ }\bibfield  {title} {\bibinfo {title} {On computing the
  distribution function for the {Poisson} binomial distribution},\ }\href@noop
  {} {\bibfield  {journal} {\bibinfo  {journal} {Computational Statistics and
  Data Analysis}\ }\textbf {\bibinfo {volume} {59}},\ \bibinfo {pages} {41}
  (\bibinfo {year} {2013})}\BibitemShut {NoStop}%
\bibitem [{\citenamefont {Cysouw}(2010)}]{Cysouw2010a}%
  \BibitemOpen
  \bibfield  {author} {\bibinfo {author} {\bibfnamefont {M.}~\bibnamefont
  {Cysouw}},\ }\bibfield  {title} {\bibinfo {title} {Dealing with diversity:
  {Towards} an explanation of {NP-internal} word order frequencies},\ }\href
  {https://doi.org/10.1515/lity.2010.010} {\bibfield  {journal} {\bibinfo
  {journal} {Linguistic Typology}\ }\textbf {\bibinfo {volume} {14}},\ \bibinfo
  {pages} {253} (\bibinfo {year} {2010})}\BibitemShut {NoStop}%
\bibitem [{\citenamefont
  {{Ferrer-i-Cancho}}(2026{\natexlab{a}})}]{Ferrer2024b}%
  \BibitemOpen
  \bibfield  {author} {\bibinfo {author} {\bibfnamefont {R.}~\bibnamefont
  {{Ferrer-i-Cancho}}},\ }\bibfield  {title} {\bibinfo {title} {The exponential
  distribution of the order of demonstrative, numeral, adjective and noun},\
  }\bibfield  {journal} {\bibinfo  {journal} {Journal of Quantitative
  Linguistics}\ }\href {https://doi.org/10.1080/09296174.2026.2617705}
  {10.1080/09296174.2026.2617705} (\bibinfo {year}
  {2026}{\natexlab{a}})\BibitemShut {NoStop}%
\bibitem [{\citenamefont {Koopmans}\ and\ \citenamefont
  {Beckmann}(1957)}]{Koopmans1957a}%
  \BibitemOpen
  \bibfield  {author} {\bibinfo {author} {\bibfnamefont {T.~C.}\ \bibnamefont
  {Koopmans}}\ and\ \bibinfo {author} {\bibfnamefont {M.~J.}\ \bibnamefont
  {Beckmann}},\ }\bibfield  {title} {\bibinfo {title} {Assignment problems and
  the location of economic activities},\ }\href@noop {} {\bibfield  {journal}
  {\bibinfo  {journal} {Econometrica}\ }\textbf {\bibinfo {volume} {25}},\
  \bibinfo {pages} {53} (\bibinfo {year} {1957})}\BibitemShut {NoStop}%
\bibitem [{\citenamefont {{Ferrer-i-Cancho}}\ \emph
  {et~al.}(2022{\natexlab{b}})\citenamefont {{Ferrer-i-Cancho}}, \citenamefont
  {Bentz},\ and\ \citenamefont {Seguin}}]{Ferrer2019c}%
  \BibitemOpen
  \bibfield  {author} {\bibinfo {author} {\bibfnamefont {R.}~\bibnamefont
  {{Ferrer-i-Cancho}}}, \bibinfo {author} {\bibfnamefont {C.}~\bibnamefont
  {Bentz}},\ and\ \bibinfo {author} {\bibfnamefont {C.}~\bibnamefont
  {Seguin}},\ }\bibfield  {title} {\bibinfo {title} {Optimal coding and the
  origins of {Zipfian} laws},\ }\href
  {https://doi.org/10.1080/09296174.2020.1778387} {\bibfield  {journal}
  {\bibinfo  {journal} {Journal of Quantitative Linguistics}\ }\textbf
  {\bibinfo {volume} {29}},\ \bibinfo {pages} {165} (\bibinfo {year}
  {2022}{\natexlab{b}})}\BibitemShut {NoStop}%
\bibitem [{\citenamefont {{Ferrer-i-Cancho}}\ \emph {et~al.}(2013)\citenamefont
  {{Ferrer-i-Cancho}}, \citenamefont {Hern\'{a}ndez-Fern\'{a}ndez},
  \citenamefont {Lusseau}, \citenamefont {Agoramoorthy}, \citenamefont {Hsu},\
  and\ \citenamefont {Semple}}]{Ferrer2012d}%
  \BibitemOpen
  \bibfield  {author} {\bibinfo {author} {\bibfnamefont {R.}~\bibnamefont
  {{Ferrer-i-Cancho}}}, \bibinfo {author} {\bibfnamefont {A.}~\bibnamefont
  {Hern\'{a}ndez-Fern\'{a}ndez}}, \bibinfo {author} {\bibfnamefont
  {D.}~\bibnamefont {Lusseau}}, \bibinfo {author} {\bibfnamefont
  {G.}~\bibnamefont {Agoramoorthy}}, \bibinfo {author} {\bibfnamefont {M.~J.}\
  \bibnamefont {Hsu}},\ and\ \bibinfo {author} {\bibfnamefont {S.}~\bibnamefont
  {Semple}},\ }\bibfield  {title} {\bibinfo {title} {Compression as a universal
  principle of animal behavior},\ }\href@noop {} {\bibfield  {journal}
  {\bibinfo  {journal} {Cognitive Science}\ }\textbf {\bibinfo {volume} {37}},\
  \bibinfo {pages} {1565} (\bibinfo {year} {2013})}\BibitemShut {NoStop}%
\bibitem [{\citenamefont {Liu}\ \emph {et~al.}(2017)\citenamefont {Liu},
  \citenamefont {Xu},\ and\ \citenamefont {Liang}}]{Liu2017a}%
  \BibitemOpen
  \bibfield  {author} {\bibinfo {author} {\bibfnamefont {H.}~\bibnamefont
  {Liu}}, \bibinfo {author} {\bibfnamefont {C.}~\bibnamefont {Xu}},\ and\
  \bibinfo {author} {\bibfnamefont {J.}~\bibnamefont {Liang}},\ }\bibfield
  {title} {\bibinfo {title} {Dependency distance: a new perspective on
  syntactic patterns in natural languages},\ }\href
  {https://doi.org/10.1016/j.plrev.2017.03.002} {\bibfield  {journal} {\bibinfo
   {journal} {Physics of Life Reviews}\ }\textbf {\bibinfo {volume} {21}},\
  \bibinfo {pages} {171} (\bibinfo {year} {2017})}\BibitemShut {NoStop}%
\bibitem [{\citenamefont {Temperley}\ and\ \citenamefont
  {Gildea}(2018)}]{Temperley2018a}%
  \BibitemOpen
  \bibfield  {author} {\bibinfo {author} {\bibfnamefont {D.}~\bibnamefont
  {Temperley}}\ and\ \bibinfo {author} {\bibfnamefont {D.}~\bibnamefont
  {Gildea}},\ }\bibfield  {title} {\bibinfo {title} {Minimizing syntactic
  dependency lengths: {Typological/Cognitive} universal?},\ }\href
  {https://doi.org/10.1146/annurev-linguistics-011817-045617} {\bibfield
  {journal} {\bibinfo  {journal} {Annual Review of Linguistics}\ }\textbf
  {\bibinfo {volume} {4}},\ \bibinfo {pages} {67} (\bibinfo {year}
  {2018})}\BibitemShut {NoStop}%
\bibitem [{\citenamefont {Semple}\ \emph {et~al.}(2022)\citenamefont {Semple},
  \citenamefont {{Ferrer-i-Cancho}},\ and\ \citenamefont
  {Gustison}}]{Semple2021a}%
  \BibitemOpen
  \bibfield  {author} {\bibinfo {author} {\bibfnamefont {S.}~\bibnamefont
  {Semple}}, \bibinfo {author} {\bibfnamefont {R.}~\bibnamefont
  {{Ferrer-i-Cancho}}},\ and\ \bibinfo {author} {\bibfnamefont
  {M.}~\bibnamefont {Gustison}},\ }\bibfield  {title} {\bibinfo {title}
  {Linguistic laws in biology},\ }\href
  {https://doi.org/10.1016/j.tree.2021.08.012} {\bibfield  {journal} {\bibinfo
  {journal} {Trends in Ecology and Evolution}\ }\textbf {\bibinfo {volume}
  {37}},\ \bibinfo {pages} {53} (\bibinfo {year} {2022})}\BibitemShut {NoStop}%
\bibitem [{\citenamefont {{Ferrer-i-Cancho}}(2014)}]{Ferrer2014c}%
  \BibitemOpen
  \bibfield  {author} {\bibinfo {author} {\bibfnamefont {R.}~\bibnamefont
  {{Ferrer-i-Cancho}}},\ }\bibfield  {title} {\bibinfo {title} {A stronger null
  hypothesis for crossing dependencies},\ }\href@noop {} {\bibfield  {journal}
  {\bibinfo  {journal} {Europhysics Letters}\ }\textbf {\bibinfo {volume}
  {108}},\ \bibinfo {pages} {58003} (\bibinfo {year} {2014})}\BibitemShut
  {NoStop}%
\bibitem [{\citenamefont {G\'omez-Rodr\'iguez}\ and\ \citenamefont
  {{Ferrer-i-Cancho}}(2017)}]{Gomez2016a}%
  \BibitemOpen
  \bibfield  {author} {\bibinfo {author} {\bibfnamefont {C.}~\bibnamefont
  {G\'omez-Rodr\'iguez}}\ and\ \bibinfo {author} {\bibfnamefont
  {R.}~\bibnamefont {{Ferrer-i-Cancho}}},\ }\bibfield  {title} {\bibinfo
  {title} {Scarcity of crossing dependencies: a direct outcome of a specific
  constraint?},\ }\href {https://doi.org/10.1103/PhysRevE.96.062304} {\bibfield
   {journal} {\bibinfo  {journal} {Physical Review E}\ }\textbf {\bibinfo
  {volume} {96}},\ \bibinfo {pages} {062304} (\bibinfo {year}
  {2017})}\BibitemShut {NoStop}%
\bibitem [{\citenamefont {G\'omez-Rodr\'iguez}\ \emph
  {et~al.}(2022)\citenamefont {G\'omez-Rodr\'iguez}, \citenamefont
  {Christiansen},\ and\ \citenamefont {{Ferrer-i-Cancho}}}]{Gomez2019a}%
  \BibitemOpen
  \bibfield  {author} {\bibinfo {author} {\bibfnamefont {C.}~\bibnamefont
  {G\'omez-Rodr\'iguez}}, \bibinfo {author} {\bibfnamefont {M.}~\bibnamefont
  {Christiansen}},\ and\ \bibinfo {author} {\bibfnamefont {R.}~\bibnamefont
  {{Ferrer-i-Cancho}}},\ }\bibfield  {title} {\bibinfo {title} {Memory
  limitations are hidden in grammar},\ }\href
  {https://doi.org/10.53482/2022_52_397} {\bibfield  {journal} {\bibinfo
  {journal} {Glottometrics}\ }\textbf {\bibinfo {volume} {52}},\ \bibinfo
  {pages} {39 } (\bibinfo {year} {2022})}\BibitemShut {NoStop}%
\bibitem [{\citenamefont {Yadav}\ \emph {et~al.}(2022)\citenamefont {Yadav},
  \citenamefont {Husain},\ and\ \citenamefont {Futrell}}]{Yadav2022a}%
  \BibitemOpen
  \bibfield  {author} {\bibinfo {author} {\bibfnamefont {H.}~\bibnamefont
  {Yadav}}, \bibinfo {author} {\bibfnamefont {S.}~\bibnamefont {Husain}},\ and\
  \bibinfo {author} {\bibfnamefont {R.}~\bibnamefont {Futrell}},\ }\bibfield
  {title} {\bibinfo {title} {Assessing corpus evidence for formal and
  psycholinguistic constraints on nonprojectivity},\ }\href
  {https://doi.org/10.1162/coli_a_00437} {\bibfield  {journal} {\bibinfo
  {journal} {Computational Linguistics}\ }\textbf {\bibinfo {volume} {48}},\
  \bibinfo {pages} {375–401} (\bibinfo {year} {2022})}\BibitemShut {NoStop}%
\bibitem [{\citenamefont {Blasi}\ \emph {et~al.}(2022)\citenamefont {Blasi},
  \citenamefont {Henrich}, \citenamefont {Adamou}, \citenamefont {Kemmerer},\
  and\ \citenamefont {Majid}}]{Blasi2022a}%
  \BibitemOpen
  \bibfield  {author} {\bibinfo {author} {\bibfnamefont {D.~E.}\ \bibnamefont
  {Blasi}}, \bibinfo {author} {\bibfnamefont {J.}~\bibnamefont {Henrich}},
  \bibinfo {author} {\bibfnamefont {E.}~\bibnamefont {Adamou}}, \bibinfo
  {author} {\bibfnamefont {D.}~\bibnamefont {Kemmerer}},\ and\ \bibinfo
  {author} {\bibfnamefont {A.}~\bibnamefont {Majid}},\ }\bibfield  {title}
  {\bibinfo {title} {Over-reliance on {English} hinders cognitive science},\
  }\href {https://doi.org/10.1016/j.tics.2022.09.015} {\bibfield  {journal}
  {\bibinfo  {journal} {Trends in Cognitive Sciences}\ }\textbf {\bibinfo
  {volume} {26}},\ \bibinfo {pages} {1153} (\bibinfo {year}
  {2022})}\BibitemShut {NoStop}%
\bibitem [{\citenamefont {Cysouw}(2008)}]{Cysouw2008a}%
  \BibitemOpen
  \bibfield  {author} {\bibinfo {author} {\bibfnamefont {M.}~\bibnamefont
  {Cysouw}},\ }\bibfield  {title} {\bibinfo {title} {Linear order as a
  predictor of word order regularities},\ }\href@noop {} {\bibfield  {journal}
  {\bibinfo  {journal} {Advances in Complex Systems}\ }\textbf {\bibinfo
  {volume} {11}},\ \bibinfo {pages} {415} (\bibinfo {year} {2008})}\BibitemShut
  {NoStop}%
\bibitem [{\citenamefont {Hammarström}(2016)}]{Hammarstroem2016a}%
  \BibitemOpen
  \bibfield  {author} {\bibinfo {author} {\bibfnamefont {H.}~\bibnamefont
  {Hammarström}},\ }\bibfield  {title} {\bibinfo {title} {{Linguistic
  diversity and language evolution}},\ }\href
  {https://doi.org/10.1093/jole/lzw002} {\bibfield  {journal} {\bibinfo
  {journal} {Journal of Language Evolution}\ }\textbf {\bibinfo {volume} {1}},\
  \bibinfo {pages} {19} (\bibinfo {year} {2016})}\BibitemShut {NoStop}%
\bibitem [{\citenamefont {Corballis}(2012)}]{Corballis2012a}%
  \BibitemOpen
  \bibfield  {author} {\bibinfo {author} {\bibfnamefont {M.~C.}\ \bibnamefont
  {Corballis}},\ }\bibfield  {title} {\bibinfo {title} {How language evolved
  from manual gestures},\ }\href {https://doi.org/10.1075/gest.12.2.04cor}
  {\bibfield  {journal} {\bibinfo  {journal} {Gesture}\ }\textbf {\bibinfo
  {volume} {12}},\ \bibinfo {pages} {200} (\bibinfo {year} {2012})}\BibitemShut
  {NoStop}%
\bibitem [{\citenamefont {Hall}\ \emph {et~al.}(2014)\citenamefont {Hall},
  \citenamefont {Ferreira},\ and\ \citenamefont {Mayberry}}]{Hall2014a}%
  \BibitemOpen
  \bibfield  {author} {\bibinfo {author} {\bibfnamefont {M.~L.}\ \bibnamefont
  {Hall}}, \bibinfo {author} {\bibfnamefont {V.~S.}\ \bibnamefont {Ferreira}},\
  and\ \bibinfo {author} {\bibfnamefont {R.~I.}\ \bibnamefont {Mayberry}},\
  }\bibfield  {title} {\bibinfo {title} {Investigating constituent order change
  with elicited pantomime: a functional account of {SVO} emergence},\ }\href
  {https://doi.org/10.1111/cogs.12105} {\bibfield  {journal} {\bibinfo
  {journal} {Cognitive Science}\ }\textbf {\bibinfo {volume} {38}},\ \bibinfo
  {pages} {943} (\bibinfo {year} {2014})}\BibitemShut {NoStop}%
\bibitem [{\citenamefont {Schouwstra}\ \emph {et~al.}(2022)\citenamefont
  {Schouwstra}, \citenamefont {Naegeli},\ and\ \citenamefont
  {Kirby}}]{Schouwstra2022a}%
  \BibitemOpen
  \bibfield  {author} {\bibinfo {author} {\bibfnamefont {M.}~\bibnamefont
  {Schouwstra}}, \bibinfo {author} {\bibfnamefont {D.}~\bibnamefont
  {Naegeli}},\ and\ \bibinfo {author} {\bibfnamefont {S.}~\bibnamefont
  {Kirby}},\ }\bibfield  {title} {\bibinfo {title} {Investigating word order
  emergence: Constraints from cognition and communication},\ }\bibfield
  {journal} {\bibinfo  {journal} {Frontiers in Psychology}\ }\textbf {\bibinfo
  {volume} {13}},\ \href {https://doi.org/10.3389/fpsyg.2022.805144}
  {10.3389/fpsyg.2022.805144} (\bibinfo {year} {2022})\BibitemShut {NoStop}%
\bibitem [{\citenamefont {Christensen}\ \emph {et~al.}(2016)\citenamefont
  {Christensen}, \citenamefont {Fusaroli},\ and\ \citenamefont
  {Tylén}}]{Christensen2016a}%
  \BibitemOpen
  \bibfield  {author} {\bibinfo {author} {\bibfnamefont {P.}~\bibnamefont
  {Christensen}}, \bibinfo {author} {\bibfnamefont {R.}~\bibnamefont
  {Fusaroli}},\ and\ \bibinfo {author} {\bibfnamefont {K.}~\bibnamefont
  {Tylén}},\ }\bibfield  {title} {\bibinfo {title} {Environmental constraints
  shaping constituent order in emerging communication systems: Structural
  iconicity, interactive alignment and conventionalization},\ }\href
  {https://doi.org/10.1016/j.cognition.2015.09.004} {\bibfield  {journal}
  {\bibinfo  {journal} {Cognition}\ }\textbf {\bibinfo {volume} {146}},\
  \bibinfo {pages} {67–80} (\bibinfo {year} {2016})}\BibitemShut {NoStop}%
\bibitem [{\citenamefont {Schouwstra}\ and\ \citenamefont {{de
  Swart}}(2014)}]{Schouwstra2014a}%
  \BibitemOpen
  \bibfield  {author} {\bibinfo {author} {\bibfnamefont {M.}~\bibnamefont
  {Schouwstra}}\ and\ \bibinfo {author} {\bibfnamefont {H.}~\bibnamefont {{de
  Swart}}},\ }\bibfield  {title} {\bibinfo {title} {The semantic origins of
  word order},\ }\href
  {https://doi.org/https://doi.org/10.1016/j.cognition.2014.03.004} {\bibfield
  {journal} {\bibinfo  {journal} {Cognition}\ }\textbf {\bibinfo {volume}
  {131}},\ \bibinfo {pages} {431} (\bibinfo {year} {2014})}\BibitemShut
  {NoStop}%
\bibitem [{\citenamefont {Levshina}(2019)}]{Levshina2019a}%
  \BibitemOpen
  \bibfield  {author} {\bibinfo {author} {\bibfnamefont {N.}~\bibnamefont
  {Levshina}},\ }\bibfield  {title} {\bibinfo {title} {Token-based typology and
  word order entropy: A study based on universal dependencies},\ }\href
  {https://doi.org/doi:10.1515/lingty-2019-0025} {\bibfield  {journal}
  {\bibinfo  {journal} {Linguistic Typology}\ }\textbf {\bibinfo {volume}
  {23}},\ \bibinfo {pages} {533} (\bibinfo {year} {2019})}\BibitemShut
  {NoStop}%
\bibitem [{\citenamefont {Levshina}\ \emph {et~al.}(2023)\citenamefont
  {Levshina}, \citenamefont {Namboodiripad}, \citenamefont
  {Allassonnière-Tang}, \citenamefont {Kramer}, \citenamefont {Talamo},
  \citenamefont {Verkerk}, \citenamefont {Wilmoth}, \citenamefont {Rodriguez},
  \citenamefont {Gupton}, \citenamefont {Kidd}, \citenamefont {Liu},
  \citenamefont {Naccarato}, \citenamefont {Nordlinger}, \citenamefont
  {Panova},\ and\ \citenamefont {Stoynova}}]{Levshina2023a}%
  \BibitemOpen
  \bibfield  {author} {\bibinfo {author} {\bibfnamefont {N.}~\bibnamefont
  {Levshina}}, \bibinfo {author} {\bibfnamefont {S.}~\bibnamefont
  {Namboodiripad}}, \bibinfo {author} {\bibfnamefont {M.}~\bibnamefont
  {Allassonnière-Tang}}, \bibinfo {author} {\bibfnamefont {M.}~\bibnamefont
  {Kramer}}, \bibinfo {author} {\bibfnamefont {L.}~\bibnamefont {Talamo}},
  \bibinfo {author} {\bibfnamefont {A.}~\bibnamefont {Verkerk}}, \bibinfo
  {author} {\bibfnamefont {S.}~\bibnamefont {Wilmoth}}, \bibinfo {author}
  {\bibfnamefont {G.~G.}\ \bibnamefont {Rodriguez}}, \bibinfo {author}
  {\bibfnamefont {T.~M.}\ \bibnamefont {Gupton}}, \bibinfo {author}
  {\bibfnamefont {E.}~\bibnamefont {Kidd}}, \bibinfo {author} {\bibfnamefont
  {Z.}~\bibnamefont {Liu}}, \bibinfo {author} {\bibfnamefont {C.}~\bibnamefont
  {Naccarato}}, \bibinfo {author} {\bibfnamefont {R.}~\bibnamefont
  {Nordlinger}}, \bibinfo {author} {\bibfnamefont {A.}~\bibnamefont {Panova}},\
  and\ \bibinfo {author} {\bibfnamefont {N.}~\bibnamefont {Stoynova}},\
  }\bibfield  {title} {\bibinfo {title} {Why we need a gradient approach to
  word order},\ }\href {https://doi.org/10.1515/ling-2021-0098} {\bibfield
  {journal} {\bibinfo  {journal} {Linguistics}\ }\textbf {\bibinfo {volume}
  {61}},\ \bibinfo {pages} {825} (\bibinfo {year} {2023})}\BibitemShut
  {NoStop}%
\bibitem [{\citenamefont {Namboodiripad}(2017)}]{Namboodiripad2017a}%
  \BibitemOpen
  \bibfield  {author} {\bibinfo {author} {\bibfnamefont {S.}~\bibnamefont
  {Namboodiripad}},\ }\emph {\bibinfo {title} {An {Experimental} {Approach} to
  {Variation} and {Variability} in {Constituent} {Order}}},\ \href
  {https://escholarship.org/uc/item/2sv6z8bz} {\bibinfo {type} {{PhD}
  {Thesis}}},\ \bibinfo  {school} {UC San Diego} (\bibinfo {year}
  {2017})\BibitemShut {NoStop}%
\bibitem [{\citenamefont {Bandt}\ and\ \citenamefont
  {Pompe}(2002)}]{Bandt2002a}%
  \BibitemOpen
  \bibfield  {author} {\bibinfo {author} {\bibfnamefont {C.}~\bibnamefont
  {Bandt}}\ and\ \bibinfo {author} {\bibfnamefont {B.}~\bibnamefont {Pompe}},\
  }\bibfield  {title} {\bibinfo {title} {Permutation entropy: A natural
  complexity measure for time series},\ }\href
  {https://doi.org/10.1103/PhysRevLett.88.174102} {\bibfield  {journal}
  {\bibinfo  {journal} {Phys. Rev. Lett.}\ }\textbf {\bibinfo {volume} {88}},\
  \bibinfo {pages} {174102} (\bibinfo {year} {2002})}\BibitemShut {NoStop}%
\bibitem [{\citenamefont {Leyva}\ \emph {et~al.}(2022)\citenamefont {Leyva},
  \citenamefont {Martínez}, \citenamefont {Masoller}, \citenamefont {Rosso},\
  and\ \citenamefont {Zanin}}]{Leyva2022a}%
  \BibitemOpen
  \bibfield  {author} {\bibinfo {author} {\bibfnamefont {I.}~\bibnamefont
  {Leyva}}, \bibinfo {author} {\bibfnamefont {J.~H.}\ \bibnamefont
  {Martínez}}, \bibinfo {author} {\bibfnamefont {C.}~\bibnamefont {Masoller}},
  \bibinfo {author} {\bibfnamefont {O.~A.}\ \bibnamefont {Rosso}},\ and\
  \bibinfo {author} {\bibfnamefont {M.}~\bibnamefont {Zanin}},\ }\bibfield
  {title} {\bibinfo {title} {20 years of ordinal patterns: {Perspectives} and
  challenges},\ }\href {https://doi.org/10.1209/0295-5075/ac6a72} {\bibfield
  {journal} {\bibinfo  {journal} {Europhysics Letters}\ }\textbf {\bibinfo
  {volume} {138}},\ \bibinfo {pages} {31001} (\bibinfo {year}
  {2022})}\BibitemShut {NoStop}%
\bibitem [{\citenamefont
  {{Ferrer-i-Cancho}}(2026{\natexlab{b}})}]{Ferrer2025c_open_data}%
  \BibitemOpen
  \bibfield  {author} {\bibinfo {author} {\bibfnamefont {R.}~\bibnamefont
  {{Ferrer-i-Cancho}}},\ }\bibfield  {title} {\bibinfo {title} {How to measure
  the optimality of word or gesture order with respect to the principle of swap
  distance minimization},\ }\href
  {https://osf.io/zxvum/overview?view_only=d0cfb658ba4e487189659d273ac2ad03}
  {\bibfield  {journal} {\bibinfo  {journal} {Open Science Foundation}\ }
  (\bibinfo {year} {2026}{\natexlab{b}})}\BibitemShut {NoStop}%
\bibitem [{\citenamefont {Àlvarez}\ \emph {et~al.}(2007)\citenamefont
  {Àlvarez}, \citenamefont {Cases}, \citenamefont {Díaz}, \citenamefont
  {Petit},\ and\ \citenamefont {Serna}}]{Alvarez2007a}%
  \BibitemOpen
  \bibfield  {author} {\bibinfo {author} {\bibfnamefont {C.}~\bibnamefont
  {Àlvarez}}, \bibinfo {author} {\bibfnamefont {R.}~\bibnamefont {Cases}},
  \bibinfo {author} {\bibfnamefont {J.}~\bibnamefont {Díaz}}, \bibinfo
  {author} {\bibfnamefont {J.}~\bibnamefont {Petit}},\ and\ \bibinfo {author}
  {\bibfnamefont {M.}~\bibnamefont {Serna}},\ }\bibfield  {title} {\bibinfo
  {title} {Communication tree problems},\ }\href
  {https://doi.org/10.1016/j.tcs.2007.04.038} {\bibfield  {journal} {\bibinfo
  {journal} {Theoretical Computer Science}\ }\textbf {\bibinfo {volume}
  {381}},\ \bibinfo {pages} {197–217} (\bibinfo {year} {2007})}\BibitemShut
  {NoStop}%
\bibitem [{\citenamefont {Garey}\ and\ \citenamefont
  {Johnson}(1979)}]{Garey1979}%
  \BibitemOpen
  \bibfield  {author} {\bibinfo {author} {\bibfnamefont {M.~R.}\ \bibnamefont
  {Garey}}\ and\ \bibinfo {author} {\bibfnamefont {D.~S.}\ \bibnamefont
  {Johnson}},\ }\href@noop {} {\emph {\bibinfo {title} {Computers and
  intractability: a guide to the theory of NP-completeness}}}\ (\bibinfo
  {publisher} {W. M. Freeman},\ \bibinfo {address} {San Francisco},\ \bibinfo
  {year} {1979})\BibitemShut {NoStop}%
\bibitem [{\citenamefont {{Ferrer-i-Cancho}}(2019)}]{Ferrer2018a}%
  \BibitemOpen
  \bibfield  {author} {\bibinfo {author} {\bibfnamefont {R.}~\bibnamefont
  {{Ferrer-i-Cancho}}},\ }\bibfield  {title} {\bibinfo {title} {The sum of edge
  lengths in random linear arrangements},\ }\href@noop {} {\bibinfo  {journal}
  {Journal of Statistical Mechanics}\ ,\ \bibinfo {pages} {053401}}\BibitemShut
  {NoStop}%
\bibitem [{\citenamefont {{Alemany-Puig}}\ \emph {et~al.}(2022)\citenamefont
  {{Alemany-Puig}}, \citenamefont {Esteban},\ and\ \citenamefont
  {{Ferrer-i-Cancho}}}]{Alemany2021a}%
  \BibitemOpen
\bibfield  {journal} {  }\bibfield  {author} {\bibinfo {author} {\bibfnamefont
  {L.}~\bibnamefont {{Alemany-Puig}}}, \bibinfo {author} {\bibfnamefont
  {J.~L.}\ \bibnamefont {Esteban}},\ and\ \bibinfo {author} {\bibfnamefont
  {R.}~\bibnamefont {{Ferrer-i-Cancho}}},\ }\bibfield  {title} {\bibinfo
  {title} {Minimum projective linearizations of trees in linear time},\ }\href
  {https://doi.org/10.1016/j.ipl.2021.106204} {\bibfield  {journal} {\bibinfo
  {journal} {Information Processing Letters}\ }\textbf {\bibinfo {volume}
  {174}},\ \bibinfo {pages} {106204} (\bibinfo {year} {2022})}\BibitemShut
  {NoStop}%
\bibitem [{\citenamefont {{Alemany-Puig}}\ and\ \citenamefont
  {{Ferrer-i-Cancho}}(2022)}]{Alemany2021b}%
  \BibitemOpen
  \bibfield  {author} {\bibinfo {author} {\bibfnamefont {L.}~\bibnamefont
  {{Alemany-Puig}}}\ and\ \bibinfo {author} {\bibfnamefont {R.}~\bibnamefont
  {{Ferrer-i-Cancho}}},\ }\bibfield  {title} {\bibinfo {title} {Linear-time
  calculation of the expected sum of edge lengths in random projective
  linearizations of trees},\ }\href {https://doi.org/10.1162/coli_a_00442}
  {\bibfield  {journal} {\bibinfo  {journal} {Journal of Computational
  Linguistics}\ }\textbf {\bibinfo {volume} {48}},\ \bibinfo {pages}
  {491–516} (\bibinfo {year} {2022})}\BibitemShut {NoStop}%
\bibitem [{\citenamefont {{Alemany-Puig}}\ and\ \citenamefont
  {{Ferrer-i-Cancho}}(2024)}]{Alemany2022c}%
  \BibitemOpen
  \bibfield  {author} {\bibinfo {author} {\bibfnamefont {L.}~\bibnamefont
  {{Alemany-Puig}}}\ and\ \bibinfo {author} {\bibfnamefont {R.}~\bibnamefont
  {{Ferrer-i-Cancho}}},\ }\bibfield  {title} {\bibinfo {title} {The expected
  sum of edge lengths in planar linearizations of trees},\ }\href
  {https://doi.org/10.15398/jlm.v12i1.362} {\bibfield  {journal} {\bibinfo
  {journal} {Journal of Language Modelling}\ }\textbf {\bibinfo {volume}
  {12}},\ \bibinfo {pages} {1} (\bibinfo {year} {2024})}\BibitemShut {NoStop}%
\bibitem [{\citenamefont {Cover}\ and\ \citenamefont
  {Thomas}(2006)}]{Cover2006a}%
  \BibitemOpen
  \bibfield  {author} {\bibinfo {author} {\bibfnamefont {T.~M.}\ \bibnamefont
  {Cover}}\ and\ \bibinfo {author} {\bibfnamefont {J.~A.}\ \bibnamefont
  {Thomas}},\ }\href@noop {} {\emph {\bibinfo {title} {Elements of information
  theory}}}\ (\bibinfo  {publisher} {Wiley},\ \bibinfo {address} {New York},\
  \bibinfo {year} {2006})\ \bibinfo {note} {2nd edition}\BibitemShut {NoStop}%
\bibitem [{\citenamefont {Hardy}\ \emph {et~al.}(1934)\citenamefont {Hardy},
  \citenamefont {Littlewood},\ and\ \citenamefont {Pólya}}]{Hardy1934a}%
  \BibitemOpen
  \bibfield  {author} {\bibinfo {author} {\bibfnamefont {G.~H.}\ \bibnamefont
  {Hardy}}, \bibinfo {author} {\bibfnamefont {J.~E.}\ \bibnamefont
  {Littlewood}},\ and\ \bibinfo {author} {\bibfnamefont {G.}~\bibnamefont
  {Pólya}},\ }\href@noop {} {\emph {\bibinfo {title} {Inequalities}}}\
  (\bibinfo  {publisher} {Cambridge University Press},\ \bibinfo {address}
  {Cambridge},\ \bibinfo {year} {1934})\BibitemShut {NoStop}%
\bibitem [{\citenamefont {Petrini}\ \emph {et~al.}(2023)\citenamefont
  {Petrini}, \citenamefont {{Casas-i-Muñoz}}, \citenamefont
  {{Cluet-i-Martinell}}, \citenamefont {Wang}, \citenamefont {Bentz},\ and\
  \citenamefont {{Ferrer-i-Cancho}}}]{Petrini2022b}%
  \BibitemOpen
  \bibfield  {author} {\bibinfo {author} {\bibfnamefont {S.}~\bibnamefont
  {Petrini}}, \bibinfo {author} {\bibfnamefont {A.}~\bibnamefont
  {{Casas-i-Muñoz}}}, \bibinfo {author} {\bibfnamefont {J.}~\bibnamefont
  {{Cluet-i-Martinell}}}, \bibinfo {author} {\bibfnamefont {M.}~\bibnamefont
  {Wang}}, \bibinfo {author} {\bibfnamefont {C.}~\bibnamefont {Bentz}},\ and\
  \bibinfo {author} {\bibfnamefont {R.}~\bibnamefont {{Ferrer-i-Cancho}}},\
  }\bibfield  {title} {\bibinfo {title} {Direct and indirect evidence of
  compression of word lengths. {Zip's law} of abbreviation revisited},\ }\href
  {https://doi.org/10.53482/2023_54_407} {\bibfield  {journal} {\bibinfo
  {journal} {Glottometrics}\ }\textbf {\bibinfo {volume} {54}},\ \bibinfo
  {pages} {58} (\bibinfo {year} {2023})}\BibitemShut {NoStop}%
\bibitem [{\citenamefont {Debowski}(2025)}]{Debowski2025a}%
  \BibitemOpen
  \bibfield  {author} {\bibinfo {author} {\bibfnamefont {L.}~\bibnamefont
  {Debowski}},\ }\bibfield  {title} {\bibinfo {title} {Corrections of
  {Zipf’s} and {Heaps’} laws derived from hapax rate models},\ }\href
  {https://doi.org/10.1080/09296174.2025.2455764} {\bibfield  {journal}
  {\bibinfo  {journal} {Journal of Quantitative Linguistics}\ }\textbf
  {\bibinfo {volume} {32}},\ \bibinfo {pages} {128} (\bibinfo {year}
  {2025})}\BibitemShut {NoStop}%
\end{thebibliography}%

\end{document}